\documentclass[twoside,11pt]{article}

\usepackage[preprint]{jmlr2e}
\usepackage{xcolor}
\usepackage{booktabs}
\usepackage{makecell}
\usepackage{multirow}
\usepackage{pgfplots}
\pgfplotsset{compat=1.18}
\usepackage{tikz}
\usetikzlibrary{backgrounds, scopes}
\usetikzlibrary{arrows.meta}
\usetikzlibrary{decorations.pathmorphing, decorations.shapes, decorations.markings, decorations.text}
\usepackage{hyperref}
\usepackage{enumitem}
\usepackage{mathtools}
\usepackage[ruled,linesnumbered]{algorithm2e}
\usepackage{subcaption}

\usepackage{amsmath,amsfonts,bm}

\def\eqref#1{(\ref{#1})}

\def\angle#1{\langle #1 \rangle}

\newcommand{\dif}{{\mathrm{d}}}

\def\vzero{{\bm{0}}}
\def\vone{{\bm{1}}}

\def\vb{{\bm{b}}}

\def\vm{{\bm{m}}}

\def\vp{{\bm{p}}}
\def\vq{{\bm{q}}}

\def\vs{{\bm{s}}}

\def\vu{{\bm{u}}}

\def\vw{{\bm{w}}}
\def\vx{{\bm{x}}}
\def\vy{{\bm{y}}}

\def\vbeta{{\bm{\beta}}}
\def\vgamma{{\bm{\gamma}}}

\def\vmu{{\bm{\mu}}}

\def\vxi{{\bm{\xi}}}

\def\mA{{\bm{A}}}

\def\mD{{\bm{D}}}

\def\mI{{\bm{I}}}

\def\mM{{\bm{M}}}

\def\mQ{{\bm{Q}}}

\def\mZ{{\bm{Z}}}

\def\mLambda{{\bm{\Lambda}}}

\def\mSigma{{\bm{\Sigma}}}

\DeclareMathAlphabet{\mathsfit}{\encodingdefault}{\sfdefault}{m}{sl}
\SetMathAlphabet{\mathsfit}{bold}{\encodingdefault}{\sfdefault}{bx}{n}

\def\gB{{\mathcal{B}}}

\def\gF{{\mathcal{F}}}

\def\gI{{\mathcal{I}}}
\def\gJ{{\mathcal{J}}}
\def\gK{{\mathcal{K}}}
\def\gL{{\mathcal{L}}}
\def\gM{{\mathcal{M}}}
\def\gN{{\mathcal{N}}}
\def\gO{{\mathcal{O}}}

\def\gS{{\mathcal{S}}}

\def\gU{{\mathcal{U}}}
\def\gV{{\mathcal{V}}}

\def\sN{{\mathbb{N}}}

\def\sQ{{\mathbb{Q}}}

\def\sT{{\mathbb{T}}}

\def\sX{{\mathbb{X}}}

\def\sZ{{\mathbb{Z}}}

\newcommand{\E}{\mathbb{E}}
\renewcommand{\P}{\mathbb{P}}

\newcommand{\R}{\mathbb{R}}

\newcommand{\KL}{D_{\mathrm{KL}}}

\newcommand{\dom}{\mathrm{dom}}

\DeclareMathOperator*{\argmin}{arg\,min}

\DeclareMathOperator{\diag}{diag}

\makeatletter
\DeclareRobustCommand{\cev}[1]{%
  {\mathpalette\do@cev{#1}}%
}
\newcommand{\do@cev}[2]{%
  \vbox{\offinterlineskip
    \sbox\z@{$\m@th#1 x$}%
    \ialign{##\cr
      \hidewidth\reflectbox{$\m@th#1\vec{}\mkern4mu$}\hidewidth\cr
      \noalign{\kern-\ht\z@}
      $\m@th#1#2$\cr
    }%
  }%
}
\makeatother

\numberwithin{equation}{section}
\numberwithin{theorem}{section}
\makeatletter
\let\c@example\c@theorem

\makeatother
\allowdisplaybreaks

\makeatletter
\def\@opargbegintheorem#1#2#3{%
  \trivlist
  \item[\hskip \labelsep{\bfseries #1\ #2}]\textbf{(#3)} \itshape%
}
\def\@begintheorem#1#2{\trivlist
  \item[\hskip \labelsep{\bfseries #1\ #2}]\itshape}
\makeatother

\newtheorem{assumption}[theorem]{Assumption}

\renewenvironment{proof}[1][Proof]{\par\noindent{\bfseries #1\ }}{\hfill\BlackBox\\[2mm]}

\usepackage[capitalise]{cleveref}
\crefname{equation}{Eq.}{Eqs.}
\crefname{figure}{Figure}{Figures}
\crefname{table}{Table}{Tables}
\crefname{section}{Section}{Sections}
\crefname{appendix}{Appendix}{Appendices}
\crefname{algorithm}{Algorithm}{Algorithms}
\crefname{lemma}{Lemma}{Lemmas}
\crefname{theorem}{Theorem}{Theorems}
\crefname{definition}{Definition}{Definitions}
\crefname{proposition}{Proposition}{Propositions}
\crefname{corollary}{Corollary}{Corollaries}
\crefname{assumption}{Assumption}{Assumptions}
\crefname{remark}{Remark}{Remarks}

\definecolor{seabornblue}{HTML}{3a76af}  
\definecolor{seabornred}{HTML}{c53932}   
\definecolor{seabornpurple}{HTML}{8e69b8}

\renewcommand{\tilde}[1]{\widetilde{#1}}
\renewcommand{\hat}[1]{\widehat{#1}}
\renewcommand{\angle}[1]{\left\langle #1 \right\rangle}
\newcommand\numberthis{\addtocounter{equation}{1}\tag{\theequation}}

\usepackage{lastpage}
\jmlrheading{27}{2026}{1-\pageref{LastPage}}{4/25; Revised
1/26}{3/26}{25-0693}{Yinuo Ren, Grant M. Rotskoff, and Lexing Ying}

\ShortHeadings{A Unified Approach to Analysis and Design of Denoising Markov Models}{Ren, Rotskoff, and Ying}
\firstpageno{1}

\begin{document}

\title{A Unified Approach to Analysis and Design of\\ Denoising Markov Models}

\author{\name Yinuo Ren \email yinuoren@stanford.edu \\
       \addr Institute for Computational and Mathematical Engineering\\
       Stanford University
       \AND
       \name Grant M. Rotskoff \email rotskoff@stanford.edu \\
       \addr Department of Chemistry\\
       Stanford University
       \AND
       \name Lexing Ying \email lexing@stanford.edu \\
       \addr Department of Mathematics\\
       Stanford University
}

\editor{Qiang Liu}

\maketitle

\begin{abstract}%
    Probabilistic generative models based on measure transport, such as diffusion and flow-based models, are often formulated in the language of Markovian stochastic dynamics, where the choice of the underlying process impacts both algorithmic design choices and theoretical analysis.
    In this paper, we aim to establish a rigorous mathematical foundation for denoising Markov models, a broad class of generative models that postulate a forward process transitioning from the target distribution to a simple, easy-to-sample distribution, alongside a backward process particularly constructed to enable efficient sampling in the reverse direction. 
    Leveraging deep connections with nonequilibrium statistical mechanics and generalized Doob's $h$-transform, we propose a minimal set of assumptions that ensure: (1) explicit construction of the backward generator, (2) a unified variational objective directly minimizing the measure transport discrepancy, and (3) adaptations of the classical score-matching approach across diverse dynamics. 
    Our framework unifies existing formulations of continuous and discrete diffusion models, identifies the most general form of denoising Markov models under certain regularity assumptions on forward generators, and provides a systematic recipe for designing denoising Markov models driven by arbitrary L\'evy-type processes.
    We illustrate the versatility and practical effectiveness of our approach through novel denoising Markov models employing geometric Brownian motion and jump processes as forward dynamics, highlighting the framework's potential flexibility and capability in modeling complex distributions.
\end{abstract}


\section{Introduction}

Generative modeling, originating from foundational works such as VAE~\citep{kingma2013auto} and GAN~\citep{goodfellow2014generative}, has become one of the central pillars of modern machine learning. 
Following the advent of normalizing flows~\citep{rezende2015variational,zhang2018monge}, Markov processes emerged as a key mathematical foundation for generative models, leading to the groundbreaking development of diffusion models~\citep{sohl2015deep,song2019generative,ho2020denoising,song2020score,song2021maximum} and flow-based generative models~\citep{lipman2022flow,liu2022flow,albergo2022building,albergo2023stochastic}. 
At their core, these models rely on Markovian dynamics to smoothly transform an easy-to-sample reference distribution into complex target data distributions. 
The adaptability of this paradigm has inspired numerous variants tailored to diverse data modalities, including discrete data~\citep{austin2021structured,campbell2022continuous,lou2023discrete}, and
data supported on manifolds~\citep{de2022riemannian,huang2022riemannian,chen2024flow,zhu2024trivialized}.
Generative models based on Markov processes have consistently achieved state-of-the-art performance across various applications, from image synthesis~\citep{meng2021sdedit,ho2022cascaded,bar2023multidiffusion,chen2024deconstructing,ma2024sit,croitoru2023diffusion}, video generation~\citep{ho2022video,yang2023diffusion,wu2023tune,xing2024survey}, and natural language processing~\citep{li2022diffusion,wu2023ar,shi2024simplified,xu2024energy}, to physical, chemical, and biological domains~\citep{xu2022geodiff,yang2023scalable,frey2023protein,guo2024diffusion,stark2024dirichlet}.

The theoretical and algorithmic development of these models is deeply rooted in the mathematical theory of Markov processes.  
Continuous diffusion models, for instance, leverage the theory of stochastic differential equations (SDEs)~\citep{tzen2019theoretical,block2020generative,chen2022sampling,lee2022convergence,chen2023improved,benton2023nearly,wang2024evaluating}, along with ordinary differential equations (ODEs) through the probability flow formulation~\citep{chen2024probability,li2024towards,gao2024convergence,huang2024convergence}.
In parallel, discrete diffusion models utilize the theory of continuous-time Markov chains and Poisson random measures~\citep{campbell2022continuous,chen2024convergence,ren2024discrete}. 
To further improve the efficiency and scalability of these generative frameworks, specialized numerical methods taking advantage of the properties of these underlying Markov processes have been introduced, including high-order numerical schemes~\citep{zhang2022fast,lu2022dpm,lu2022dpm++,dockhorn2022genie,zheng2023dpm,li2024accelerating,wu2024stochastic,ren2025fast}, parallel sampling strategies~\citep{chen2024accelerating,shih2024parallel,selvam2024self}, and randomized algorithms~\citep{kandasamy2024poisson,gupta2024faster,li2024improved}. 
Recent advances in diffusion models have also significantly informed and advanced research in related guidance~\citep{ho2022classifier,chidambaram2024does,skreta2025feynman,chen2025solving,ren2025driftlite} and sampling problems~\citep{zhang2021path,richter2023improved,vargas2023denoising,albergo2024nets,chen2024sequential,guo2024provable,guo2025complexity}.

Motivated by these remarkable empirical successes, there is a growing interest in establishing a unified theoretical understanding and generalizing current Markovian generative frameworks to accommodate a broader class of processes. Such extensions promise to facilitate the modeling of more intricate probability paths, potentially offering enhanced efficiency and expressiveness for handling complex, heavy-tailed, or multimodal distributions. 
Recently, \citet{benton2024denoising} proposed a general framework known as \emph{denoising Markov models}, which operates directly on the level of generators rather than explicit transition kernels to unify continuous and discrete diffusion models. Meanwhile, \citet{holderrieth2024generator} developed the generator matching framework, extending flow-matching methods to general Markov processes, including pure diffusion and jump dynamics. Nevertheless, a comprehensive generalization and theoretical characterization of denoising Markov models remain open challenges.

\subsection{Contributions}

This paper provides a rigorous treatment of the mathematical foundations of denoising Markov models using the language of Markov process generators. Our analysis deepens the theoretical understanding and systematically guides the algorithmic design of such generative models. Specifically, our contributions are summarized as follows:

\begin{itemize}
    \item We employ a fundamental connection between the parametrization of backward Markov processes and the generalized Doob's $h$-transform~\citep{chetrite2015nonequilibrium}. This insight allows us to formulate minimal and explicit assumptions required to justify the application of the change-of-measure argument, derive a theoretically grounded loss function that directly minimizes the KL divergence between the target and generated distributions, and generalize the classical score-matching technique~\citep{hyvarinen2005estimation} to a broad class of Markovian models.
    \item Our mathematical framework unifies and extends several existing denoising Markov models, encompassing both continuous and discrete diffusion models as well as more sophisticated L\'evy-It\^o models~\citep{yoon2023score}. Furthermore, we characterize the most general form of denoising Markov models under standard regularity conditions on their generators~\citep{courrege1965forme} and thus generalize the permissible class of forward processes from diffusion and jump processes to arbitrary L\'evy-type processes.
    \item Based on these theoretical developments, we provide a general yet practical recipe for training and inference in denoising Markov models. Our framework accommodates any underlying Markov process satisfying the established conditions, ensuring theoretical guarantees on model performance. Analogous to the generator matching framework~\citep{holderrieth2024generator}, which extends the flow-matching paradigm, our results significantly enhance the flexibility of model design and efficiency of training algorithms in the context of denoising Markov models.
\end{itemize}

In addition, we introduce novel instances of denoising Markov models by utilizing geometric Brownian motion and general jump processes as forward processes, which, to our knowledge, have not been previously explored. Through carefully designed empirical experiments, we validate the versatility, practicality, and effectiveness of our generalized framework, demonstrating its capability to model complex target distributions.

\begin{figure}[!p]
    \captionsetup{singlelinecheck=off}
    \centering
    \begin{tikzpicture}

        \begin{scope}[on background layer]
            \shade[left color=seabornred!5, right color=seabornred!20] 
                  (-7.3, -0.6) rectangle (7.3, 3.4);
            \shade[left color=seabornblue!20, right color=seabornblue!5] 
                  (-7.3, -4.8) rectangle (7.3, -0.6);
        \end{scope}
      
        \node[align=left, font={\itshape}] (left)  at (-6.0, 2.7) {Data\\Distribution};
        \node[align=right, font={\itshape}] (right) at (5.8, 2.7) {\itshape Easy-to-Sample\\Distribution};
      
        \draw[dotted, line width=1pt] (-3.6, -4.8) -- (-3.6, 3.4);
        \draw[dotted, line width=1pt] (3.3, -4.8)  -- (3.3, 3.4);
        \draw[dotted, line width=1pt] (-7.3, -0.6) -- (7.3, -0.6);

        \draw[line width=.2pt] (-7.3, 2.2) -- (-4.4, 2.2);
        \draw[line width=.2pt] (4.3, 2.2) -- (7.3, 2.2);
      
        \draw[line width=1pt, ->, seabornred] (-7.0, -0.5) to node[midway, right, align=left, seabornred, font={\bfseries}]{True} (-7.0, 0.5);
          
        \draw[line width=1pt, ->, seabornblue] (-7.0, -0.5) to node[midway, right, align=left, seabornblue, font={\bfseries}]{Estimated} (-7.0, -1.5);
      
        \draw[line width=.8pt, ->] (-4.5, 2.8) to node[midway, above, sloped, font=\bfseries]{Forward (Noising) Process} (4.0, 2.8);
        \draw[line width=.8pt, <-] (-4.5, -4.2) to node[midway, below, sloped, font=\bfseries]{Backward (Denoising) Process} (4.0, -4.2);
      
        \node (x0) at (-4.4, 1.6) {\color{seabornred} $x_0$};
        \node (xT) at (4.2, 1.6)  {\color{seabornred} $x_T$};
        \draw[line width=1pt, ->, seabornred] (x0)
          to [in=170, out=10]
            node[midway, above] {$\partial_t p_t = \gL_t^* p_t$}
            node (3_x) [pos=0.3] {}
            node (15_x) [pos=0.13] {}
          (xT);
        \node[left of=x0, xshift=-3pt, align=right, seabornred, font=\footnotesize] {Given\\Samples};
      
        \node (cev_x0) at (4.2, 0)  {\color{seabornred} $\cev x_0$};
        \node (cev_xT) at (-4.4, 0) {\color{seabornred} $\cev x_T$};
        \draw[line width=1pt, <-, seabornred] (cev_xT)
          to [in=170, out=10]
            node[midway, below] {$\partial_t \cev p_t = \cev \gL_t^* \cev p_t$}
            node (3_cev_x) [pos=0.3] {}
            node (15_cev_x) [pos=0.13] {}
            node (85_cev_x) [pos=0.85] {}
          (cev_x0);
      
        \draw[<->] (3_x) -- (3_cev_x)
          node[midway, right, align=left]{
            ${\color{seabornred} \cev \gL_{T-t}} = {\color{seabornred}\gL_t^*} + {\color{seabornred} p_t^{-1}} {\color{seabornred} \Gamma_t^*}({\color{seabornred} p_t}, \cdot)$\\
            (\cref{thm:time_reversal})
          };
      
        \path (x0) to node[midway]{\rotatebox{90}{\scalebox{3}[1]{$=$}}} (cev_xT);
        \path (xT) to node[midway]{\rotatebox{90}{\scalebox{3}[1]{$=$}}} (cev_x0);
      
        \node (y0) at (4.2, -3.0)  {\color{seabornblue} $y_0$};
        \node (yT) at (-4.4, -3.0) {\color{seabornblue} $y_T$};
        \draw[line width=1pt, <-, in=-170, out=-10, seabornblue] (yT)
          to node[midway, below] {$\partial_t q_t = \gK_t^* q_t$}
            node (15_y) [pos=0.13] {}
            node (85_y) [pos=0.85] {}
          (y0);
        \node[left of=yT, xshift=-3pt, yshift=-5pt, align=right, seabornblue, font=\footnotesize] {Generated\\Samples};
      
        \draw[<->] (15_y) -- (15_x)
          node[pos=0.37, right, align=left]{
            ${\color{seabornblue} \gK_{T-t}} = {\color{seabornred} \gL_t^*} + {\color{seabornblue} \varphi_t^{-1}} {\color{seabornred} \Gamma_t^*}({\color{seabornblue} \varphi_t}, \cdot)$\\
            (\cref{ass:K})
          };
      
        \draw[<->] (85_y) -- (85_cev_x)
          node[pos=0.2, left, align=right]{
            ${\color{seabornblue} \gK_{T-t}} = {\color{seabornred} \cev\gL_{T-t}} + {\color{seabornpurple} \eta_t^{-1}} {\color{seabornred} \cev\Gamma_{T-t}}({\color{seabornpurple}  \eta_t}, \cdot)$\\
            (\cref{lem:alpha})
          };
      
        \draw[<->] (cev_xT) to node[midway, pos=0.7, left, align=right] {
            $\KL({\color{seabornred} p_0} \| {\color{seabornblue} q_T})$\\
            (\cref{cor:error_bound})
          } (yT);
        
        \path (cev_x0) to node[midway]{\rotatebox{90}{\scalebox{3}[.8]{$\approx$}}} (y0);
      
        \node (P) at (4.7, 1)    {\color{seabornred} \Large $\P$};
        \node (Q) at (4.7, -2.8) {\color{seabornblue} \Large $\sQ$};
        \draw[<->] (P) to node[right, pos=0.6, align=left] {
            $\KL({\color{seabornred} \P} \| {\color{seabornblue} \sQ})$\\
            (\cref{thm:change_of_measure})
          } (Q);
          
    \end{tikzpicture}
      
    \caption{\textbf{Conceptual roadmap of our denoising Markov model framework.} This diagram summarizes how forward, backward, and estimated processes are related at the level of generators.}
    \vskip -.5em
    \begin{itemize}
        \item {\color{seabornred} \textbf{Red (true)}}: The forward Markov process $x_t$ evolves from the data distribution $p_0$ to an easy-to-sample reference $p_T \approx q_0$, with generator $\gL_t$ and adjoint $\gL_t^*$ satisfying their corresponding Kolmogorov equation. The relation between the forward and backward processes is given by \cref{thm:time_reversal}.
        \item {\color{seabornblue} \textbf{Blue (estimated)}}: The learned backward process $y_t$ starts from $q_0$ and aims to reconstruct $p_0$. Its generator $\gK_t$ is assumed to be parametrized by a positive function $\varphi_t$ in the form of \cref{ass:K}, where $\Gamma_t^*$ is the carr\'e du champ operator associated with the adjoint generator $\gL_t^*$. 
        \item {\color{seabornpurple} \textbf{Purple (link)}}: The density ratio $\eta_t = \varphi_t p_t^{-1}$ connects the two processes via \cref{lem:alpha}, where $\cev \Gamma_t$ is the carr\'e du champ operator associated with the backward generator $\cev \gL_t$. This demonstrates that the estimated backward generator $\gK_t$ can be deemed as a perturbation from the backward generator $\cev \gL_t$, which leads to the change-of-measure loss (\cref{thm:change_of_measure}) that upper bounds $\KL(p_0\|q_T)$ (\cref{cor:error_bound}).
    \end{itemize}
    \label{fig:denoising}
\end{figure}

\subsection{Outline}

This paper is organized in the following structure: \cref{sec:prelim} reviews the background of denoising Markov models and recalls two key special cases, continuous and discrete diffusion models, that motivate our generalization; \cref{sec:markov} develops the core mathematical theory of denoising Markov models at the generator level, the roadmap of which is illustrated in \cref{fig:denoising}. \cref{sec:examples} specializes these general results to concrete process classes, including diffusions, finite-state jump processes, and general Lévy-type processes, illustrating how familiar score-matching losses and their discrete counterparts arise as special cases.
\cref{sec:experiments} presents proof-of-concept experiments, including models driven by geometric Brownian motion and pure jump processes, which demonstrate the flexibility of our framework beyond classical diffusion dynamics.
Finally, \cref{sec:conclusion} discusses implications and future directions.

\section{Preliminaries}
\label{sec:prelim}

In this section, we outline the problem settings and basic concepts pertinent to generative modeling with Markov processes. We review the problem setting of denoising Markov models, emphasizing two prominent and well-studied diffusion model variants designed respectively for continuous and discrete data distributions.

\subsection{Problem Setting and Notations}

We consider the following problem setting and adopt these notational conventions throughout this paper. We refer readers to~\cref{fig:denoising} for a diagram of this framework. 
The main objective of generative modeling is to sample from a \emph{target distribution} $p_\text{data}$ supported on a measure space $(E, \gB(E), \mu)$, given a dataset sampled from $p_\text{data}$, which forms an empirical distribution $\hat p_\text{data} \approx p_\text{data}$. Throughout, general states are denoted $x_t \in E$ and in bold specifically when the states have a clear multi-dimensional structure, \emph{i.e.}, $E = \R^d$ or $E = [S]^d$.

In denoising Markov models, we first construct a c\'adl\'ag Markov process $(x_t)_{t \in [0, T]}$ on $E$ in the probability space $(\Omega, \gF, \P)$, referred to as the \emph{forward process} in the sequel, originating from $x_0 \sim p_\text{data}$. This process progressively injects noise, causing the sample to deviate from the original data distribution. We denote the distribution of $x_t$ as $p_t$, with $p_0 = p_\text{data}$. The forward process is designed in a way that after a sufficiently long time horizon $T$, the resulting distribution $p_T$ approximates or coincides with a simpler distribution $q_0$ that is relatively straightforward to sample from (\emph{e.g.}, a standard Gaussian) than $p_0$.

Denoising Markov models then aim to construct another denoising Markov process $(y_t)_{t \in [0, T]}$, termed the \emph{backward process} below, whose distribution at time $t$ is denoted as $q_t$.
Starting from the simpler distribution $q_0$, the backward process also evolves over the time horizon $T$ to yield a distribution $q_T$, intended to closely approximate the target distribution $p_\text{data}$. Typically, constructing the backward process involves the following two main steps: (1) parametrizing the backward process $(y_t)_{t \in [0, T]}$ using a neural network, and (2) training the neural network by minimizing the discrepancy between the time-reversed forward process $(\cev x_t)_{t \in [0, T]}$ and the backward process $(y_t)_{t \in [0, T]}$.

In general, the learnability of the backward process is significantly influenced by its parametrization, \emph{i.e.}, the class of Markov processes representable by a particular neural network architecture. Choosing an appropriate parametrization is thus a crucial component in the design of denoising Markov models. Additionally, careful construction of the loss function is essential for enabling efficient training and providing strong theoretical guarantees. Favorably, the minimization of the loss function directly optimizes the discrepancy between the target and generated distributions, which is often quantified by the KL divergence. 

In the following, we review two well-known denoising Markov models, namely continuous and discrete diffusion models, by introducing their respective forward and backward processes and the corresponding parametrization and training procedures.

\subsection{(Continuous) Diffusion Models}
\label{sec:cont_diffusion}

(Continuous) diffusion models~\citep{song2020score} are denoising Markov models designed for data supported on the $d$-dimensional space $\R^d$, \emph{i.e.}, $E = \R^d$, $\gB(E)$ is the Borel $\sigma$-algebra on $\R^d$, and $\mu$ is the Lebesgue measure. The forward process $(\vx_t)_{t \in [0, T]}$ is described by a diffusion process. Specifically, the forward process $(\vx_t)_{t \in [0, T]}$ is governed by the following stochastic differential equation (SDE):
\begin{equation}
    \dif \vx_t = \vb_t(\vx_t) \dif t + \dif \vw_t,
    \label{eq:continuous_diffusion_forward}
\end{equation}
where $\vb_t$ denotes the drift term and $(\vw_t)_{t \geq 0}$ is a $d$-dimensional Wiener process. For simplicity, we assume an identity diffusion coefficient, and its generalization to the time-inhomogeneous case is straightforward with time reparametrization. A common choice of the drift is $\vb_s = - \frac{1}{2} \vx_s$, yielding the Ornstein-Uhlenbeck (OU) process and thus driving the forward process~\eqref{eq:continuous_diffusion_forward} towards a standard Gaussian distribution exponentially fast. 

The forward SDE~\eqref{eq:continuous_diffusion_forward} naturally corresponds to a backward process $(\cev \vx_t)_{t\in [0, T]}$, which is also a diffusion process satisfying the following SDE~\citep{anderson1982reverse}:
\begin{equation}
    \dif \cev \vx_t = \left(- \vb_{T-t}(\cev \vx_t) + \nabla \log p_{T-t} (\cev \vx_t) \right) \dif t + \dif \vw_t,
    \label{eq:continuous_diffusion_backward}
\end{equation}
where the term $\vs_t(\vx) := \nabla \log p_t(\vx)$, known as the \emph{score function}, is often approximated by a neural network $\hat \vs_t^\theta(\vx)$ parametrized by $\theta$. In other words, the backward process is parametrized by
the family of estimated backward processes $(\vy_t)_{t \in [0, T]}$ satisfying the following SDE:
\begin{equation*}
    \dif \vy_t = \left(- \vb_{T-t}(\vy_t) + \hat \vs_{T-t}^\theta(\vy_t) \right) \dif t + \dif \vw_t,
\end{equation*}
and when the estimated score function $\hat \vs_t^\theta(\vx)$ coincides with the true score function $\vs_t(\vx)$, \emph{i.e.}, $\hat \vs_t^\theta(\vx) = \vs_t(\vx)$, the backward process $(\vy_t)_{t \in [0, T]}$ is exactly the time-reversed forward process $(\cev \vx_t)_{t \in [0, T]}$.

The parameters are estimated through the score-matching objective~\citep{hyvarinen2005estimation}:
\begin{equation}
    \min_\theta \E_{\vx_0 \sim p_0}\left[ \int_0^T \psi_t \E_{\vx_t \sim p_{t|0}(\cdot | \vx_0)} \left[ \left\| \nabla \log p_{t|0}(\vx_t|\vx_0) - \hat \vs_t^\theta(\vx_t) \right\|^2 \right] \dif t \right], 
    \label{eq:continuous_diffusion_loss}
\end{equation}
where $p_{t|0}(\cdot|\vx_0)$ denotes the conditional distribution at time $t$ given a sample $\vx_0$ from the data distribution $p_0$, and $\psi_t$ is a time-dependent weighting function. Specifically, in the case of the OU process where $\vb_t = -\frac{1}{2} \vx_t$, we have the following closed-form formula for the conditional density:
\begin{equation*}
    \nabla \log p_{t|0}(\vx_t|\vx_0) = -\dfrac{\vx_t - \vx_0 e^{-t/2}}{1 - e^{-t}},
\end{equation*}
through which \cref{eq:continuous_diffusion_loss} gives a practical loss function for training.

\subsection{Discrete Diffusion Model}
\label{sec:disc_diffusion}

Discrete diffusion models consider data supported on a finite discrete set $\sX$, \emph{i.e.}, $E = \sX$, $\gB(E) = 2^{\sX}$, and $\mu$ is the counting measure.
The forward process $(x_t)_{t \in [0, T]}$ is a continuous-time Markov chain and the distribution vector $\vp_t = (p_t(x))_{x \in \sX}$ of the forward process at time $t$ satisfies the following evolution equation:
\begin{equation}
    \dfrac{\dif \vp_t}{\dif t} = \mLambda_t \vp_t, 
\label{eq:discrete_diffusion_forward}
\end{equation}
where $\mLambda_t = (\Lambda_t(y, x))_{x, y \in \sX}$ is the \emph{rate matrix} of the continuous-time Markov chain satisfying the following two conditions: (1) $\Lambda_t(x, x) = - \sum_{y \neq x} \Lambda_t(y, x)$, for all $x \in \sX$, and (2) $\Lambda_t(x, y) \geq 0$, for all $x \neq y$. 
At the pointwise level, the forward process $(x_t)_{t \in [0, T]}$ also admits a stochastic integral formulation~\citep{ren2024discrete}:
\begin{equation*}
    x_t = x_0 + \int_0^t \int_\sX (x - x_{s-}) N[\lambda](\dif s, \dif x),
\end{equation*}
where $N[\lambda](\dif s, \dif x)$ denotes the Poisson random measure with intensity $\lambda_s(x, x_{s^-}) \dif x$, with the intensity $\lambda_t(y, x) = \Lambda_t(y, x) (1 - \delta_x(y))$.

Strictly speaking, the jump measure depends on the outcome $\omega\in\Omega$ through the predictable
state process $x_{s^-}(\omega)$. A convenient and rigorous construction is via thinning
of a standard Poisson random measure on an augmented space.
Let $\mu$ be a reference measure on $\sX$ (\emph{e.g.}, counting measure when $\sX$ is discrete), and let
$N(\dif s,\dif x,\dif z)$ be a Poisson random measure on $[0,\infty)\times \sX\times \R_+$ with deterministic
intensity $\dif s\mu(\dif x)\dif z$, meaning that (1) for any $0\le s<t<\infty$ and any Borel sets
$A\subset \sX$ and $B\subset \R_+$ with finite measure,
\[
    N\big((s,t]\times A\times B\big)\sim \mathrm{Poisson}\big((t-s)\mu(A)|B|\big),
\]
and (2) the random variables $N\big((s,t]\times A_i\times B_i\big)$ are independent over disjoint sets
$(s,t]\times A_i\times B_i$.
Given a predictable nonnegative intensity $\lambda_s(x, x_{s^-}(\omega))$, define the thinned
(integer-valued) random measure on $[0,\infty)\times\sX$ by
\[
    N[\lambda](\omega,\dif s,\dif x)
    :=
    \int_0^\infty \mathbf 1_{z\le \lambda_s(x,x_{s^-}(\omega))} N(\dif s,\dif x,\dif z).
\]
Then $N[\lambda]$ has predictable compensator $\lambda_s(x,x_{s^-}(\omega))\dif s\mu(\dif x)$.
We refer readers to~\citet{bottcher2013levy} for standard Poisson random measures and to~\citet{protter1983point,ren2024discrete}
for more details on this construction procedure and applications. In the following we
omit the dependence on $\omega$ for notational simplicity.

It is a classical result that the forward process~\eqref{eq:discrete_diffusion_forward} also corresponds to a backward process $(\cev x_t)_{t\in [0, T]}$ with another rate matrix $\cev{\mLambda}_t = (\cev \Lambda_t(y, x))_{x, y \in \sX}$ defined as
\begin{equation}
    \cev \Lambda_t(y, x) =
    \begin{cases}
        s_{T-t}(x, y) \Lambda_{T-t}(x, y),\ &\forall x \neq y,\\
        - \sum_{y' \neq x} \cev \Lambda_t(y', x),\ &\forall x = y,
    \end{cases}
    \label{eq:discrete_diffusion_backward}
\end{equation}
where the score function $\vs_t$, a vector-valued function, is defined as
$$\vs_t(x) = (s_t(x,y))_{y\in\sX} = \frac{\vp_t}{p_t(x)}.$$

Similar to the continuous case, we use a neural network $\hat \vs_t^\theta$ with parameters $\theta$ to estimate the score function $\vs_t$, \emph{i.e.}, parametrizing the backward process $(\cev x_t)_{t \in [0, T]}$ by the family of estimated backward processes $(y_t)_{t \in [0, T]}$ satisfying the following evolution equation:
\begin{equation}
    \dfrac{\dif \vq_t}{\dif t} = \overline \mLambda_t \vq_t, \quad \text{with} \quad \overline \Lambda_t(y, x) =
    \begin{cases}
        \hat s_{T-t}^\theta(x, y) \Lambda_{T-t}(x, y),\ &\forall x \neq y,\\
        - \sum_{y' \neq x} \overline \Lambda_t(y', x),\ &\forall x = y.
    \end{cases}
    \label{eq:discrete_diffusion_backward_param}
\end{equation}
This also corresponds to a pointwise stochastic integral formulation
\begin{equation*}
    y_t = y_0 + \int_0^t \int_\sX (y - y_{s-}) N[\overline \lambda](\dif s, \dif y),
\end{equation*}
where $N[\overline \lambda](\dif s, \dif y)$ is the Poisson random measure with evolving intensity $\overline \lambda_s(y, y_{s^-}) \dif y$, with the intensity $\overline \lambda_t(y, x) = \overline \Lambda_t(y, x)(1-\delta_x(y))$.

The score function $\vs_t$ is estimated via the score-matching objective~\citep{lou2023discrete}:
\begin{equation}
    \min_{\theta}\E_{x_0 \sim p_0} \left[\int_0^T \psi_t \E_{x_t \sim p_{t|0}(\cdot|x_0)} \left[\sum_{y \neq x_t} \left(\hat s^\theta_t(x_t, y) - \frac{p_{t|0}(y|x_0)}{p_{t|0}(x_t|x_0)} \log \hat s^\theta_t(x_t, y) \right) \Lambda_t(x_t, y)\right] \right] \dif t. 
    \label{eq:discrete_diffusion_loss}
\end{equation}
In practice, a usual setting of the state space is the set of $d$-dimensional discrete vectors, \emph{i.e.}, $\sX = [S]^d$, where $S$ is the number of possible states along each dimension. The forward process $(\vx_t)_{t \in [0, T]}$ is often chosen such that the rate matrix $\mLambda_t$ is sparse, for which two rate matrices are commonly adopted: (1) the uniform rate matrix: $\lambda(\vy, \vx) = d^{-1}$ only if $\|\vy-\vx\|_0 = 1$,
for which we have 
\begin{equation*}
    p_{t|0}(\vx_t|\vx_0) = \prod_{i=1}^d \left(e^{-t}\delta_{x_0^i}(x_t^i)+\dfrac{1-e^{-t}}{S}\right),
\end{equation*}
and (2) the masked rate matrix: $\lambda(\vx_{-k}\oplus 0, \vx) = 1$ if $x_k \neq 0$, where $0$ is appended to each dimension of the state space as a null state and $\vx_{-k}\oplus 0$ denotes the state $x$ with the $k$-th dimension set to $0$, in which case we have 
\begin{equation*}
    p_{t|0}(\vx_t|\vx_0) = \prod_{i=1}^d \left(e^{-t}\delta_{x_0^i}(x_t^i) + (1-e^{-t})\delta_0(x_t^i) \right). 
\end{equation*}
Both cases converge exponentially fast, and~\cref{eq:discrete_diffusion_loss} gives a practical loss function for training.

\section{Denoising Markov Models}
\label{sec:markov}

In this section, we present the mathematical framework underpinning denoising Markov models. We first introduce the concept of generators, which serve as fundamental building blocks of these models. Next, we discuss the time reversal of the forward process, yielding the backward process, explore how its generator relates explicitly to the generator of the forward process, and provide appropriate assumptions regarding the parameterization of the estimated backward process. 
Connecting the relations between the forward and backward processes to the generalized Doob's $h$-transform, we provide the change-of-measure arguments that naturally lead to a derivation of the score-matching objective for general denoising Markov models formulated in terms of generators. Finally, we give meta-algorithms for training and sampling from denoising Markov models, which can be applied to any forward process satisfying the established conditions, and provide a theoretical guarantee on the performance of the estimated backward process in terms of the KL divergence between the target and generated distributions.

Throughout this section, definitions, theorems, and proofs are presented informally to highlight the essential concepts and intuitions. For precise mathematical definitions, rigorous theorem statements, explicit assumptions, and detailed proofs, we direct the reader to \cref{app:math}.

\subsection{Mathematical Background}

We assume that the forward process $(x_t)_{t\in [0, T]}$ is governed by an \emph{evolution system} $(U_{t,s})_{s \leq t}$ (\cref{def:evolution_system}), defined by
\begin{equation*}
    U_{t,s}f(x) = \E[f(x_t) | x_s = x].
\end{equation*} 

Furthermore, we assume that $(U_{t,s})_{s \leq t}$ is a \emph{Feller evolution system} (\cref{def:feller_evolution_system}). Roughly speaking, a Feller evolution system is a family of linear, time-evolution operators mapping the space $C_0(E)$ consisting of all continuous functions vanishing at infinity on $E$, into itself, that satisfy positivity-preservation, contractivity, and strong continuity.

Feller evolution systems can be characterized by their right generators. The right generator $\gL_t$ of the forward process $(x_t)_{t\in [0, T]}$ is defined as follows:
\begin{definition}[Forward Generator, Informal Version of~\cref{def:right_and_left_generator}]
    For each $t \in [0, T]$, the right generator of the forward process $(x_t)_{t\in [0, T]}$ at time $t$ is given by:
    \begin{equation*}
        \gL_t f(x) = \lim_{h \to 0^+} \E\left[\dfrac{f(x_{t+h}) - f(x_t)}{h}\bigg| x_t = x\right].
    \end{equation*}
    \label{def:right_generator}
\end{definition}

Should no ambiguity arise, we refer simply to this as the forward generator. If the forward process $(x_t)$ is time-homogeneous, then $T_{t-s}:=U_{t,s}$ forms a \emph{one-parameter semigroup} or simply semigroup (\cref{def:semigroup}), with a time-homogeneous generator $\gL$ (\cref{def:generator}).
In the special case when $(U_{t,s})_{s \leq t}$ is a Feller evolution system, the corresponding semigroup $(T_t)_{t\geq 0}$ is called a \emph{Feller semigroup} (\cref{def:feller_semigroup}).  

Feller processes, which are Markov processes governed by Feller semigroups (\cref{def:feller_process}), have been extensively studied in the literature~\citep{ethier2009markov,bottcher2013levy}. They possess several convenient properties, including the existence of transition kernels as unique positive Radon measures by the Riesz representation theorem~\citep[Theorem~1.5]{bottcher2013levy}, an adapted form of the Hille-Yosida theorem~\citep[Theorem~4.2.2]{ethier2009markov}, and the following Dynkin's formula (\cref{thm:dynkin}):
\begin{equation*}
    \E[f(x_t) | x_0 = x] - f(x_0) = \E\left[\int_0^t \gL f(x_s)\dif s\bigg| x_0 = x\right].
\end{equation*}

In most practical settings, however, the evolution system $(U_{t,s})_{s \leq t}$ is time-inhomogeneous. Nevertheless, one can often transform the forward process $(x_t)_{t\in [0, T]}$ into an augmented process that incorporates the time index explicitly, without introducing additional randomness. We refer to~\cref{def:augmented} for details of this transformation. 
By~\cref{thm:equivalence}, the augmented process $(\tilde x_t)_{t\in [0, T]}$, defined on the probability space $(\tilde\Omega, \tilde\gF, \tilde\P)$ with the augmented generator $\tilde \gL$, is a Feller process if and only if the original process $(x_t)_{t\in [0, T]}$ is governed by a Feller evolution system. Under certain smoothness conditions, the generators of the original and augmented processes are connected as follows (\cref{prop:augmented_generator}):
\begin{equation*}
    \tilde \gL f(t, x) = \lim_{h \to 0^+} \left[\dfrac{f(t+h, x_{t+h}) - f(t, x_t)}{h}\bigg| x_t = x\right] = \partial_t f(t, x) + \gL_t f(t, x).
\end{equation*}

Assuming the existence of the adjoint of the forward generator $\gL_t$ (\cref{def:adjoint}), denoted as $\gL_t^*$, we recall the classical Kolmogorov forward equation (\cref{thm:kolmogorov_forward}):
\begin{equation*}
    \partial_t p_t = \gL_t^* p_t.
\end{equation*}

\subsection{Time-Reversal of Forward Process}

As discussed in~\cref{sec:cont_diffusion,sec:disc_diffusion}, once the forward process $(x_t)_{t\in [0, T]}$ has been constructed with the marginal distributions satisfying $p_0 = p_\text{data}$ and $p_T \approx q_0$, the objective is to construct a corresponding backward process $(y_t)_{t\in [0, T]}$, whose marginal distribution $q_T$ at time $T$ is close to $p_0$.

Due to the Markov property of the forward process $(x_t)_{t\in [0, T]}$, there exists a natural \emph{time-reversal process} $(\cev x_t)_{t\in [0, T]}$ (\cref{def:time_reversal}), which we refer to below as the (true) backward process. Specifically, for each $t \in [0, T)$, we define $\cev x_t = x_{(T-t)^-}$, where $x_{t^-}$ denotes the left limit at time $t$. The time-reversal process thus has the initial distribution $\cev x_0 \sim p_T$ and terminal distribution $\cev x_T := x_0 \sim p_0$. 
It can be verified that the time-reversal process $(\cev x_t)_{t\in[0, T]}$ of a c\'adl\'ag Markov process $(x_t)_{t\in[0, T]}$ remains c\'adl\'ag and Markovian, thereby admitting a generator called the backward generator, denoted by $\cev \gL_t$, defined analogously to the forward generator $\gL_t$ in~\cref{def:right_generator}.

The relationship between the forward generator $\gL_t$ and the backward generator $\cev \gL_t$ is characterized by the following result:
\begin{theorem}[Time-Reversal of Forward Process~\citep{cattiaux2023time}, Informal Version of~\cref{thm:time_reversal_app}]
    Under certain regularity assumptions on the densities $p_t$ (\cref{ass:density}),
    the backward generator $\cev \gL_t$ can be explicitly expressed in terms of the forward generator $\gL_t$ and the density $p_t$ at time $t$ as follows:
    \begin{equation}
        p_t \cev \gL_{T-t} f = p_t \gL_t^*f + \Gamma_t^*(p_t, f),
        \label{eq:backward_generator}
    \end{equation}
    where $\gL_t^*$ is the adjoint operator of $\gL_t$ (\cref{def:adjoint}), and 
    $$
        \Gamma_t^*(p_t, f) := \gL_t^*(p_t f) - p_t \gL_t^* f - f \gL_t^* p_t
    $$ is the \emph{carr\'e du champ operator} associated with the adjoint operator $\gL_t^*$.
    \label{thm:time_reversal}
\end{theorem}

Although \cref{thm:time_reversal} provides an explicit formula connecting the backward generator $\cev \gL_t$ to the forward generator $\gL_t$, the expression inherently involves the unknown density $p_t$ at each time $t$. To overcome this practical limitation, we introduce another Markov process $(y_t)_{t\in [0, T]}$, driven by another generator $\gK_t$. Typically, we restrict the possible choices of the generator $\gK_t$ to a parametrized class of generators that is sufficiently flexible to contain a broad spectrum of generators, especially including the true backward generator $\cev \gL_t$, while remaining computationally tractable and learnable from available data samples from $p_0$ and simulations of the forward generator $\gL_t$. 

Motivated by the particular form involving the carr\'e du champ operator associated with the adjoint generator $\gL_t^*$ in the backward generator $\cev \gL_{T-t}$ given by~\cref{eq:backward_generator}, the following assumption constitutes a central structural hypothesis in the denoising Markov model framework:
\begin{assumption}[Parametrization of Backward Generator, Informal Version of~\cref{ass:K_app}]
    Under certain regularity conditions on $\eta_t$ (\cref{ass:regularity_alpha}), we assume the estimated backward generator $\gK_t$, associated with the estimated backward process $(y_t)_{t \in [0, T]}$, is parametrized by a strictly positive function $\varphi_t: E \to \R^+$ satisfying
    \begin{equation}
        \varphi_t \gK_{T-t} f = \varphi_t \gL_t^* f + \Gamma_t^*(\varphi_t, f).
        \label{eq:K}
    \end{equation}
    In practical implementations, we parametrize either $\varphi_t$ itself or its equivalent coordinates by neural networks, whose parameters are denoted by $\theta$. The corresponding parametrized generator is denoted by $\gK_t^\theta$.
    \label{ass:K}
\end{assumption}

Observe that when the function $\varphi_t$ coincides exactly with the marginal density $p_t$, the estimated backward generator $\gK_t$ reduces to the true backward generator $\cev \gL_t$. Therefore, the strictly positive function $\varphi_t$ intrinsically represents an approximation to the marginal densities.
More generally, defining the density ratio of $\varphi_t$ and $p_t$ by $\eta_t = \varphi_t p_t^{-1}$, and assuming certain regularity conditions thereof (\cref{ass:regularity_alpha}), we can relate the estimated backward generator $\gK_t$ to the true backward generator $\cev \gL_{T-t}$ through the following lemma:
\begin{lemma}[Informal Version of~\cref{prop:K_cev_L}]
    The estimated backward generator $\gK_t$ and the true backward generator $\cev \gL_t$ are related through the following relation:
    \begin{equation*}
        \eta_t \gK_{T-t} f = \eta_t \cev \gL_{T-t} f + \cev \Gamma_{T-t}(\eta_t, f),
    \end{equation*}
    where $\cev \Gamma_t(f,g) =  \cev \gL_t(fg) - f\cev \gL_t g - g \cev \gL_t f$ is the carr\'e du champ operator associated with the time-reversal generator $\cev \gL_t$.
    \label{lem:alpha}
\end{lemma}

This lemma offers a convenient way of understanding and analyzing the parameterized backward generator $\gK_t$ relative to the unknown true backward dynamics.

\subsection{Change of Measure}

The relation between the estimated backward generator $\gK_t$ and the true backward generator $\cev \gL_t$, outlined in~\cref{lem:alpha}, can be viewed as applying a ``perturbation'' of the form $\eta_{T-t}^{-1} \cev \Gamma_t(\eta_{T-t}, \cdot)$ to $\cev \gL_t$. This particular perturbation, involving the carr\'e du champ operator, repeatedly appears in the connection between the backward generator $\cev \gL_t$ and the adjoint operator $\gL_t^*$ (\cref{thm:time_reversal}), as well as in our parametrization for the estimated backward generator $\gK_t$ (\cref{ass:K}). These relationships are summarized in~\cref{fig:denoising}. The question remains as to how this specific form provides a quantitative measure of perturbation, enabling the derivation of practical loss functions and error bounds for denoising Markov models.

As explored in the literature~\citep{kunita1969absolute, ethier2009markov,fleming2006controlled, palmowski2002technique}, this form of perturbation closely relates to the \emph{generalized Doob's $h$-transform}~\citep{chetrite2015nonequilibrium}. The main idea of the transform is summarized in the following theorem:
\begin{theorem}[Generalized Doob's $h$-Transform~\citep{chetrite2015nonequilibrium}]
    Suppose $\gL_t$ is the generator of a Markov process $(x_t)_{t \in [0, T]}$ with path measure $\P$, $h_t$ is a strictly positive and sufficiently smooth function, and $\lambda_t$ is an arbitrary function, for any $t \in [0, T]$. Then the transformed process with the generator $\gL_t^{h, \lambda}$, defined as
    \begin{equation*}
        \gL_t^{h, \lambda} f = h_t^{-1} \gL_t(h_t f) - \lambda_t f,
    \end{equation*}
    is another Markov process with path measure $\P^{h, \lambda}$ absolutely continuous with respect to $\P$, with the Radon-Nikodym derivative satisfying
    \begin{equation}
        \dfrac{\dif \P^{h, \lambda}}{\dif \P}(x_{[0,T]}) = \dfrac{h_T(x_T)}{h_0(x_0)}\exp\left(-\int_0^T \left(h_t^{-1} \partial_t h_t + \lambda_t \right)(x_t)\dif t\right),
        \label{eq:doob}
    \end{equation}
    where $x_{[0,T]}$ denotes the path of the process $(x_t)_{t \in [0, T]}$.

    In particular, choosing $\lambda_t = h_t^{-1} \gL_t h_t$ results in a conservative transformed process, \emph{i.e.}, $\gL_t^{h, \lambda} 1 = 0$, satisfying
    \begin{equation*}
        \gL_t^{h, \lambda} f = h_t^{-1} \gL_t(h_t f) - h_t^{-1} f \gL_t h_t = \gL_t f + h_t^{-1} \Gamma_t(h_t, f),
    \end{equation*}
    where $\Gamma_t(h_t, f) = \gL_t(h_t f) - h_t \gL_t f - f \gL_t h_t$ is the carr\'e du champ operator associated with the generator $\gL_t$.
    \label{thm:doob}
\end{theorem}

The generalized Doob's transform is broadly used in physical contexts, as it is related to stochastic mechanics~\citep{meyer1985construction} and fluctuation-dissipation relations~\citep{chetrite2011two}. It is also explored in the study of the Metropolis algorithm~\citep{diaconis2009characterizations}. For non-perturbative generalization, we refer to~\citet{jarzynski1997nonequilibrium,gallavotti1995dynamical}.

In the denoising Markov model context, \cref{lem:alpha} can be rewritten, after replacing $t$ by $T-t$, as
\[
    \gK_t f
    = \eta_{T-t}^{-1}\cev \gL_t(\eta_{T-t} f)
      - \eta_{T-t}^{-1} f \cev \gL_t \eta_{T-t}.
\]
Therefore, $\gK_t$ is precisely the generalized Doob's $h$-transform of the true backward generator $\cev \gL_t$ with
\[
    h_t = \eta_{T-t},
    \qquad
    \lambda_t = \eta_{T-t}^{-1}\cev \gL_t \eta_{T-t}.
\]
In particular, the approximate backward process may be viewed as a path-space reweighting of the true backward process, with the density ratio $\eta_t$ acting as the reweighting factor. When this reweighting is trivial (\emph{i.e.}, $\eta_t = 1$), the estimated backward generator $\gK_t$ coincides with the true backward generator $\cev \gL_t$.

This observation leads directly to the following change-of-measure identity, which provides a quantitative measure of the perturbation between the estimated backward generator $\gK_t$ and the true backward generator $\cev \gL_t$ in terms of a variational form of the density ratio $\eta_t$. The perturbation should originally involve $\cev \gL_t$ (\emph{cf.}~\cref{thm:change_of_measure_app}) as in~\cref{eq:doob} when perturbing $\gL$ but thanks to \cref{thm:time_reversal}, it can be rewritten in terms of the forward generator $\gL_t$ as follows:
\begin{theorem}[Change of Measure, Informal Version of \cref{cor:change_of_measure}]
    Under certain regularity conditions on $\eta_t$ (\cref{ass:regularity_alpha}), there exists a probability measure $\sQ$ absolutely continuous to the original probability measure $\P$ with the Radon-Nikodym derivative satisfying that
    \begin{equation}
        \KL(\P \| \sQ) = \E_{\P}\left[\log\dfrac{\dif \P}{\dif \sQ}\right] = \E_{\P}\left[ \int_0^T \left( \eta_t \gL_t \eta_t^{-1} + \gL_t \log \eta_t \right)(x_t) \dif t\right] := \mathfrak{L}[\eta_t],
        \label{eq:change_of_measure}
    \end{equation}
    under which the time-reversal process $(\cev x_t)_{t\in[0, T]}$ is governed by the generator $\gK_t$, where the expectation is taken over the forward process $(x_t)_{t\in[0, T]}$.
    \label{thm:change_of_measure}
\end{theorem}
We will omit the subscript $\P$ indicating the expectation over the forward process $(x_t)_{t\in[0, T]}$ in the following discussion for brevity.
The result in~\cref{thm:change_of_measure} directly yields error bounds for generative models in terms of KL divergence, as shown by the following corollary by applying the data-processing inequality and the chain rule of KL divergence:
\begin{corollary}
    The following error bound holds for the denoising Markov model:
    \begin{equation*}
        \KL(p_0 \| q_T) \leq \KL(p_T \| q_0) + \mathfrak{L}[\eta_t],
    \end{equation*}
    where $\mathfrak{L}[\eta_t]$ is defined as in~\cref{eq:change_of_measure}.
    \label{cor:error_bound}
\end{corollary}

This result highlights that the quantity $\mathfrak{L}[\eta_t]$ not only quantifies the perturbation between the estimated backward generator $\gK_t$ and the true backward generator $\cev{\gL}_t$ but also naturally serves as an appropriate loss function for training denoising Markov models. By minimizing this loss, we directly obtain an error bound on the generative model, measured by the KL divergence between the target distribution $p_0$ and the estimated distribution $q_T$, up to an additional term $\KL(p_T \| q_0)$ that typically decays exponentially and reflects the convergence of the forward process.

To facilitate practical and efficient training, the marginal densities $p_t(x_t)$ can be expressed instead in terms of the conditional distributions $p_{t|0}(x_t|x_0)$, which are often simpler to compute or available in closed-form solutions. This substitution leads to a score-matching objective tailored for general denoising Markov models, as captured in the following corollary, extending the classical score-matching framework~\citep{hyvarinen2005estimation, vincent2011connection}.
\begin{corollary}[Score-Matching]
    Define the conditional density ratio as 
    \begin{equation}
        \eta_{t|0}(\cdot|x_0) = \varphi_t(\cdot) p_{t|0}^{-1}(\cdot|x_0), \label{eq:eta_t_0}
    \end{equation}
    where $p_{t|0}(\cdot|x_0)$ denotes the conditional distribution of $x_t$ given $x_0$. 
    Then, the score-matching loss is given by
    \begin{equation}
        \mathfrak{L}_\mathrm{SM}[\eta_{t|0}] = \E_{x_0 \sim p_0} \left[ \int_0^T \E_{x_t \sim p_{t|0}(\cdot| x_0)}\left[\left(\eta_{t|0}(\cdot|x_0) \gL_t\eta_{t|0}^{-1}(\cdot|x_0) + \gL_t \log \eta_{t|0}(\cdot|x_0)\right)(x_t) \right] \dif t\right], 
        \label{eq:score_matching}
    \end{equation}
    and when the function $\eta_t$ depends on the parameters $\theta$ through the parametrization of $\varphi_t$ by its neural network representation, the score-matching objective satisfies that
    \begin{equation*}
        \argmin_{\theta}\KL(\P \| \sQ^\theta) = \argmin_\theta \mathfrak{L}[\eta_t^\theta] = \argmin_\theta \mathfrak{L}_{\mathrm SM}[\eta_{t|0}^\theta],
    \end{equation*}
    where constants independent of the model parameters $\theta$ are omitted.
    \label{cor:score_matching}
\end{corollary}

In general, due to the non-negativity of the integrands of $\mathfrak{L}[\eta_t]$~\eqref{eq:change_of_measure} and $\mathfrak{L}_\mathrm{SM}[\eta_{t|0}]$~\eqref{eq:score_matching} (\emph{cf.}~\citep[Proposition 1(a)]{benton2024denoising}), one may reweight the integrands by a function $\psi_t$ to mitigate the concentration of the loss in regions where its magnitude is large\citep{song2020score}, similar to the approach in~\cref{eq:continuous_diffusion_loss,eq:discrete_diffusion_loss}. Practically, this reweighting can be implemented by sampling time $t$ from a non-uniform distribution $\Psi$ over $[0, T]$ during training (\emph{cf.}~\cref{alg:training}), resulting in the following time-reweighted loss:
\begin{equation*}
    \mathfrak{L}_{\mathrm{SM},\Psi}[\eta_{t|0}] = \E_{x_0 \sim p_0, t \sim \Psi} \left[  \E_{x_t \sim p_{t|0}(\cdot| x_0)}\left[\left(\eta_{t|0}(\cdot|x_0) \gL_t\eta_{t|0}^{-1}(\cdot|x_0) + \gL_t \log \eta_{t|0}(\cdot|x_0)\right)(x_t) \right]\right].
\end{equation*}

\begin{remark}[Comparison with~\citet{benton2024denoising}]
    We briefly compare our proposed framework and the approach presented by~\citet{benton2024denoising}. While our framework relies on a single explicit assumption (\cref{ass:K}), \citet{benton2024denoising} propose two independent assumptions concerning the parametrization of the estimated backward generator $\gK_t$, which, in our notation, can be expressed as follows:
    \begin{enumerate}[label=(\arabic*)]
        \item For each $t \in [0, T]$, there exists an auxiliary Markov process $(z_t)_{t\in[0, T]}$ governed by a Feller evolution system with generator $\gM_t$, together with a function $c_t$, such that 
        $$
            \gK_{T-t}^* p_t = \gM_{T-t} \cev p_t + c_{T-t} p_t;
        $$
        \item For each $t \in [0, T]$, there exists a strictly positive function $\beta_t$ satisfying
        \begin{equation*}
            \beta_t^{-1} \gM_{T-t} f = \gL_t(\beta_t^{-1} f) - f \gL_t \beta_t^{-1}.
        \end{equation*}
    \end{enumerate}
    Under these assumptions,\citet{benton2024denoising} apply a change-of-measure argument between the original process $(x_t)_{t\in[0, T]}$ and the auxiliary process $(z_t)_{t\in[0, T]}$, and subsequently invoke the integral-by-parts formula for semimartingales and the Feynman-Kac formula to handle the additional term $c_t$.

    In contrast, our simpler and more explicit assumption (\cref{ass:K}) inherently satisfies the assumptions above. Specifically, using \cref{eq:K} and the definition of the adjoint operator, we have
    \begin{equation*}
        \gK_{T-t}^* g 
        = \varphi_t \gL_t(\varphi_t^{-1} g) - \varphi_t^{-1} g \gL_t^* \varphi_t 
        = \varphi_t \gL_t(\varphi_t^{-1} g) - \varphi_t g \gL_t \varphi_t^{-1} + g \left(\varphi_t \gL_t \varphi_t^{-1} - \varphi_t^{-1}  \gL_t^* \varphi_t\right),
    \end{equation*}
    which corresponds precisely to their assumptions with the choices
    $$
    \beta_t = \varphi_t,\quad \gM_{T-t} g = \varphi_t \gL_t(\varphi_t^{-1} g) - \varphi_t g \gL_t \varphi_t^{-1}, \quad \text{and} 
        \quad c_{T-t} = \varphi_t \gL_t \varphi_t^{-1} - \varphi_t^{-1} \gL_t^* \varphi_t.
    $$ 
    Moreover, while the choice of $c_t$ is left implicit in the approach of~\citet{benton2024denoising}, we explicitly characterize $c_t$ in our framework and demonstrate via~\cref{lem:alpha,thm:doob} that this selection is uniquely determined to ensure the conservativeness of the estimated backward process $(y_t)_{t \in [0, T]}$. By clearly specifying $c_t$, our framework provides a direct and explicit representation of the estimated backward generator $\gK_t$, thereby enabling a broader class of denoising Markov models with arbitrary Feller forward processes, as elaborated in~\cref{sec:examples} and verified by several examples in~\cref{sec:experiments}.

    Furthermore, our analysis reveals explicit and intuitive relationships between the forward generator $\gL_t$, the true backward generator $\cev \gL_t$, and the estimated backward generator $\gK_t$, as illustrated clearly in~\cref{fig:denoising}. By interpreting these relationships through a generalized Doob's $h$-transform (\cref{thm:time_reversal,ass:K,lem:alpha}), our framework avoids the necessity of introducing an auxiliary process $(z_t)_{t\in[0, T]}$. Instead, we directly employ a streamlined change-of-measure argument (\cref{thm:change_of_measure}) between the true backward process $(\cev x_t)_{t\in[0, T]}$ and its estimate $(y_t)_{t\in[0, T]}$. This simplification significantly clarifies the analysis, enhances the conceptual transparency of the method, and provides clearer insights for practical implementation in designing denoising Markov models.
\end{remark}

\subsection{Meta-Algorithm}

We now introduce the meta-algorithm for denoising Markov models and outline several key principles underlying their design.

\paragraph{Training Process.} The training process (\cref{alg:training}) of denoising Markov models typically relies on four crucial components: (1) an empirical distribution $\hat p_0$ approximating the target distribution $p_0$, (2) a simple, easy-to-sample reference distribution to which the forward process rapidly converges after time horizon $T$, (3) an explicit formula for the conditional distribution $p_{t|0}(\cdot | x_0)$, and (4) an efficient method for simulating the forward process $(x_t)_{t\in[0, T]}$ generated by $\gL_t$. 

In generative modeling scenarios, the first requirement is naturally fulfilled by available data samples. However, the remaining conditions depend on careful choices of the forward process. Choosing the transition kernel $p_{t|0}$ to be space-homogeneous and even further time-homogeneous is common in practice to simplify its computation. We provide concrete examples addressing these criteria in~\cref{sec:examples}.

\IncMargin{1em}
\begin{algorithm}[!ht]
	\caption{Training Procedure of Denoising Markov Model}
	\label{alg:training}
	\Indm
	\KwIn{Empirical distribution $\hat p_0$ of the data distribution $p_0$, the conditional probability $p_{t|0}(\cdot | x_0)$, number of total epochs $N$, batch size $B$, time distribution $\Psi$ on $[0, T]$.}
	\KwOut{The estimated backward generator $\gK_t^{\theta}$ depending on model parameters $\theta$ through $\varphi_t$ represented by a neural network.}
	\Indp
	\For{$n = 0$ \KwTo $N-1$}{
		Sample samples $\{x_0^{(k)}\}_{k=1}^B \sim \hat p_0$ and time points $\{t^{(k)}\}_{k=1}^B \sim \Psi$
        \For{$k = 1$ \KwTo $B$}{
            Sample $x_t^{(k)} \sim p_{t^{(k)}|0}(\cdot | x_0^{(k)})$ by simulating the forward process $(x_t)_{t\in[0, T]}$ with the generator $\gL_{t^{(k)}}$
        }
        Compute the empirical average of the integrand in~\cref{eq:score_matching} as 
        \begin{equation*}
            \begin{aligned}
                &\hat{\mathfrak{L}}_{\mathrm{SM},\Psi}[\eta_{t|0}^\theta] =\\
                &\dfrac{1}{B} \sum_{k=1}^B \left(
                    \eta_{t^{(k)}|0}^\theta\left(\cdot\big|x_0^{(k)}\right) \gL_{t^{(k)}} \eta_{t^{(k)}|0}^\theta\left(\cdot\big|x_0^{(k)}\right)^{-1} 
                    + \gL_{t^{(k)}} \log \eta_{t^{(k)}|0}^\theta\left(\cdot \big|x_0^{(k)}\right)
                \right)\left(x_t^{(k)}\right);
            \end{aligned}
        \end{equation*}
        Update the parameters $\theta$ by gradient descent as $\theta \leftarrow \theta - \epsilon \nabla \hat{\mathfrak{L}}_{\mathrm{SM},\Psi}[\eta_{t|0}^\theta]$
	}
\end{algorithm}

\paragraph{Inference Process.} Inference for denoising Markov models is straightforward in principle: given the explicit form of the estimated backward generator $\gK_t^\theta$ (specific examples again appear in~\cref{sec:examples}), we simulate the estimated backward process $(y_t)_{t \in [0, T]}$. Starting from the initial distribution $q_0$, samples are generated by numerically discretizing and simulating this continuous-time process.

Since practical implementations involve discretization, controlling the resulting numerical error becomes crucial for the model's performance. 
To quantify the numerical error, we consider a uniform discretization of the interval $[0, T]$ into time steps of length $\kappa$, assuming $T = L\kappa$ for some integer $L$. For $\ell \in [L]$, let $\hat q_{\ell\kappa}$ denote the distribution of samples after $\ell$ discretization steps. The inference procedure for the discretized denoising Markov model is detailed in~\cref{alg:inference}. We impose the following assumption on the numerical error introduced at each discretization step:

\begin{algorithm}[!ht]
	\caption{Inference Procedure of Denoising Markov Model}
	\label{alg:inference}
	\Indm
	\KwIn{The estimated backward generator $\gK_t^{\theta}$ depending on model parameters $\theta$ through $\varphi_t$ represented by a neural network, a simple, easy-to-sample distribution $q_0 \approx p_T$.}
	\KwOut{A sample $\hat y_T$ from $\hat q_T \approx p_0$.}
	\Indp
	Sample $\hat y_0 \sim \hat q_0 = q_0$
    \For{$\ell =0$ \KwTo $L-1$}{
        Sample $\hat y_{(\ell+1)\kappa}$ from the distribution $\hat q_{(\ell+1)\kappa|\ell\kappa}(\cdot|\hat y_{\ell\kappa})$ by the numerical method satisfying~\cref{ass:one_step_simulation}
    }
\end{algorithm}

In the following theorem, we consider a general case where the time horizon $T$ is discretized into time steps of equal length $\kappa$. For simplicity, we assume that $T = L \kappa$ for some integer $L$. For any $\ell \in [L]$, we denote the distribution of obtained samples after $\ell$ steps of simulation by $\hat q_{\ell\kappa}$. We refer to~\cref{alg:inference} for the inference procedure of the denoising Markov model. Within each interval $[\ell\kappa, (\ell+1)\kappa)$, we assume the existence of a numerical method that simulates the estimated backward process $(\hat y_t)_{t\in[\ell\kappa, (\ell+1)\kappa]}$ with a bounded one-step numerical error. 
\begin{assumption}[One-Step Numerical Error]
    For each $\ell \in [0, L-1]$, there exists an algorithm for simulating the estimated backward generator $\gK_t^\theta$. Given $\hat y_{\ell\kappa}$, this algorithm induces a conditional distribution $\hat q_{(\ell+1)\kappa|\ell\kappa}(\cdot|\hat y_{\ell\kappa})$ approximating the estimated backward process $(y_t)_{t\in[\ell\kappa, (\ell+1)\kappa]}$ and the corresponding conditional distribution $q_{(\ell+1)\kappa|\ell\kappa}(\cdot|\hat y_{\ell\kappa})$. 
    
    Moreover, we assume that the numerical error of the algorithm within each interval $[\ell\kappa, (\ell+1)\kappa]$ is of order $\gO(\kappa^{1+r})$, \emph{i.e.}, the following bound holds:
    \begin{equation*}
        \E_{\cev x_{\ell\kappa} \sim \cev p_{\ell\kappa}, \cev x_{(\ell+1)\kappa} \sim \cev p_{(\ell+1)\kappa}}\left[\log \dfrac{q_{(\ell+1)\kappa|\ell\kappa}(\cev x_{(\ell+1)\kappa}|\cev x_{\ell\kappa})}{\hat q_{(\ell+1)\kappa|\ell\kappa}(\cev x_{(\ell+1)\kappa}|\cev x_{\ell\kappa})} \right] \lesssim \kappa^{1+r},
    \end{equation*}
    for some $r > 0$, where we write $a\lesssim b$ to mean $a\le Cb$ for a constant $C>0$ that may depend on fixed,
    time-uniform structural parameters of the model class (dimension, conditioning, regularity),
    but is independent of the step size $\kappa$, the time horizon $T$, and any realized path.
    \label{ass:one_step_simulation}
\end{assumption}

Under this assumption, we obtain the following meta-error bound:
\begin{theorem}[Meta-Error Bound for Denoising Markov Models]
    Under~\cref{ass:one_step_simulation}, the distribution $\hat q_T$ of the samples generated by the inference algorithm (\cref{alg:inference}) satisfies:
    \begin{equation}
        \KL(p_0 \| \hat q_T) \lesssim \KL(p_T \| q_0) + \mathfrak{L}(\eta_t^\theta) + T \kappa^r,
        \label{eq:meta_error_bound}
    \end{equation}
    where $\mathfrak{L}(\eta_t^\theta)$ is defined in~\cref{thm:change_of_measure}.
    \label{thm:numerical_error}
\end{theorem}

The error bound~\eqref{eq:meta_error_bound} reveals three main sources of error in the model:
\begin{enumerate}[label=(\arabic*)]
    \item {\bfseries Truncation Error $\KL(p_T \| q_0)$}: Results from approximating the forward process distribution $p_T$ at the terminal time $T$ by a simpler distribution $q_0$. If the forward process converges exponentially, as in certain cases of the continuous and discrete diffusion models discussed in~\cref{sec:cont_diffusion,sec:disc_diffusion}, this term typically becomes negligible for sufficiently large $T$.
    \item {\bfseries Estimation Error $\mathfrak{L}(\eta_t^\theta)$}: Corresponds exactly to the KL divergence between the path measure of the true backward process and that of its estimate as introduced in~\cref{thm:change_of_measure}. This quantity measures how well $\varphi_t$ aligns with the true marginal distributions $p_t$ and can be directly minimized during training (\cref{alg:training}).
    \item {\bfseries Numerical Error $T \kappa^r$}: Represents cumulative numerical errors arising from the discrete approximation of the continuous-time backward process. This term heavily depends on the choice of numerical method and the underlying Markov process. Standard Euler-Maruyama methods typically yield $r=1$. For general L\'evy processes introduced in \cref{sec:general_levy}, the constant factor before the scaling $T\kappa^r$ depends on the dimension $d$, the conditioning, uniform regularity, and moment bounds on the coefficients of the L\'evy triplet $(\cev\vb_t^\varphi, \cev\mD_t^\varphi, \cev\lambda_t^\varphi)$. Practitioners may choose adaptive discretization strategies based on the data distribution's characteristics to reduce this error. For relevant discussions, we refer readers to~\citet{chen2023improved,benton2023nearly} (continuous-time diffusion models) and~\citet{chen2024convergence,ren2024discrete} (discrete-time diffusion models). 
\end{enumerate}

\section{Examples}
\label{sec:examples}

In this section, we illustrate our framework by applying the theoretical results developed earlier to several important classes of Markov processes. Specifically, we revisit diffusion processes and jump processes, which recover continuous and discrete diffusion models as special cases. We then extend our analysis to general Lévy-type processes, revealing the broad applicability of our denoising Markov model framework. For clarity and brevity, we provide main results and intuitive explanations here, deferring rigorous mathematical details and proofs to~\cref{app:proofs}.

\subsection{Diffusion Process}
\label{sec:diffusion}

We first consider diffusion processes, where the underlying state space is the Euclidean space $\R^d$, and the generator acts on twice-differentiable functions, \emph{i.e.}, functions in $C^2(\R^d)$.

\paragraph{Forward Process.} We assume that the forward process $(\vx_t)_{t\in [0, T]}$ is a diffusion process governed by the following SDE:
\begin{equation}
    \dif \vx_t = \vb_t(\vx_t) \dif t + \mSigma_t(\vx_t) \dif \vw_t,
    \label{eq:forward_diffusion}
\end{equation}
where $\vb_t(\vx)$ is the drift, $(\vw_t)_{t \geq 0}$ is the Wiener process, and the diffusion coefficient $\mSigma_t(\vx)$ satisfies $\mSigma_t(\vx) \mSigma_t^\top(\vx) = \mD_t(\vx)$.  Regularity conditions on the coefficients $\vb_t(\vx)$ and $\mD_t(\vx)$ are detailed in~\cref{ass:regularity_diffusion}.

The generator of the diffusion process $(\vx_t)_{t\in [0, T]}$ is given by
\begin{equation*}
    \gL_t f = \vb_t(\vx) \cdot \nabla f + \dfrac{1}{2} \mD_t(\vx) : \nabla^2 f,
\end{equation*}
where we adopt the notations $\nabla^2 = \nabla \nabla^\top$ and $\mD_t(\vx) : \nabla^2 f = \textstyle \sum_{i,j=1}^d D_t^{ij}(\vx) \partial_{ij} f$.

\paragraph{Backward Process.} 

Given sufficient regularity (\cref{ass:regularity_diffusion}),~\cref{thm:time_reversal} shows that the backward generator $\cev\gL_t$ can be explicitly obtained as:
\begin{equation*}
    \cev \gL_{T-t} f = \left(-\vb_t + \mD_t \nabla \log p_t + \nabla \cdot \mD_t \right) \cdot \nabla f + \dfrac{1}{2} \mD_t: \nabla^2 f,
\end{equation*}
while our assumption (\cref{ass:K}) leads to a similar structure for the estimated backward generator:
\begin{equation*}
    \gK_{T-t} f = \left(-\vb_t + \mD_t \nabla \log \varphi_t + \nabla \cdot \mD_t \right) \cdot \nabla f + \dfrac{1}{2} \mD_t: \nabla^2 f,
\end{equation*} 
implying that the estimated backward process $(\vy_t)_{t\in[0,T]}$ is itself a diffusion process, described by the following SDE:
\begin{equation}
    \dif \vy_t = \left(-\vb_{T-t}(\vy_t) + \mD_{T-t}(\vy_t) \hat \vs_{T-t}(\vy_t) + \nabla \cdot \mD_{T-t}(\vy_t) \right) \dif t + \mSigma_{T-t}(\vy_t) \dif \vw_t,
    \label{eq:backward_diffusion}
\end{equation} 
where $\hat\vs_t^\theta = \nabla\log\varphi_t$ estimates the true score $\vs_t = \nabla\log p_t$.

\paragraph{Loss Function.}

Applying~\cref{thm:change_of_measure}, the KL divergence between the true and estimated backward processes is given by:
\begin{equation}
    \mathfrak{L}[\hat\vs_t] = \E\left[\int_0^T \dfrac{1}{2} \mD_t(\vx_t) : \left(\hat \vs_t(\vx_t) - \nabla \log p_t(\vx_t)\right) \left(\hat \vs_t(\vx_t) - \nabla \log p_t(\vx_t)\right)^\top \dif t\right],
    \label{eq:diffusion_loss}
\end{equation} 
and the corresponding score-matching loss replaces $\nabla\log p_t$ by the conditional score $\nabla\log p_{t|0}$.

\begin{example}[Continuous Diffusion Model]
    \label{ex:continuous_diffusion}
    Choosing the diffusion coefficient $\mD_t(\vx)=\mI_d$ simplifies the score-matching loss to the familiar form from continuous diffusion models as in~\cref{eq:continuous_diffusion_forward}:
    \begin{equation}
        \mathfrak{L}_\mathrm{SM}[\hat\vs_t] = \E_{\vx_0 \sim p_0}\bigg[\int_0^T \E_{\vx_t \sim p_{t|0}(\cdot| \vx_0)}\bigg[\dfrac{1}{2} \left\|\hat\vs_t(\vx_t) - \nabla \log p_{t|0}(\vx_t|\vx_0)\right\|^2 \bigg] \dif t\bigg],
        \label{eq:continuous_diffusion_loss_2}
    \end{equation}
    matching exactly~\cref{eq:continuous_diffusion_loss} up to time-reweighting.
\end{example}

\subsection{Jump Process with Finite State Spaces}
\label{sec:jump}

We now discuss finite-state jump processes, represented by continuous-time Markov chains on a finite state space $\sX$, whose size we denote by $|\sX|$.

\paragraph{Forward Process.} 
The generator of the jump process on $\sX$ is given by
\begin{equation*}
    \gL_t f(x) = \sum_{y \in \sX} (f(y) - f(x)) \lambda_t(y, x),
\end{equation*}
where $\lambda_t(y,x)$ represents jump rates from $x$ to $y$ at time $t$, with $\lambda_t(x,x)=0$ for all $x\in\sX$.
Then it is a classical result that the forward Markov chain $(x_t)_{t\in[0,T]}$ evolves according to~\eqref{eq:discrete_diffusion_forward} and one may verify that the jump rate $\lambda_t(y,x)$ matches that defined with the rate matrix $\mLambda_t$ in~\cref{sec:disc_diffusion}.

\paragraph{Backward Process.} \cref{ass:K} reduces to the following form: 
\begin{equation*}
    \gK_{T-t} f = \sum_{y \in \sX} (f(y) - f(x)) \dfrac{\varphi_t(y)}{\varphi_t(x)} \lambda_t(x, y)= \sum_{y \in \sX} (f(y) - f(x)) \hat s_t(x, y) \lambda_t(x, y),
\end{equation*}
with estimated score function $\hat s_t(x,y)=\varphi_t(y)/\varphi_t(x)$ approximating the true score $s_t(x,y)=p_t(y)/p_t(x)$. Thus, the estimated backward process $(y_t)_{t \in [0, T]}$ is also a continuous-time Markov chain with the form~\eqref{eq:discrete_diffusion_backward}.

\paragraph{Loss Function.} \cref{thm:change_of_measure} now corresponds to
\begin{equation}
    \mathfrak{L}[\hat s_t] = \E\left[\int_0^T \sum_{y \in \sX} \left(\frac{\hat s_t(x_t, y)}{s_t(x_t, y)} -1 - \log \frac{\hat s_t(x_t, y)}{s_t(x_t, y)} \right) s_t(x_t, y) \lambda_t(x_t, y) \dif t\right].
    \label{eq:jump_loss}
\end{equation}
We adopt the convention $\lambda_t(x, x) = 0$ and therefore the summation is effectively excluding the case where $y = x_t$. We will also adopt this notation for brevity in the following discussions.

\begin{example}[Discrete Diffusion Model]
    \label{ex:discrete_diffusion}
    As shown in~\cref{app:jump_process}, the score-matching objective~\eqref{eq:score_matching} in this case matches exactly the loss of discrete diffusion models (\cref{eq:discrete_diffusion_loss}) up to time reweighting.
\end{example}

\subsection{General L\'evy-Type Process}
\label{sec:general_levy}

We now investigate the most general setting for denoising Markov models, focusing on Lévy-type processes defined on $\R^d$. Lévy-type processes generalize both diffusion and jump processes and include a rich class of stochastic dynamics.

\paragraph{Forward Process.} Under suitable assumptions on the generator (\cref{ass:smooth}), which ensure that the generator is sufficiently regular, \emph{i.e.}, the domain of the augmented generator contains all compactly supported smooth functions, we characterize the forward generator explicitly using the L\'evy-Khintchine representation:
\begin{theorem}[General Form of Forward Process, Informal Version of~\cref{thm:courrege_form}]
    Under~\cref{ass:smooth}, the generator $\gL_t$ of the forward process $(x_t)_{t\in[0, T]}$ can be represented in the following generalized L\'evy-Khintchine form:
    \begin{equation}
        \begin{aligned}
            \gL_t f(\vx) =& \vb_t(\vx) \cdot \nabla f(\vx) + \dfrac{1}{2} \mD_t(\vx) : \nabla^2 f(\vx) \\
            &+ \int_{\R^d \backslash \{\vx\}} \left(f(\vy) - f(\vx) - (\vy - \vx) \cdot \nabla f(\vx) \vone_{B(\vx, 1)}(\vy - \vx)\right) \lambda_t(\dif \vy, \vx),
        \end{aligned}
        \label{eq:generator_levy_type}
    \end{equation}
    where $\vb_t(\vx)$ is the drift term, $\mD_t(\vx)$ is the diffusion coefficient, which is positive semidefinite for any $\vx \in \R^d$, and $\lambda_t(\dif \vy, \vx)$ is a L\'evy measure, \emph{i.e.}, a Borel measure on $\R^d \backslash \{\vx\}$ satisfying
    \begin{equation*}
        \int_{\R^d \backslash \{\vx\}} \left(1 \wedge \|\vy-\vx\|^2\right) \lambda_t(\dif \vy, \vx) < +\infty,
    \end{equation*}
    for any $\vx \in \R^d$. The triple $(\vb_t,\mD_t,\lambda_t)$ is known as the Lévy triplet.
    \label{thm:general_levy}
\end{theorem}

Note that if the coefficients $(\vb_t,\mD_t,\lambda_t)$ are time-invariant, this reduces to the classical L\'evy-Khintchine form. 
By similar arguments as for~\citep[Theorem~3.33]{ccinlar1981representation}, the generator $\gL_t$ of the general L\'evy-type process $(\vx_t)_{t \in \R}$ admits the stochastic integral representation (\cref{prop:stochastic_integral_general_levy}):
\begin{equation*}
    \begin{aligned}
        \vx_t &= \vx_0 + \int_0^t \vb_s(\vx_s) \dif s + \int_0^t \mSigma_s(\vx_s) \dif \vw_s \\
        &+ \int_0^t \int_{\R^d \backslash B(\vx_{s^-}, 1)} (\vy - \vx_{s^-}) N[\lambda](\dif s, \dif \vy) + \int_0^t \int_{B(\vx_{s^-}, 1)} (\vy - \vx_{s^-}) \tilde N[\lambda](\dif s, \dif \vy),
    \end{aligned}
\end{equation*}
where the matrix $\mSigma_s(x)$ satisfies $\mSigma_s(x) \mSigma_s(x)^\top = \mD_s(x)$ for any $x \in \R^d$, $(\vw_s)_{s \geq 0}$ is a $d$-dimensional Wiener process, $N[\lambda](\dif s, \dif \vy)$ denotes the Poisson random measure with intensity $\lambda_s(\dif \vy, \vx_{s^-})$, and $\tilde N[\lambda](\dif s, \dif \vy)$ is the compensated version of $N[\lambda](\dif s, \dif \vy)$, \emph{i.e.},
\begin{equation*}
    \tilde N[\lambda](\dif s, \dif \vy) = N[\lambda](\dif s, \dif \vy) - \lambda_s(\dif \vy, \vx_{s^-}) \dif s.
\end{equation*}

\paragraph{Backward Process.} Under further regularity conditions (\cref{ass:regularity_levy}), notably, particularly assuming that the L\'evy measure $\lambda_t(\dif \vy, \vx)$ admits a density, \emph{i.e.}, $\lambda_t(\dif \vy, \vx) = \lambda_t(\vy, \vx) \dif \vy$, one may explicitly compute the backward generator $\cev \gL_t$, yielding the same form as~\cref{eq:generator_levy_type}:
\begin{equation}
    \begin{aligned}
        \cev \gL_t f(\vx) =& \cev \vb_t(\vx) \cdot \nabla f(\vx) + \dfrac{1}{2} \cev \mD_t(\vx) : \nabla^2 f(\vx) \\
        &+ \int_{\R^d \backslash \{\vx\}} \left(f(\vy) - f(\vx) - (\vy - \vx) \cdot \nabla f(\vx) \vone_{B(\vx, 1)}(\vy - \vx)\right) \cev \lambda_t(\dif \vy, \vx),
    \end{aligned}
    \label{eq:backward_generator_general_levy}
\end{equation}
with modified L\'evy triplet $(\cev\vb_t,\cev\mD_t,\cev\lambda_t)$ explicitly given by:
\begin{gather*}
    \cev \vb_{T-t}(\vx) = -\vb_t(\vx) + \mD_t(\vx) \nabla \log p_t(\vx) + \nabla \cdot \mD_t(\vx) + \gI[p_t](\vx),\\
    \cev \mD_{T-t}(\vx) = \mD_t(\vx),\quad \text{and}\quad \cev \lambda_{T-t}(\vy, \vx) = \dfrac{p_t(\vy)}{p_t(\vx)} \lambda_t(\vx, \vy),
\end{gather*}
where the integral term $\gI[p_t]$ accounts for asymmetric jumps:
\begin{equation*}
    \gI[p_t](\vx) = \int_{B(\vx, 1)} (\vy - \vx) \left(\dfrac{p_t(\vy)}{p_t(\vx)} \lambda_t(\vx, \vy) + \lambda_t(\vy, \vx)\right) \dif \vy.
\end{equation*}

Under our estimation assumption (\cref{ass:K}), the estimated backward generator $\gK_t$ shares the same general L\'evy-type form~\eqref{eq:generator_levy_type} but with the estimated L\'evy triplet $(\cev\vb^\varphi_t,\cev\mD^\varphi_t,\cev\lambda^\varphi_t)$:
\begin{gather*}
    \cev \vb^\varphi_{T-t}(\vx) = -\vb_t(\vx) + \mD_t(\vx) \nabla \log \varphi_t(\vx) + \nabla \cdot \mD_t(\vx) + \gI[\varphi_t](\vx)\\
    \cev \mD^\varphi_{T-t}(\vx) = \mD_t(\vx),\quad \text{and}\quad \cev \lambda^\varphi_{T-t}(\vy, \vx) = \dfrac{\varphi_t(\vy)}{\varphi_t(\vx)} \lambda_t(\vx, \vy).
\end{gather*}

We remark that the integral $\gI[\cdot]$ should be understood as a \emph{bona fide} integral or in the Cauchy principal value sense, depending on the singularity of $\lambda_t(\vy,\vx)$ (see~\citet[Theorem 5.7]{conforti2022time} for sufficient conditions and~\cref{ex:alpha_stable} for an explicit example).

\paragraph{Loss Function.} 
To practically implement this model, we define the diffusion and jump score functions separately as:
\begin{equation*}
    \vs_t^{\mathrm{diff}}(\vx) = \nabla \log p_t(\vx), \quad \text{and} \quad \hat \vs_t^{\mathrm{diff}}(\vx) = \nabla \log \varphi_t(\vx),
\end{equation*}
and 
\begin{equation}
    s_t^{\mathrm{jump}}(\vx, \vy) = \dfrac{p_t(\vy)}{p_t(\vx)}, \quad \text{and} \quad \hat s_t^{\mathrm{jump}}(\vx, \vy) = \dfrac{\varphi_t(\vy)}{\varphi_t(\vx)},
    \label{eq:jump_score}
\end{equation}
respectively. Using these definitions, the KL divergence between the backward processes decomposes naturally into diffusion and jump parts, yielding the unified loss function:
\begin{equation}
    \begin{aligned}
        \mathfrak{L}[\hat \vs_t^{\mathrm{diff}}, &\hat s_t^{\mathrm{jump}}] = \E\Bigg[\\
        &\int_0^T \dfrac{1}{2} \mD_t(\vx_t) : \left(\hat \vs_t^{\mathrm{diff}}(\vx_t) - \nabla \log p_t(\vx_t)\right) \left(\hat \vs_t^{\mathrm{diff}}(\vx_t) - \nabla \log p_t(\vx_t)\right)^\top \dif t\\ 
        +& \int_0^T\int_{\R^d} \left(\frac{\hat s_t^{\mathrm{jump}}(\vx_t, \vy)}{s_t^{\mathrm{jump}}(\vx_t, \vy)} -1 - \log \frac{\hat s_t^{\mathrm{jump}}(\vx_t, \vy)}{s_t^{\mathrm{jump}}(\vx_t, \vy)} \right) s_t^{\mathrm{jump}}(\vx_t, \vy) \lambda_t(\vx_t, \vy) \dif \vy \dif t\Bigg],
    \end{aligned}
    \label{eq:general_levy_loss}
\end{equation}
which is exactly the sum of the KL divergence~\eqref{eq:diffusion_loss} derived from the discussion of diffusion processes and that of the jump process~\eqref{eq:jump_loss}, and the corresponding score-matching loss is thus
\begin{equation}
    \begin{aligned}
        \mathfrak{L}_\mathrm{SM}[\hat \vs_t^{\mathrm{diff}}, &\hat s_t^{\mathrm{jump}}]
        = \E_{\vx_0 \sim p_0}\Bigg[\int_0^T \E_{\vx_t \sim p_{t|0}(\cdot| \vx_0)} \Bigg[ \\
        &\dfrac{1}{2} \mD_t(\vx_t) : \left(\hat \vs_t^{\mathrm{diff}}(\vx_t) - \nabla \log p_{t|0}(\vx_t|\vx_0)\right) \left(\hat \vs_t^{\mathrm{diff}}(\vx_t) - \nabla \log p_{t|0}(\vx_t|\vx_0)\right)^\top\\ 
        &+ \int_{\R^d} \left(\hat s_t^{\mathrm{jump}}(\vx_t, \vy) - \dfrac{p_{t|0}(\vy|\vx_0)}{p_{t|0}(\vx_t|\vx_0)} \log \hat s_t^{\mathrm{jump}}(\vx_t, \vy)\right) \lambda_t(\vx_t, \vy) \dif \vy \Bigg] \dif t\Bigg].
    \end{aligned}
    \label{eq:general_levy_score_matching}
\end{equation}

\begin{example}[L\'evy-It\^o Model~\citep{yoon2023score}]
    \label{ex:alpha_stable}
    Consider the special case where the diffusion component vanishes, i.e., $\mD_t(\vx) \equiv 0$, and the L\'evy measure $\lambda_t(\vy,\vx) \dif \vy$ is chosen as a time-homogeneous and isotropic $\alpha$-stable L\'evy measure with the following form:
    \begin{equation}
        \lambda_t(\vy, \vx) \dif \vy = \dfrac{\alpha 2^{\alpha-1} \pi^{-d/2} \Gamma(\frac{\alpha+d}{2})}{\Gamma(1-\frac{\alpha}{2})} \dfrac{1}{\|\vy - \vx\|^{d+\alpha}} \dif \vy.
        \label{eq:alpha_stable}
    \end{equation}
    In this setting, the integral term $\gI[\cdot]$ arising in the backward drift admits a concise expression in the Fourier domain \citep[Theorem~B.3]{yoon2023score}:
    \begin{equation*}
        \gI[f](\vx) = \dfrac{1}{f(\vx)}\gF^{-1}\left\{-i\alpha \vxi\|\vxi\|^{\alpha-2} \gF\{f\}\right\}(\vx),
    \end{equation*}
    where $\gF$ and $\gF^{-1}$ denote the Fourier transform and its inverse, respectively. 
    
    While the L\'evy measures associated with the true and estimated backward processes differ from the forward measure, a practical simplification used by~\citet{yoon2023score} is to approximate the ratio terms $p_t(\vy)/p_t(\vx)$ or $\varphi_t(\vy)/\varphi_t(\vx)$ by $1$. This approximation is justified by noting that the L\'evy measure defined in~\cref{eq:alpha_stable} concentrates strongly near the origin, thus allowing the standard isotropic $\alpha$-stable process to be used directly in the backward simulation.
    
    Typically, to obtain the estimated drift term $\cev\vb_t^\varphi(\vx)$, one would first estimate the jump-related score function $\hat{s}_t^\mathrm{jump}(\vx,\vy)$ by optimizing the general Lévy-type score-matching loss~\eqref{eq:general_levy_score_matching}, and then plug this into the integral term $\gI[\varphi_t](\vx)$. However, an alternative approach proposed by~\citet{yoon2023score} is to directly learn the integral term $\gI[\varphi_t](\vx)$ by minimizing the following simpler loss:
    \begin{equation*}
        \mathfrak{L}_\mathrm{LIM}[\gI[\varphi_t]] = \E_{\vx_0\sim p_0}\left[\int_0^T \E_{\vx_t\sim p_{t|0}(\cdot|\vx_0)}\left[\left\|\gI[\varphi_t](\vx) - \gI[p_{t|0}(\cdot|\vx_0)](\vx)\right\|^2 \right] \dif t\right].
    \end{equation*} 
    This loss no longer directly corresponds to minimizing an upper bound on $\KL(p_0\|q_T)$ as in~\cref{cor:error_bound}, but it is practically appealing due to its explicit form. When the forward drift is chosen as $\vb_t(\vx)=-\vx/\alpha$, this integral term simplifies to a closed-form expression:
    \begin{equation*}
        \gI[p_{t|0}(\vx_t|\vx_0)](\vx) = -\dfrac{\vx_t - e^{-\frac{t}{\alpha}}\vx_0}{\alpha(1-e^{-t})},
    \end{equation*}
    thus enabling efficient training and inference.
\end{example}

In the following theorem we justify the necessity of allowing general L\'evy-type forward processes in the denoising Markov model framework, including instances such as the L\'evy-It\^o model~\citep{yoon2023score}. We show that, under the standard diffusion-only setting with affine drift and additive Gaussian noise, heavy-tailed data inevitably induce an \emph{infinite} terminal mismatch against a Gaussian reference. Consequently, the KL-based error bound degenerates, which motivates enlarging the design space beyond diffusions.

\begin{theorem}
Consider a denoising Markov model whose forward process is the diffusion in \cref{eq:forward_diffusion},
$$
\dif \vx_t = \vb_t(\vx_t)\dif t + \mSigma_t\dif \vw_t,
$$
with affine drift $\vb_t(\vx)=\mA_t \vx + \vu_t$, where $\mA_t\in\R^{d\times d}$ and $\vu_t\in\R^d$ are measurable and bounded on $[0,T]$, and $\mSigma_t$ is measurable and bounded on $[0,T]$. Assume the data distribution $p_0$ is absolutely continuous, spherically symmetric, and heavy-tailed with $\E[\|\vx_0\|^2]=\infty$, and let the reference $q_0$ be Gaussian. Then, for every time horizon $T>0$,
$$
\KL\bigl(p_T\|q_0\bigr)=\infty,
$$
where $p_T$ is the $T$-marginal of the forward process.
\label{thm:heavy_tail}
\end{theorem}

The proof is provided in \cref{app:general_levy_type_process}. By \cref{cor:error_bound}, the sample error bound degenerates for diffusion-only forward processes with affine drift and additive Gaussian noise when $p_0$ is heavy-tailed:
\[
\KL(p_0\|q_T)\le\KL(p_T\|q_0)+\mathfrak{L}[\eta_t]=\infty,
\]
because the terminal marginal $p_T$ retains heavy tails, making the terminal mismatch against a Gaussian reference $q_0$ infinite and the KL-based guarantee vacuous. 

\begin{remark}
    The failure mode of \cref{thm:heavy_tail} is intrinsic to any finite physical time Gaussian noising path and therefore cannot be eliminated by a mere time reparameterization. 
    Consider compressing $t\in[0,\infty)$ into $\tau\in[0,1)$ via an increasing $C^1$ map $g:[0,1)\to[0,\infty)$ with $g(\tau)\to\infty$ as $\tau\to 1$, and the reparametrized process $\vy_\tau:=\vx_{g(\tau)}$. The time-changed process satisfies
    \[
    \dif \vy_\tau = g'(\tau) \vb_{g(\tau)}(\vy_\tau) \dif \tau + \sqrt{g'(\tau)}\mSigma_{g(\tau)} \dif \vw_\tau,
    \]
    and since $g(\tau)=\int_0^\tau g'(s) \dif s\to\infty$ as $\tau\to 1$, the factor $g'$ cannot remain bounded and the
    dynamics become stiff near $\tau\to 1$. As a result, any implementable numerical scheme must effectively
    truncate at $\tau=1-\varepsilon$ for some $\varepsilon>0$ (or equivalently, at finite physical time $t=g(1-\varepsilon)$),
    where the KL divergence remains infinite.
\end{remark}

This exposes a concrete failure mode of diffusion-only denoising Markov models for heavy-tailed data and motivates enlarging the design space to L\'evy-type forward processes with heavy-tailed marginals, \emph{e.g.}, the Ornstein-Uhlenbeck process driven by an $\alpha$-stable L\'evy process, whose stationary law is $\alpha$-stable. In such cases one can simply choose this heavy-tailed stationary distribution as the reference $q_0$ to match the tails, restoring a finite and informative terminal bound and making \cref{cor:error_bound} meaningful. For empirical evidence on using $\alpha$-stable L\'evy forwards with real-world heavy-tailed data, we refer to the results in \citep{yoon2023score}, without repeating those experiments in this paper. This provides a principled rationale for adopting general L\'evy-type processes in our framework when heavy-tailed data may arise. This framework is practically attractive when $p_{t|0}$, or a surrogate such as $\gI[p_{t|0}]$, is available in closed form or can be approximated efficiently.

To end this section, we provide insights into the connections between denoising Markov models and generator matching~\citep{holderrieth2024generator} in the following remark.

\begin{remark}[Connections to Generator Matching~\citep{holderrieth2024generator}]
    The primary conceptual distinction between our proposed framework of denoising Markov models and the generator matching approach of~\citet{holderrieth2024generator} lies in the direction and strategy used to specify and estimate the underlying Markov dynamics.

    In denoising Markov models, we specify the forward generator $\gL_t$ explicitly and subsequently estimate the backward generator $\cev \gL_t$ through the parameterization of a function $\varphi_t$, guided by conditional densities $p_{t|0}$. In contrast, the generator matching method specifies the form of the backward generator $\cev \gL_t$ first and then directly learns its coefficients using a prescribed probability path $(p_t)_{t\in[0,T]}$ (or equivalently, $(\cev p_t)_{t\in[0,T]}$).
    
    Specifically, generator matching assumes that for each terminal condition $\cev \vx_T \sim \cev p_T$ (or equivalently, $\vx_0 \sim p_0$), the conditional probability path $(\cev p_{t|T}(\cdot|\cev \vx_T))_{t\in[0,T]}$ (or equivalently, $(p_{T-t|0}(\cdot|\vx_0))_{t\in[0,T]}$) is pre-determined and assumed to be governed by a general Lévy-type conditional backward generator of the form:
    \begin{equation}
        \begin{aligned}
            \cev \gL_{t|T}^{\cev \vx_T} =& \cev \vb_{t|T}(\vx|\cev \vx_T) \cdot \nabla f(\vx) + \dfrac{1}{2} \cev \mD_{t|T}(\vx|\cev \vx_T) : \nabla^2 f(\vx) \\
            &+ \int_{\R^d \backslash \{\vx\}} \left(f(\vy) - f(\vx) - (\vy - \vx) \cdot \nabla f(\vx) \vone_{B(\vx, 1)}(\vy - \vx)\right) \cev \lambda_{t|T}(\dif \vy, \vx|\cev \vx_T),
        \end{aligned}
        \label{eq:conditional_generator_matching}
    \end{equation}
    where the triplet $(\cev \vb_{t|T}(\vx|\cev \vx_T), \cev \mD_{t|T}(\vx|\cev \vx_T), \cev \lambda_{t|T}(\dif \vy, \vx|\cev \vx_T))$ is assumed to be available in closed form formulas, derived directly from Kolmogorov's forward equation (\cref{thm:kolmogorov_forward}):
    $$
        \partial_t \cev p_{t|T}(\cdot|\cev \vx_T) = \left(\cev \gL_{t|T}^{\cev \vx_T}\right)^* \cev p_{t|T}(\cdot|\cev \vx_T).
    $$

    Then, leveraging the linearity of the generator, the unconditional backward generator $\cev \gL_t$ can be expressed as the posterior expectation of this conditional generator (\emph{cf.} \citep[Proposition~1]{holderrieth2024generator}):
    \begin{equation}
        \begin{aligned}
            &\cev \gL_t f(\vx) = \E_{\cev \vx_T \sim \cev p_{T|t}(\cdot|\vx)}\left[\cev \gL_{t|T}^{\cev \vx_T} f(\vx)\right]\\
            =& \E_{\cev \vx_T \sim \cev p_{T|t}(\cdot|\vx)}\left[\cev \vb_{t|T}(\vx|\cev \vx_T)\right] \cdot \nabla f(\vx) + \dfrac{1}{2} \E_{\cev \vx_T \sim \cev p_{T|t}(\cdot|\vx)}\left[\cev \mD_{t|T}(\vx|\cev \vx_T)\right] : \nabla^2 f(\vx) \\
            +& \int_{\R^d \backslash \{\vx\}} \left(f(\vy) - f(\vx) - (\vy - \vx) \cdot \nabla f(\vx) \vone_{B(\vx, 1)}(\vy - \vx)\right) \E_{\cev \vx_T \sim \cev p_{T|t}(\cdot|\vx)}\left[\cev \lambda_{t|T}(\dif \vy, \vx|\cev \vx_T)\right],
        \end{aligned}
        \label{eq:generator_matching_backward}
    \end{equation}
    \emph{i.e.}, the L\'evy triple $(\cev \vb_t, \cev \mD_t, \cev \lambda_t)$ of the backward process is given by the posterior expectations of the L\'evy triple $(\cev \vb_{t|T}(\cdot|\cev \vx_T), \cev \mD_{t|T}(\cdot|\cev \vx_T), \cev \lambda_{t|T}(\cdot|\cev \vx_T))$ of the conditional generator $\cev \gL_{t|T}^{\cev \vx_T}$ given $\cev \vx_T$, as
    \begin{gather*}
        \cev \vb_t(\vx) = \E_{\cev \vx_T \sim \cev p_{T|t}(\cdot|\vx)}\left[\cev \vb_{t|T}(\vx|\cev \vx_T)\right], \quad \cev \mD_t(\vx) = \E_{\cev \vx_T \sim \cev p_{T|t}(\cdot|\vx)}\left[\cev \mD_{t|T}(\vx|\cev \vx_T)\right],\\
        \text{and}\quad \cev \lambda_t(\dif \vy, \vx) = \E_{\cev \vx_T \sim \cev p_{T|t}(\cdot|\vx)}\left[\cev \lambda_{t|T}(\dif \vy, \vx|\cev \vx_T)\right].
    \end{gather*}

    The generator matching training strategy thus involves learning these posterior expectations~\eqref{eq:generator_matching_backward} by formulating suitable regression problems. For instance, the drift term $\cev \vb_t(\vx)$ can be learned via the minimization of the following regression-based loss:
    \begin{equation*}
        \begin{aligned}
            \mathfrak{L}_\mathrm{GM}^{\vb}[\cev \vb^\theta_t] =& \int_0^T \E_{\cev \vx_t \sim \cev p_t}\left[\E_{\cev \vx_T \sim \cev p_{T|t}(\cdot|\cev \vx_t)}\left[\left\|\cev \vb^\theta_t(\cev \vx_t) - \cev \vb_{t|T}(\cev \vx_t|\cev \vx_T)\right\|^2 \right]\right] \dif t\\
            =& \int_0^T \E_{\vx_0 \sim p_0}\left[\E_{\vx_{T-t} \sim p_{T-t|0}(\cdot|\vx_0)}\left[\left\|\cev \vb^\theta_t(\vx_{T-t}) - \cev \vb_{t|T}(\vx_{T-t}|\vx_0)\right\|^2 \right]\right] \dif t,
        \end{aligned}
    \end{equation*} 
    with samples $\cev \vx_t\sim\cev p_{t|T}(\cdot|\cev \vx_T)$ obtained by simulating the conditional generator~\eqref{eq:conditional_generator_matching}. After obtaining the estimated Lévy triple $(\cev\vb_t^\theta,\cev\mD_t^\theta,\cev\lambda_t^\theta)$, inference proceeds via simulation from the resulting estimated backward generator.
\end{remark}

\section{Experiments}
\label{sec:experiments}

In this section, we present several experiments validating our theoretical results. Given the extensive empirical literature on continuous and discrete diffusion models as special cases of our framework (\emph{cf.}~\cref{ex:continuous_diffusion,ex:discrete_diffusion}), we focus on novel examples of Markov processes that have not yet been extensively explored in the context of denoising Markov models. The primary purpose of these experiments is to demonstrate the generality and flexibility of our theoretical approach. Given the established success of diffusion-based models, our objective is not to outperform existing methods; instead, we aim to highlight the viability of alternative formulations by applying them to conceptually simple tasks.

We will consider two examples of denoising Markov models: (1) a diffusion model with geometric Brownian motion as the forward process, and (2) a ``jump model'' with a jump process in $\R^d$ as the forward process. In both cases, we will demonstrate that the learned backward processes can effectively approximate the target distribution.

\subsection{Diffusion Model with Geometric Brownian Motion}
\label{sec:geometrical}

In this illustrative example, we consider geometric Brownian motion (GBM) as our forward process. The forward dynamics follow:
\begin{equation*}
    \dif \vx_t = \vx_t \odot \mSigma \dif \vw_t,
\end{equation*}
where $\odot$ denotes the element-wise product.
Through a change of variables, this equation admits a closed-form solution:
\begin{equation}
    \vx_t = \vx_0 \odot \exp\left(\mSigma \vw_t - \dfrac{1}{2} \diag(\mSigma \mSigma^\top) t\right),
    \label{eq:geometric_solution}
\end{equation}
where $\diag(\mSigma \mSigma^\top)$ denotes the vector with the diagonal elements of the matrix $\mSigma \mSigma^\top$. 
The diffusion matrix in this case is given explicitly by $\mD(\vx_t) = \diag(\vx_t)\mSigma\mSigma^\top\diag(\vx_t)$. Thus, this choice clearly fits within our denoising Markov model framework, with a diffusion process as the forward dynamics. Unlike the standard diffusion models (\cref{ex:continuous_diffusion}), the diffusion matrix here is neither constant nor spatially homogeneous, depending explicitly on the current state $\vx_t$.

Following the framework in~\cref{sec:diffusion}, we parameterize the estimated score function $\hat \vs_t(\vx)$ by a neural network with parameters $\theta$ and optimize the following score-matching loss~\eqref{eq:diffusion_loss}:
\begin{equation*}
    \begin{aligned}
        \mathfrak{L}_\mathrm{SM}[\hat \vs_t] &= \E_{\vx_0 \sim p_0}\bigg[\int_0^T \E_{\vx_t \sim p_{t|0}(\cdot|\vx_0)}\bigg[\\
        &\dfrac{1}{2} \mD(\vx_t) : \left(\hat \vs_t^\theta(\vx_t) - \nabla \log p_{t|0}(\vx_t|\vx_0)\right) \left(\hat \vs_t^\theta(\vx_t) - \nabla \log p_{t|0}(\vx_t|\vx_0)\right)^\top\bigg] \dif t\bigg].
    \end{aligned}
\end{equation*}
Using the explicit solution~\eqref{eq:geometric_solution}, the gradient of the conditional log-density is given by:
\begin{equation}
    \nabla \log p_{t|0}(\vx_t|\vx_0) = - \vx_t^{-1} \odot \left(1 + \dfrac{1}{t} (\mSigma \mSigma^\top)^{-1} \left(\log \vx_t - \log \vx_0 + \dfrac{t}{2} \diag\left(\mSigma \mSigma^\top\right) \right)\right),
    \label{eq:geometric_score}
\end{equation}
with the vector operations taken element-wise. Note that, when $\mD(\vx_t)$ is not diagonal, this loss no longer simplifies to a standard mean squared error, highlighting the generality of our method.

Due to the positivity-preserving nature of GBM, our experiments focus on the positive orthant $\R^d_+$. We choose the target distribution as the absolute value of mixtures of Gaussian distributions. We select a sufficiently large terminal time $T$ to ensure that the forward dynamics~\eqref{eq:geometric_solution} approach the origin closely with high probability. Correspondingly, the backward dynamics, derived from~\cref{eq:backward_diffusion}, read:
\begin{equation*}
    \dif \vy_t = \big(\mD(\vy_t) \hat \vs_{T-t}(\vy_t) + \nabla \cdot \mD(\vy_t) \big) \dif t + \vy_t \odot \mSigma \dif \vw_t.
\end{equation*}

\begin{figure}[!t]
    \centering
    \begin{subfigure}{0.45\textwidth}
        \centering
        \includegraphics[width=\textwidth]{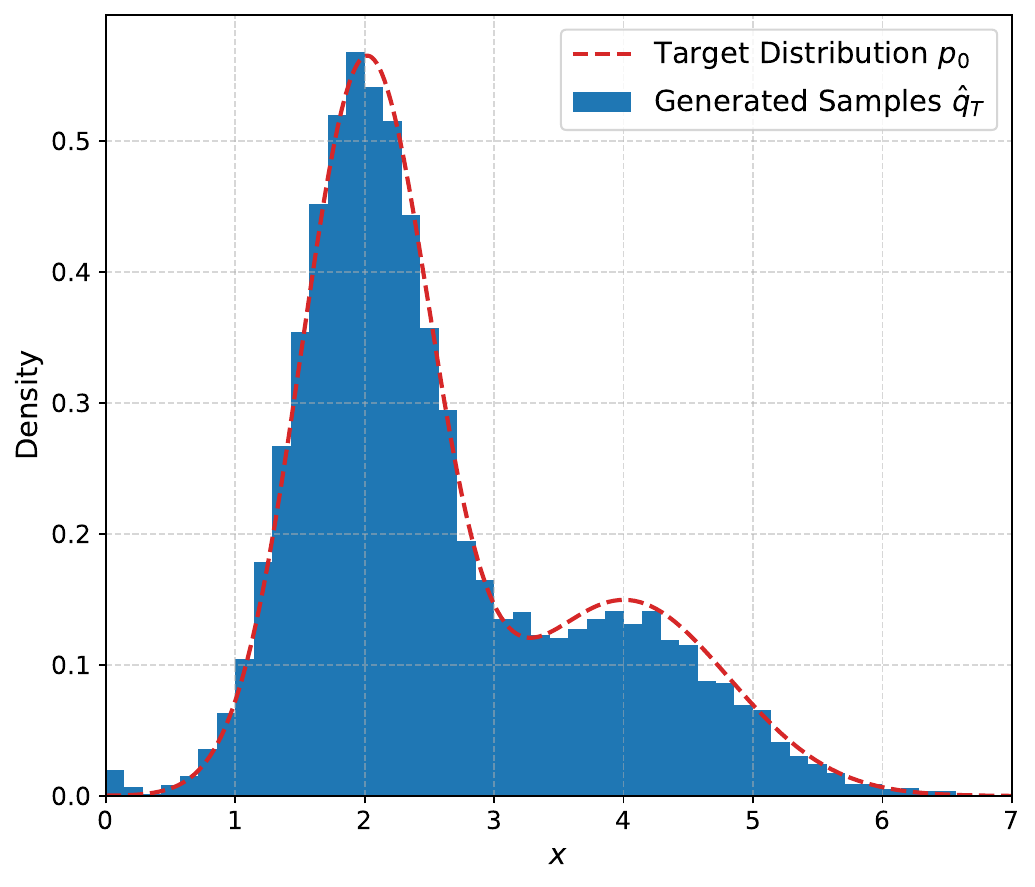}
        \caption{1-D Case.}
    \end{subfigure}
    \hspace{0.02\textwidth}
    \begin{subfigure}{0.45\textwidth}
        \centering
        \includegraphics[width=\textwidth]{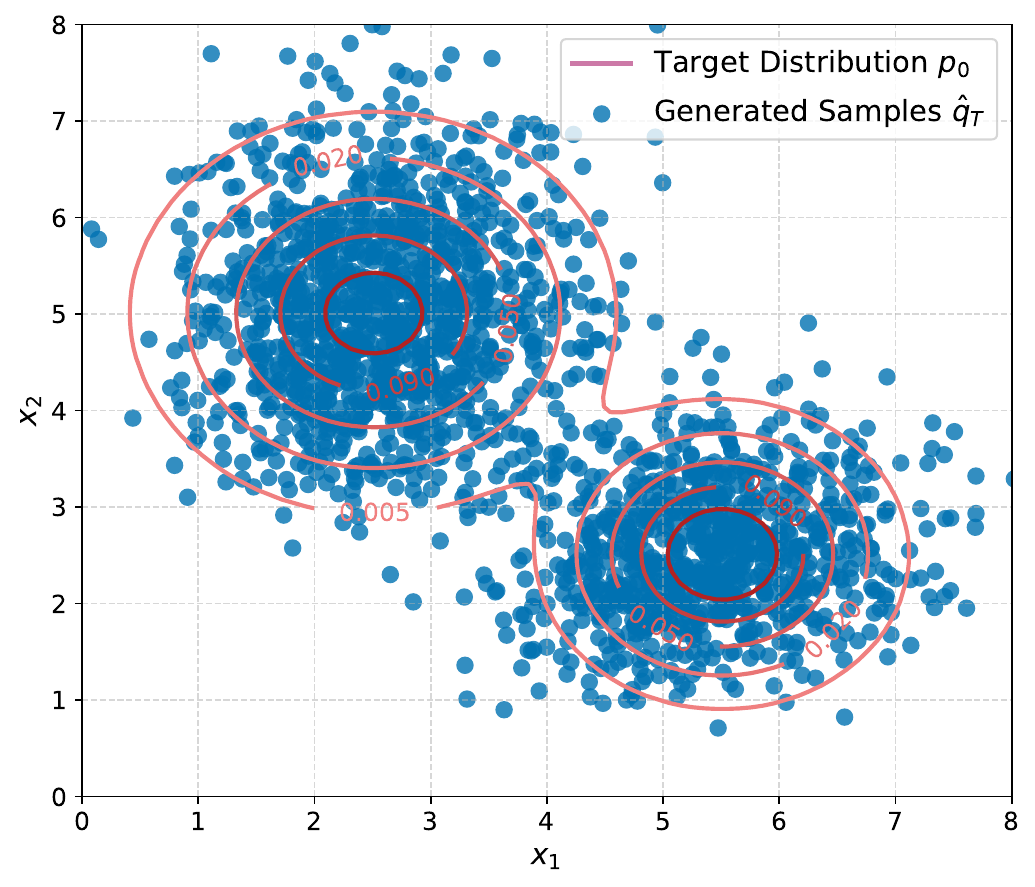}
        \caption{2-D Case.}
    \end{subfigure} 
    \caption{Empirical results of denoising Markov models using geometric Brownian motion as the forward process. The generated samples from $\hat q_T$ (in {\color{seabornblue} blue}) closely match the target distribution $p_0$ (dashed line in 1-D and contour lines in 2-D in {\color{seabornred} red}), highlighting the effectiveness of the learned backward process.}
    \label{fig:geometric}
\end{figure}

\begin{figure}[!t]
    \centering
    \begin{subfigure}{0.9\textwidth}
        \centering
        \includegraphics[width=\textwidth]{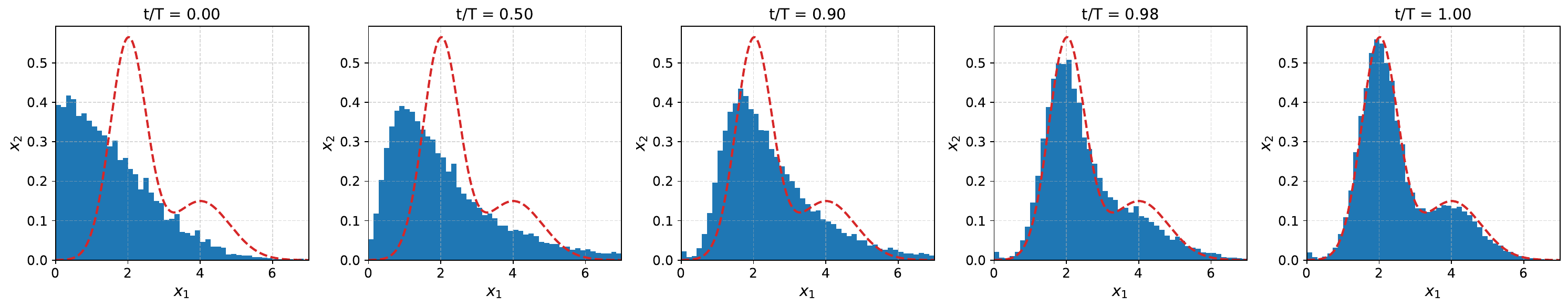}
        \caption{1-D Case.}
    \end{subfigure}
    
    \begin{subfigure}{0.9\textwidth}
        \centering
        \includegraphics[width=\textwidth]{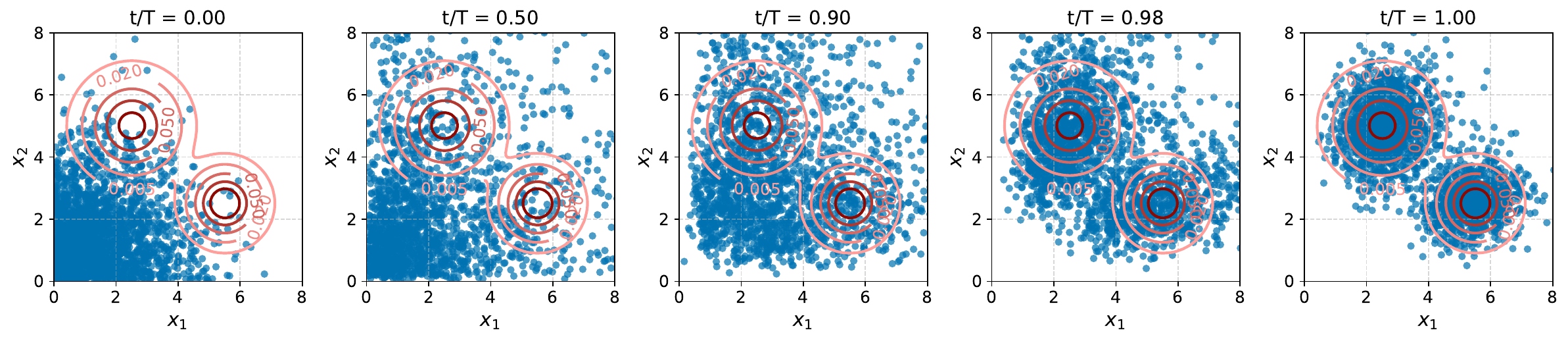}
        \caption{2-D Case.}
    \end{subfigure} 
    \caption{Visualization of generated backward trajectories using geometric Brownian motion as the forward process.  Samples generated via the learned backward dynamics (in {\color{seabornblue} blue}) evolve from the initial distribution $q_0$ with the estimated score function $\hat \vs_t(\vx)$ into the target distribution $p_0$ (in {\color{seabornred} red}), further validating the theoretical robustness and practical effectiveness of our denoising Markov models.}
    \label{fig:geometric_trajectories}
\end{figure}

To demonstrate the effectiveness of our method, we conduct experiments in both one and two-dimensional cases (\cref{fig:geometric}). For both cases, we use an MLP with 5 layers and 128 hidden units for the score function $\hat \vs_t(\vx)$ with the input simply being the concatenation of the state $\vx_t$ and the time $t$.
We optimize the score-matching loss~\eqref{eq:diffusion_loss} using the Adam optimizer with a learning rate of $10^{-3}$.
In the 1-D scenario, we set $\Sigma = 1$ and use the following target distribution:
\begin{equation*}
    p_0(y) = 0.7 \gN(y; 2, 0.25) + 0.3 \gN(y; 4, 0.64),\ \forall \quad y \in \R_+.
\end{equation*}
In the 2D scenario, we use:
\begin{equation*}
    \mSigma = \begin{bmatrix} 1.0 & 0.4 \\ 0.2 & 1.0 \end{bmatrix}, \ \text{and} \ p_0(\vy) = 0.6 \gN\left(\vy; \begin{bmatrix} 2.5 \\ 5 \end{bmatrix}, \mI\right) + 0.4 \gN\left(\vy; \begin{bmatrix} 5.5 \\ 2.5 \end{bmatrix}, \mI\right), \ \forall \vy \in \R_+^2, 
\end{equation*}
where $\gN(\cdot; \vmu, \mSigma)$ denotes the Gaussian density with mean $\vmu$ and covariance $\mSigma$. We approximate the initial distribution $q_0$ by sampling from $\gN(\mathbf{0},2\mI)$ and then applying the absolute-value transformation.

As evidenced in~\cref{fig:geometric}, our model achieves satisfactory performance: the generated samples closely approximate the target distributions in both dimensions. Despite the potential singularities in the score function~\eqref{eq:geometric_score} as $t \to 0$ or as $\|\vx_t\|\to 0$, training is notably stable, requiring only moderate neural-network sizes. Furthermore, \cref{fig:geometric_trajectories} illustrates the evolution of the generated trajectories during inference, demonstrating how samples move smoothly from the initial noise toward the complex target distribution along a different path from the pure Brownian motion (Variance-Exploding SDE) or OU process (Variance-Preserving SDE).
This provides empirical support for the soundness of the backward process learned via score-matching and emphasizes our framework's flexibility in accommodating more general Markov dynamics beyond standard diffusion processes.

\subsection{Jump Model}
\label{sec:jump_experiment}

In this example, we propose a pure jump process-based model as a new type of denoising Markov model. Specifically, we consider a forward process consisting exclusively of jumps, \emph{i.e.}, we set the drift $\vb_t(\vx) \equiv \vzero$ and the diffusion coefficient $\mD_t(\vx) \equiv \vzero$.
For simplicity, we assume that the L\'evy measure $\lambda_t(\dif \vy, \vx)$ admits a density $\lambda_t(\vy, \vx)$ with respect to the Lebesgue measure and satisfies the following integrability condition: 
\begin{equation*}
    \int_{B(\vx, 1)} \|\vy-\vx\| \lambda_t(\dif \vy, \vx) < +\infty.
\end{equation*}
This ensures that small jumps in~\cref{eq:generator_levy_type} are integrable and thus can be naturally included with the drift term, which vanishes in this case.

Under these assumptions, the forward generator simplifies to:
\begin{equation*}
    \gL_t f(\vx) = \int_{\R^d \backslash \{\vx\}} \left(f(\vy) - f(\vx)\right) \lambda_t(\vy, \vx) \dif \vy.
\end{equation*}
This corresponds to a pure-jump L\'evy-type process represented by the stochastic integral:
\begin{equation}
    \vx_t = \vx_0 + \int_0^t \int_{\R^d} (\vy - \vx_{s^-}) N[\lambda](\dif s, \dif \vy),
    \label{eq:jump_forward}
\end{equation}
where $N[\lambda](\dif s, \dif \vy)$ denotes the Poisson random measure with compensator $\lambda_s(\vy, \vx_{s^-}) \dif \vy \dif s$. We refer readers to~\cref{sec:jump} for more details on this formulation.

Correspondingly, \cref{ass:K} yields a simplified estimated backward generator:
\begin{equation*}
    \gK_t f(\vx) = \int_{\R^d \backslash \{\vx\}} \left(f(\vy) - f(\vx)\right) \cev \lambda^\varphi_t(\vy, \vx) \dif \vy,
\end{equation*}
where the estimated L\'evy measure $\cev \lambda^\varphi_t(\vy, \vx) \dif \vy$ is defined explicitly via the score function for the jump part:
\begin{equation*}
    \cev \lambda^\varphi_t(\vy, \vx) =  \hat s_{T-t}^{\mathrm{jump}}(\vx, \vy) \lambda_{T-t}(\vx, \vy) = \dfrac{\varphi_{T-t}(\vy)}{\varphi_{T-t}(\vx)} \lambda_{T-t}(\vx, \vy).
\end{equation*}
This estimated L\'evy measure is both temporally and spatially inhomogeneous.
In contrast to general Lévy-type processes, this jump-only model eliminates the need to estimate any diffusion-based score function. The score-matching loss from~\eqref{eq:general_levy_score_matching} then reduces to:
\begin{equation}
    \begin{aligned}
        \mathfrak{L}_\mathrm{SM}[\hat s_t^{\mathrm{jump}}] = &\E_{\vx_0 \sim p_0}\Bigg[\int_0^T \E_{\vx_t \sim p_{t|0}(\cdot| \vx_0)} \Bigg[\\
        &\int_{\R^d} \left(\hat s_t^{\mathrm{jump}}(\vx_t, \vy) - \dfrac{p_{t|0}(\vy|\vx_0)}{p_{t|0}(\vx_t|\vx_0)} \log \hat s_t^{\mathrm{jump}}(\vx_t, \vy)\right) \lambda_t(\vx_t, \vy) \dif \vy \Bigg] \dif t\Bigg].
    \end{aligned}
    \label{eq:jump_score_matching}
\end{equation}
Once the jump-based score function $\hat s_t^{\mathrm{jump}}$ is learned, we simulate the backward process by:
\begin{equation}
    \vy_t = \vy_0 + \int_0^t \int_{\R^d} (\vy - \vy_{s^-}) N[\cev \lambda^\varphi](\dif s, \dif \vy),
    \label{eq:jump_backward}
\end{equation}
where $N[\cev \lambda^\varphi](\dif s, \dif \vy)$ denotes a Poisson random measure with compensator $\cev \lambda^\varphi_t(\vy,\vy_{s^-}) \dif \vy \dif s$.

We highlight that this jump model generalizes discrete diffusion models discussed in~\cref{sec:jump}, extending them from finite state spaces to continuous domains.  Instead of continuous-time Markov chains in the finite-state setting, the forward and backward processes are now stochastic integrals with respect to Poisson random measures, and the score-matching loss is an integral instead of a sum. These differences pose computational challenges to the implementation of the model:
\begin{itemize}
    \item {\bfseries Forward Simulation:} By choosing a space and time-homogeneous Lévy measure $\lambda(\vy - \vx)$, the forward process~\eqref{eq:jump_forward} reduces to a compound Poisson process. Specifically, in order to sample $\vx_t$ given $\vx_0$, one first samples the jump count from a Poisson distribution with mean $t \int_{\R^d}\lambda(\vy-\vx)\dif\vy$, then samples corresponding jump locations from $\lambda(\cdot)$.
    \item {\bfseries Numerical Integration:} If the L\'evy measure $\lambda(\vy - \vx) \dif \vy$ is further isotropic as in~\cref{ex:alpha_stable}, the integral over $\lambda(\|\vy - \vx\|) \dif \vy$ in the score-matching loss~\eqref{eq:jump_score_matching} is efficiently approximated by Monte Carlo integration using samples from the isotropic kernel
    $\lambda(\|\cdot\|)$.
    \item {\bfseries Backward Simulation:} 
    The backward simulation~\eqref{eq:jump_backward} is nontrivial due to the space-time inhomogeneity of the estimated Lévy measure $\cev\lambda^\varphi_t(\vy,\vx)$. Specifically, computing the integral:
    \begin{equation}
        \gJ[\cev \lambda^\varphi_t(\cdot, \vy_t^{(k)})] = \int_{\R^d} \cev \lambda^\varphi_t(\vy, \vy_t^{(k)}) \dif \vy = \int_{\R^d} \hat s_{T-t}^{\mathrm{jump}}(\vy_t^{(k)}, \vy) \lambda(\|\vy - \vy_t^{(k)}\|) \dif \vy
        \label{eq:jump_integral}
    \end{equation}
    for each sample $\vy_t^{(k)}$ at time $t$ is computationally expensive. 
    However, noting that the score function satisfies:
    \begin{equation*}
        \hat s_t^{\mathrm{jump}}(\vy_t^{(k)}, \vy) = \dfrac{\hat s_t^{\mathrm{jump}}(\vy^{\mathrm{ref}}, \vy)}{\hat s_t^{\mathrm{jump}}(\vy^{\mathrm{ref}}, \vy_t^{(k)})},
    \end{equation*}
    with a fixed reference point $\vy^{\mathrm{ref}}$, one can reuse numerical integrations as
    \begin{equation*}
        \gJ[\cev \lambda^\varphi_t(\cdot, \vy_t^{(k)})] = \dfrac{\gJ[\cev \lambda^\varphi_t(\cdot, \vy^{\mathrm{ref}})]}{\hat s_{T-t}^{\mathrm{jump}}(\vy^{\mathrm{ref}}, \vy_t^{(k)})},
    \end{equation*}
    significantly reducing computational costs across samples.
    Subsequently, given the integral~\eqref{eq:jump_integral}, one can sample the jump count from a Poisson distribution with mean $\kappa\gJ[\cev \lambda^\varphi_t(\cdot, \vy_t^{(k)})]$ and use rejection sampling to draw jump locations according to the estimated backward Lévy measure $\cev\lambda_t^\varphi(\vy,\vy_t^{(k)})$, by first estimating $\sup_{\vy \in \R^d} \hat s_t^{\mathrm{jump}}(\vy_t^{(k)}, \cdot)$ and then drawing samples from the kernel $\lambda(\|\vy - \vy_t^{(k)}\|)$.
\end{itemize}

\begin{figure}[!t]
    \centering
    \begin{subfigure}{0.98\textwidth}
        \centering
        \includegraphics[width=\textwidth]{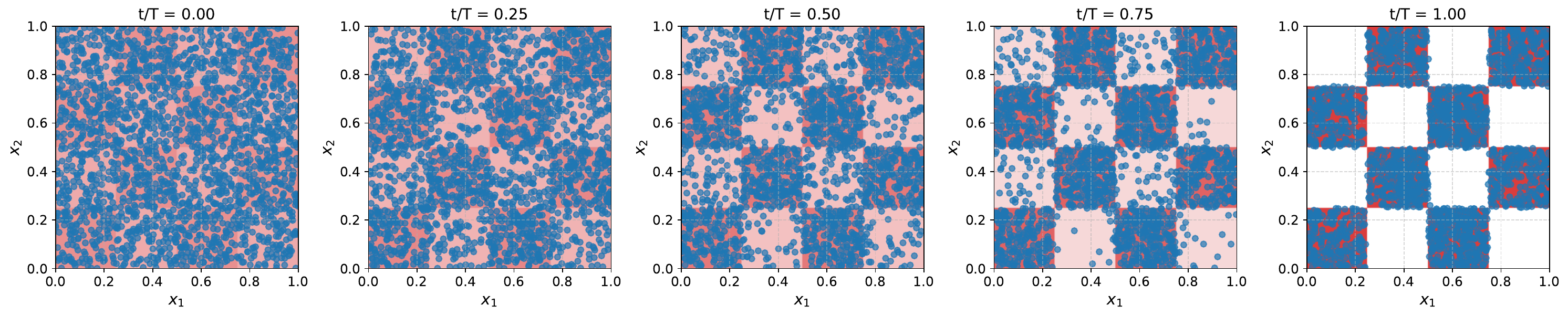}
        \caption{Chessboard.}
    \end{subfigure}
    
    \begin{subfigure}{0.98\textwidth}
        \centering
        \includegraphics[width=\textwidth]{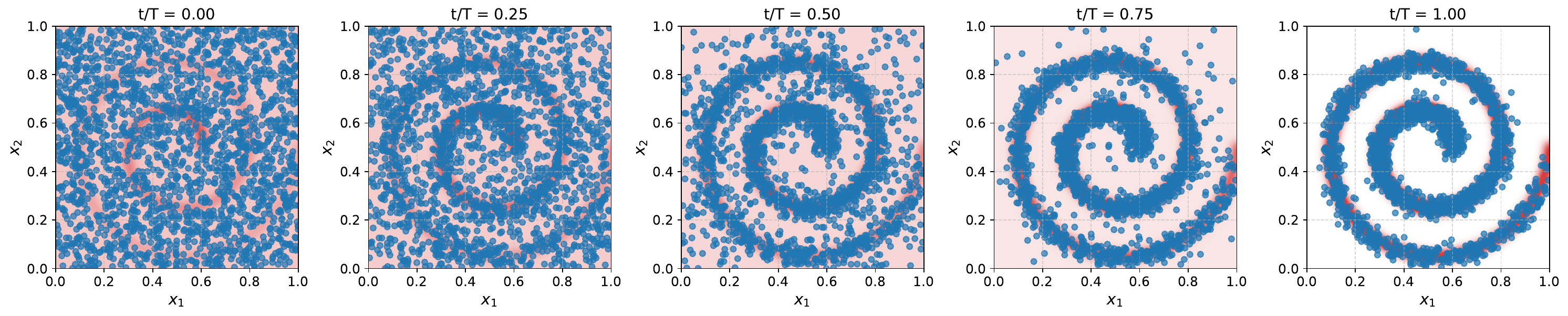}
        \caption{Swiss Roll.}
    \end{subfigure} 

    \begin{subfigure}{0.98\textwidth}
        \centering
        \includegraphics[width=\textwidth]{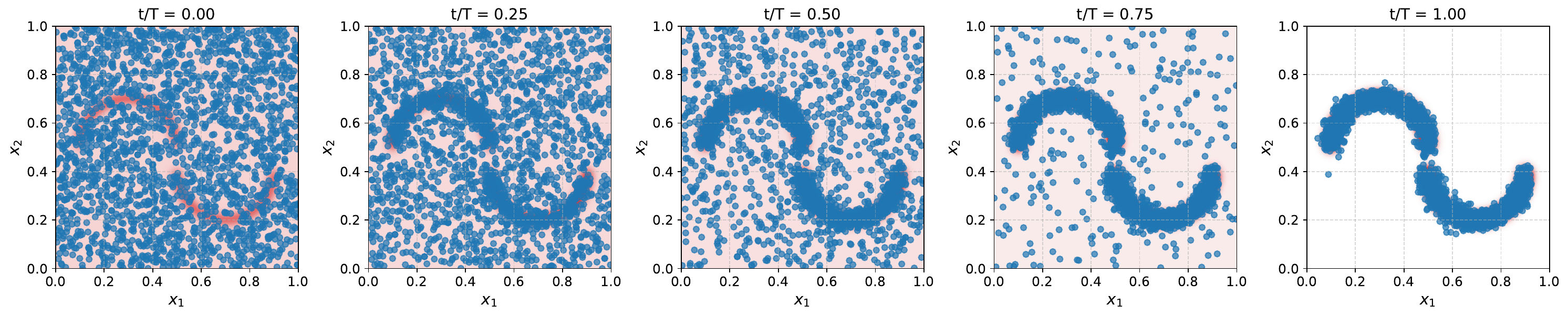}
        \caption{Moons.}
    \end{subfigure} 

    \caption{Visualization of generated backward trajectories using the jump process as the forward process.
    Samples from $q_t$ generated via the learned backward dynamics (in {\color{seabornblue} blue}) are compared with the corresponding distribution $p_{T-t}$ from the forward process (in {\color{seabornred} red}). Results for three distributions (Chessboard, Swiss roll, Moons) demonstrate close agreement between the generated and target distributions.}
    \label{fig:jump}
\end{figure}

In our empirical experiments, we consider a 2-D spatial domain defined on the torus $\mathbb{T}^2 = \R^2 / \sZ^2$. The L\'evy measure $\lambda_t(\vy, \vx) \dif \vy$ is chosen as a time and space-homogeneous, isotropic Gaussian kernel:
\begin{equation*}
    \lambda_t(\vy, \vx) = \gN(\vy - \vx; \vzero, \mI) = \dfrac{1}{2\pi} \exp\left(-\dfrac{\|\vy - \vx\|^2}{2}\right).
\end{equation*}
We select a sufficiently large time horizon $T$ such that the forward jump process converges towards the uniform distribution on $\mathbb{T}^2$.
To efficiently compute the integral term~\eqref{eq:jump_integral}, we leverage the periodic structure of the domain and the property of the Gaussian kernel by employing numerical quadrature based on the fast Fourier transform. For all cases, we use an MLP with 5 layers and 128 hidden units to approximate the score function $\hat s_t^{\mathrm{jump}}(\vx, \vy)$. We optimize the score-matching loss~\eqref{eq:jump_score_matching} using the Adam optimizer with a learning rate of $10^{-3}$.
Additionally, we use Fourier features for the input $\vx$ and $\vy$ to account for the periodic boundary conditions of the torus in the neural network architectures used to approximate the score functions. 

The empirical results for three distinct 2-D target distributions (Chessboard, Swiss Roll, and Moons) are presented in Figure 4, illustrating the evolution of samples from the estimated backward jump dynamics. Specifically, we visualize 2048 backward-generated samples at different intermediate time steps and compare them with the corresponding forward distributions $p_{T-t}$. As demonstrated in Figure 4, the backward jump trajectories initially start from a uniform distribution at $t=0$, gradually capturing intricate features of the target distributions as time progresses. By the end of the backward simulation ($t=T$), the generated samples accurately reconstruct the complex structures of each target distribution, verifying the effectiveness of our jump-based denoising Markov model.

Unlike diffusion models, which typically begin their backward evolution from Gaussian distributions, our jump-based model starts from the uniform distribution, highlighting a distinct and versatile approach within generative modeling frameworks. We also refer readers to~\citep[Figure~2 and Appendix~E]{holderrieth2024generator} for a comparison of our denoising approach with the generator matching framework in the context of pure jump processes. Our denoising Markov models do not directly work with the Kolmogorov's forward equation to obtain the conditional generator $\cev \gL_{t|T}^{\cev \vx_T}$ through its coefficients $\cev \lambda_{t|T}$, but instead leverage the score-matching technique to learn the backward process through the score function $\hat s_t^\text{jump}$.

\section{Discussion}
\label{sec:conclusion}

In this paper, we introduced a novel framework for denoising Markov models, significantly extending classical diffusion models to a broader class of Markovian dynamics. By establishing connections between the backward process and the generalized Doob's $h$-transform, we provided a rigorous theoretical foundation and a versatile methodology for constructing denoising Markov models from general Lévy-type forward processes. Under suitable regularity conditions, we justified the use of standard score-matching techniques for efficiently training these models. Additionally, we discussed practical implementation strategies and elaborated on connections with the generator matching approach recently proposed by~\citet{holderrieth2024generator}.
Through empirical experiments, we explored several novel instances of our general framework, including geometric Brownian motions and jump processes, highlighting their effectiveness in generative modeling scenarios that had not been previously studied. These experimental results demonstrated the flexibility and robustness of our approach, validating our theoretical insights.

For future work, several promising directions remain open. A key area of interest is extending the jump model to high-dimensional problems by introducing sparsity structures into the Lévy measure \(\lambda_t(\vy, \vx)\), following recent advances by~\citet{lou2023discrete}. This would enable efficient large-scale sampling and mass transportation, potentially unlocking powerful new generative modeling capabilities. Furthermore, a deeper theoretical investigation into how the choice of forward process affects training efficiency, generalization performance, and model robustness remains an exciting avenue. Specifically, it would be beneficial to systematically explore the optimal selection of forward processes tailored to specific structures present in target distributions, such as multimodality or heavy tails.

\acks{The authors thank the editor and the anonymous reviewers for their valuable comments and suggestions. Yinuo Ren and Lexing Ying were supported by the National Science Foundation under Grant No. DMS-2208163. This material is based upon work supported by the National Science Foundation under Grant No. CHE-2441297.
}

\newpage
\vskip 0.2in
\bibliography{25-0693}

\newpage
\appendix
\section{Mathematical Approach to Denoising Markov Models}
\label{app:math}

In this section, we assume that $E$ is a locally compact, separable Hausdorff space, equipped with a base Radon measure $(E, \gB(E), \mu)$, where $\gB(E)$ denotes the Borel $\sigma$-algebra on $E$. 
Let $(x_t)_{t\in \sT}$ be a time-inhomogeneous Markov process on $E$, defined on the probability space $(\Omega, \gF, \P)$, where $\sT$ is the time index set, either $\R$ or $[0, +\infty)$. 
For notational simplicity, we will write $\P(\cdot)$ instead of $\P(\{\omega \in \Omega \mid \cdot\})$ when there is no risk of confusion.

\subsection{Feller Evolution Systems}

We begin by introducing the concept of time-evolution operators and evolution systems, which form a family of linear operators closely associated with Markov processes.

\begin{definition}[Time-Evolution Operators]
    A family of linear operators $(U_{t,s})_{s\leq t, s,t\in \sT}$ is called a \emph{time-evolution operator family}, or simply an \emph{evolution system}, if the following conditions hold:
    \begin{enumerate}[label=(\arabic*)]
        \item $U_{t,t} = \mathrm{id}$ for any $t \in \sT$;
        \item $U_{t,s} = U_{t,r} U_{r,s}$ for any $s \leq r \leq t$, with $s, r, t \in \sT$.
    \end{enumerate}
    \label{def:evolution_system}
\end{definition}

This definition is motivated by interpreting the operator $U_{t,s}$ as the conditional expectation operator associated with the Markov process $(x_t)_{t\in \sT}$. Specifically, for any $f \in B_b(E)$, the space of bounded Borel measurable functions on $E$, and any $s \leq t$, we define the operator $U_{t,s}$ by
\begin{equation}
    U_{t,s} f(x) = \E\left[f(x_t) \mid x_s = x\right].
    \label{eq:evolution_markov}
\end{equation}
It is straightforward to verify that $U_{t,s}$ satisfies both conditions in \cref{def:evolution_system}, and hence qualifies as an evolution system. Moreover, the operator family $(U_{t,s})$ defined in \eqref{eq:evolution_markov} possesses the following properties:
\begin{enumerate}[label=(\arabic*)]
    \item (Positivity-Preserving Property) $U_{t,s} f \geq 0$, if $f \geq 0$ and $f \in B_b(E)$, for any $s \leq t$, $s, t \in \sT$;
    \item (Contractivity) $\|U_{t,s} f\|_\infty \leq \|f\|_\infty$, for any $f \in B_b(E)$, $s \leq t$, $s, t \in \sT$;
    \item (Conservation) $U_{t,s} 1 = 1$, for any $f \in B_b(E)$, $s \leq t$, $s, t \in \sT$.
\end{enumerate}

We also define a special class of evolution systems, namely the \emph{Feller evolution system}~\citep{bottcher2014feller}, as follows:
\begin{definition}[Feller Evolution System]
    An evolution system $(U_{t,s})_{s \leq t, s,t\in \sT}$ is called a \emph{Feller evolution system} if, for any $f \in C_0(E)$, the space of continuous functions vanishing at infinity\footnote{A function $f$ is said to \emph{vanish at infinity} if for any $\epsilon > 0$, there exists a compact set $K \subset E$ such that $|f(x)| < \epsilon$ for all $x \in E \setminus K$.}, the following conditions are satisfied:
    \begin{enumerate}[label=(\arabic*)]
        \item (Feller Property) $U_{t,s} f \in C_0(E)$, for any $s \leq t$, $s, t \in \sT$;
        \item (Positivity-Preserving Property) $U_{t,s} f \geq 0$ if $f \geq 0$, for any $s \leq t$, $s, t \in \sT$;
        \item (Contractivity) $\|U_{t,s} f\|_\infty \leq \|f\|_\infty$, for any $s \leq t$, $s, t \in \sT$;
        \item (Strong Continuity) The following limit holds:
        \begin{equation*}
            \lim_{(\sigma, \tau) \to (s, t)} \|U_{\sigma, \tau} f - U_{t,s} f\|_\infty = 0, \text{ for any } s \leq t, s, t \in \sT.
        \end{equation*}
    \end{enumerate}
    \label{def:feller_evolution_system}
\end{definition}

For any evolution system, we can define its right and left generators as follows. 
\begin{definition}[Generator of Evolution System]
    For any fixed $s \in \R$, the \emph{right generator} of the evolution system is defined by the limit
    \begin{equation*}
        \gL_s f = \lim_{h \to 0^+} \dfrac{U_{s+h,s} f - f}{h}.
    \end{equation*}
    The domain $\dom(\gL_s)$ of the right generator $\gL_s$ at time $s$ consists of all functions $f \in C_b(E)$ such that the above limit exists in the $L^\infty$ norm, in which $C_b(E)$ denotes the space of bounded continuous functions on $E$.

    Similarly, the \emph{left generator} is defined as
    \begin{equation*}
        \gL^-_s f = \lim_{h \to 0^+} \dfrac{U_{s, s-h} f - f}{h},
    \end{equation*}
    with the corresponding domain denoted by $\dom(\gL^-_s)$.
    \label{def:right_and_left_generator}
\end{definition}

\subsection{Feller Semigroups}

For time-homogeneous evolution systems, observe that for any $s \in \sT$ and $t \geq 0$, we have
\begin{equation*}
    T_t := U_{s+t,s} = U_{s, s-t},
\end{equation*}
which implies that the left and right generators coincide, \emph{i.e.}, $\gL = \gL_s = \gL^-_s$. In this case, the family of linear operators $(T_t)_{t \geq 0}$ forms a semigroup, or more precisely, a \emph{one-parameter semigroup}, as defined below:
\begin{definition}[One-Parameter Semigroup]
    A family of linear operators $(T_t)_{t \geq 0}$ is called a \emph{one-parameter semigroup} if the following conditions are satisfied:
    \begin{enumerate}[label=(\arabic*)]
        \item $T_0 = \mathrm{id}$;
        \item $T_s T_t = T_{s+t}$, for any $s, t \geq 0$.
    \end{enumerate}
    \label{def:semigroup}
\end{definition}

For such semigroups, we define their generator as follows:
\begin{definition}[Generator of Semigroup]
    The generator of a semigroup is given by the limit
    \begin{equation*}
        \gL f = \lim_{h \to 0} \dfrac{T_h f - f}{h},
    \end{equation*}
    where the domain $\dom(\gL)$ consists of all functions $f \in C_b(E)$ for which the above limit exists in the $L^\infty$ norm.
    \label{def:generator}
\end{definition}

In general, the generator $\gL$ is a closed operator on its domain, which is dense in $C_b(E)$, and it maps $\dom(\gL)$ into $C_b(E)$~\citep{bottcher2013levy}.  To focus on a class of semigroups that are particularly suitable for Markovian dynamics, we now introduce the notion of a Feller semigroup.

\begin{definition}[Feller Semigroup]
    A one-parameter semigroup $(T_t)_{t \geq 0}$ is called a \emph{Feller semigroup} if the following conditions hold for every $f \in C_0(E)$:
    \begin{enumerate}[label=(\arabic*)]
        \item (Feller Property) $T_t f \in C_0(E)$ for any $t \geq 0$;
        \item (Positivity-Preserving Property) $T_t f \geq 0$ if $f \geq 0$, for any $t \geq 0$;
        \item (Contractivity) $\|T_t f\|_\infty \leq \|f\|_\infty$ for any $t \geq 0$;
        \item (Strong Continuity) The following limit holds:
        \begin{equation*}
            \lim_{h \to 0^+} \|T_h f - f\|_\infty = 0.
        \end{equation*}
    \end{enumerate}
    \label{def:feller_semigroup}
\end{definition}

We now introduce the concept of a Feller process, which is a time-homogeneous Markov process whose transition semigroup forms a Feller semigroup.

\begin{definition}[Feller Process]
    A time-homogeneous Markov process $(x_t)_{t\in \sT}$ on $E$ is called a \emph{Feller process} if its transition semigroup, defined as 
    \begin{equation*}
        T_t f(x) := \E[f(x_t) | x_0 = x]
    \end{equation*}
    is a Feller semigroup.
    \label{def:feller_process}
\end{definition}

An important identity associated with Feller processes and their generators is Dynkin's formula. This result provides a key connection between the infinitesimal generator and the expected value of a function evaluated along the process.
 
\begin{theorem}[Dynkin's Formula~\citep{bottcher2013levy}]
    For any $f \in \dom(\gL) \cap C_0(E)$ and stopping time $t$ with $\E[t|x_0 = x] < +\infty$, the following relation holds:
    \begin{equation*}
        \E[f(x_t) | x_0 = x] - f(x_0) = \E\left[\int_0^t \gL f(x_s)\dif s\bigg| x_0 = x\right].
    \end{equation*}
    \label{thm:dynkin}
\end{theorem}

To establish a more general framework that connects Feller evolution systems with Feller semigroups, we introduce an augmentation construction. This construction allows us to reinterpret a time-inhomogeneous process as a time-homogeneous one by lifting it into an augmented state space.

\begin{definition}[Augmented Process~\citep{wentzell1979theorie}]
    Let $(x_t)_{t\in \sT}$ be a Markov process on $E$ with the evolution system $(U_{t,s})_{s \leq t, s,t\in\sT}$. Then we define the following \emph{augmented process} $(\tilde x_s)_{s \geq 0}$ in the augmented probability space $(\tilde \Omega, \tilde \gF, \tilde \P)$ defined as follows:
    \begin{enumerate}[label=(\arabic*)]
        \item (State Space) The augmented state $\tilde x$ is defined as $\tilde x := (s, x) \in \sT \times E$, with the augmented state space $\sT \times E$ equipped with the following $\sigma$-algebra:
        \begin{equation*}
            \tilde \gB := \left\{ \tilde B \subset \sT \times E \bigg|  
            \tilde B_s := \left\{ x \in E \big| \tilde x = (s, x) \in \tilde B \right\} \in \gB(E) , 
            \forall s \in \sT \right\}
        \end{equation*}
        where $\gB(E)$ is the Borel $\sigma$-algebra on $E$;
        \item (Sample Space) The augmented event $\tilde \omega$ is defined as $\tilde \omega := (s, \omega) \in \sT \times \Omega := \tilde \Omega$, with the augmented sample space $\tilde \Omega$ equipped with the following $\sigma$-algebra:
        \begin{equation*}
            \tilde \gF := \left\{ \tilde F \subset \tilde \Omega \bigg| 
            \tilde F_s := \left\{ \omega \in \Omega \big| \tilde \omega = (s, \omega) \in \tilde F \right\} \in \gF, 
            \forall s \in \sT \right\};
        \end{equation*}
        \item (Probability Measure) The augmented process $(\tilde x_t)_{t\geq 0}$ is defined as 
        \begin{equation*}
            \tilde x_t(\tilde \omega) = \tilde x_t(s, \omega) := (s+t, x_{s+t}(\omega)),
        \end{equation*}
        and the augmented probability measure $\tilde \P$ is defined such that the following relation holds:
        \begin{equation*}
            \begin{aligned}
                &\tilde \P_{\tilde x_0}\left(\tilde x_t(\tilde \omega) \in \tilde B \right):= \tilde \P\left(\tilde x_t(\tilde \omega) \in \tilde B \big| \tilde x_0(\tilde \omega) = \tilde x\right) \\
                =& \tilde \P\left((s+t, x_{s+t}(\omega)) \in \tilde B \big| \tilde x_0(\omega) = (s, x_s(\omega)) = (s, x)\right) \\
                :=& \P\left(x_{s+t}(\omega) \in \tilde B_{s+t} \big| x_s(\omega) = x\right),
            \end{aligned}
        \end{equation*} 
        for any $\tilde x = (s, x) \in \sT \times E$, $t \geq 0$, and $\tilde B \in \tilde \gB$.

        For simplicity, we will omit the explicit dependence on $\omega$ and $\tilde \omega$ in the remainder of our discussion.
    \end{enumerate}
    \label{def:augmented}
\end{definition}

The following proposition establishes that the augmented process is indeed a Markov process on $\sT \times E$, provided that the original process is Markovian on $E$.

\begin{proposition}
    If the original process $(x_t)_{t\in \sT}$ is a Markov process on $E$, then the augmented process $(\tilde x_t)_{t\in \sT}$ is a Markov process on $\sT \times E$ with the following transition probability:
    \begin{equation*}
        \tilde \P\left(\tilde x_{t'} \in \tilde B \big| \tilde x_t = \tilde x\right) = \P\left(x_{s'+t'-t} \in \tilde B_{s'+t'-t} \big| x_{s'} = x\right),
    \end{equation*}
    for any $\tilde x = (s', x) \in \sT \times E$, $\tilde B \in \tilde \gB$, and $s', t' \geq t$.
    \label{prop:augmented_markov}
\end{proposition}

\begin{proof}
    Let $\tilde F \in \tilde \gF_t$ and $\tilde B \in \tilde \gB$ with $t' \geq t$, where $\tilde \gF_t$ denotes the filtration generated by the augmented process up to time $t$. We verify the Markov property of the augmented process by computing:
    \begin{equation*}
        \begin{aligned}
            &\int_{\tilde F} \vone_{\tilde B} (\tilde x_{t'}) \tilde \P_{\tilde x_0}\left(\dif \tilde \omega\right)
            = \int_{\tilde F} \vone_{\tilde B} (\tilde x_{t'}) \tilde \P\left(\dif \tilde \omega \big| \tilde x_0 = (s, x)\right)\\
            =& \int_{\tilde F_s} \vone_{\tilde B_{s+t'}} (x_{s+t'}) \P(\dif \omega | x_s = x) \\
            =& \int_{\tilde F_s} \P\left(x_{s+t'} \in \tilde B_{s+t'} \big| x_{s+t}\right) \P(\dif \omega|x_s = x) \\
            =& \int_{\tilde F} \P\left(\tilde x_{t'}\in \tilde B \big| \tilde x_t\right) \tilde \P\left(\dif \tilde \omega \big| \tilde x_0 = \tilde x\right),
        \end{aligned}
    \end{equation*}
    where the second equality is due to that conditioning $\tilde x_0(\tilde \omega) = (s, x)$ implies that $\tilde \omega = (s, \omega)$ such that $x_s(\omega) = x$, and
    the third equality is by the Markov property of the original process $(x_t)_{t\in \sT}$. 
    
    We now compute the transition probability of the augmented process explicitly. For any $x \in E$, $\tilde B \in \tilde \gB$, and $s', t' \geq t$, we have
    \begin{equation*}
        \begin{aligned}
            &\tilde \P\left(\tilde x_{t'} \in \tilde B \big| \tilde x_t = \tilde x\right) \\
            =& \tilde \P\left((s'+t'-t, x_{s'+t'-t}) \in \tilde B \big| \tilde x_t = (s', x)\right) \\
            =& \P\left(x_{s'+t'-t} \in \tilde B_{s'+t'-t} \big| x_{s'} = x\right).
        \end{aligned}
    \end{equation*}
    This shows that the transition kernel depends only on the difference $t' - t$, and hence the augmented process is time-homogeneous.
\end{proof}

We may now define the semigroup associated with the augmented time-homogeneous Markov process:
\begin{equation*}
    \tilde T_t f(\tilde x) = \tilde \E\left[f(\tilde x_t) | \tilde x_0 = \tilde x\right],
\end{equation*}
where $\tilde \E$ denotes the expectation with respect to the augmented probability measure $\tilde \P$. Correspondingly, the generator of the augmented process is introduced below.
\begin{definition}[Augmented Generator]
    The \emph{augmented generator} $\tilde \gL$ of the augmented process is defined as the following limit:
    \begin{equation*}
        \tilde \gL f(\tilde x) = \lim_{h \to 0} \dfrac{\tilde T_h f(\tilde x) - f(\tilde x)}{h},
    \end{equation*} 
    for any $f \in C_b(\sT \times E)$. We denote the domain $\dom(\tilde \gL)$ of the augmented generator $\tilde \gL$ as the space of all function $f \in C_b(\sT \times E)$ satisfying the conditions aforementioned such that the limit above exists with respect to $L^\infty$ norm.
    \label{def:augmented_generator}
\end{definition}

The next proposition establishes the relation between the augmented generator and the right generator of the original time-inhomogeneous Markov process.
\begin{proposition}
    Let $(x_t)_{t \in \sT}$ be a Markov process governed by the evolution system $(U_{t,s})_{s \leq t}$ with right generator $\gL_s$. Let $(\tilde x_t)_{t \geq 0}$ be the corresponding augmented process with semigroup $(\tilde T_t)_{t \geq 0}$ and generator $\tilde \gL$ as in \cref{def:augmented_generator}.

    Suppose $f \in C_b([0, T] \times E)$ satisfies the following conditions:
    \begin{enumerate}[label=(\arabic*)]
        \item $f(\cdot, x) \in C^1(\sT)$ for any $x \in E$;
        \item $f(s, \cdot), \partial_s f(s, \cdot) \in \dom(\gL_s)$ for any $s \in \sT$,
    \end{enumerate}
    then we have $f \in \dom(\tilde \gL)$ and the generator satisfies
    \begin{equation}
        \tilde \gL f(\tilde x) = \tilde \gL f(s, x) = \partial_s f(s, x) + \gL_s f(s, x).
        \label{eq:augmented_generator}
    \end{equation}
    \label{prop:augmented_generator}
\end{proposition}

\begin{proof}
    Starting from the definition, we write
    \begin{equation*}
        \begin{aligned}
            &\tilde \gL f(s, x) = \lim_{h \to 0} \dfrac{\tilde T_h f(s, x) - f(s, x)}{h}\\
            =& \lim_{h \to 0} \dfrac{\E\left[f(s+h, x_{s+h}) | x_s = x\right] - f(s, x)}{h}\\
            =&\lim_{h \to 0} \dfrac{\E\left[f(s+h, x_{s+h}) - f(s, x_{s+h}) + f(s, x_{s+h}) - f(s, x) | x_s = x\right]}{h}\\
            =&\lim_{h \to 0} \dfrac{\E\left[f(s+h, x_{s+h}) - f(s, x_{s+h})| x_s = x\right]}{h} + \lim_{h \to 0} \dfrac{\E\left[f(s, x_{s+h}) - f(s, x) | x_s = x\right]}{h},
        \end{aligned}
    \end{equation*}
    where the second term on the right-hand side matches the definition of the right generator $\gL_s$ of the original process $(x_t)_{t\in \sT}$ and thus converges to $\gL_sf(s, x)$ as $h \to 0$, given $f(s, \cdot) \in \dom(\gL_s)$.
    
    The first term on the right-hand side can be computed as follows:
    \begin{equation*}
        \begin{aligned}
            &\left|\dfrac{1}{h}\E\left[f(s+h, x_{s+h}) - f(s, x_{s+h})| x_s = x\right] - \partial_sf(s, x_s) \right|\\
            \leq & \left|\dfrac{1}{h}\E\left[f(s+h, x_{s+h}) - f(s, x_{s+h})| x_s 
            = x\right] - \E\left[\partial_sf(s, x_{s+h}) \big| x_s = x\right] \right| \\
            +& \left|\E\left[\partial_sf(s, x_{s+h}) \big| x_s = x\right] - \partial_sf(s, x) \right|\\
            =&\left|\dfrac{1}{h}\E\left[\int_s^{s+h} \left( \partial_\sigma f(\sigma, x_{s+h}) - \partial_sf(s, x_{s+h})\right) \dif \sigma \bigg| x_s = x\right]  \right| \\
            +& \left|\E\left[\partial_sf(s, x_{s+h}) \big| x_s = x\right] - \partial_sf(s, x) \right|,
        \end{aligned}
    \end{equation*}
    where the first term is bounded by the continuity of $\partial_s f(\cdot, x)$ as
    \begin{equation*}
        \begin{aligned}
            &\left|\dfrac{1}{h}\E\left[\int_s^{s+h} \left( \partial_\sigma f(\sigma, x_{s+h}) - \partial_sf(s, x_{s+h})\right) \dif \sigma \bigg| x_s = x\right]  \right|\\
            \leq & \sup_{\sigma \in [s, s+h], x \in E} \left|\partial_\sigma f(\sigma, x) - \partial_sf(s, x)\right| \to 0, \text{ as } h \to 0,
        \end{aligned}
    \end{equation*}
    and the second term is bounded as 
    \begin{equation*}
        \begin{aligned}
            &\left|\E\left[\partial_sf(s, x_{s+h}) \big| x_s = x\right] - \partial_sf(s, x) \right|\\
            = & U_{s+h,s} \partial_sf(s, \cdot)(x) - \partial_sf(s, x) \to 0, \text{ as } h \to 0,
        \end{aligned}
    \end{equation*}
    where we treat the derivative $\partial_s f(s, x)$ as a function of only $x$ and thus the limit converges to $0$ as $h \to 0$, given $\partial_s f(s, \cdot) \in \dom(\gL_s)$. The proof is thus completed.
\end{proof}

We now state an important corollary, which extends Dynkin's formula to time-inhomogeneous processes via the augmented framework. We will adopt the shorthand notation $f_t(\cdot) = f(t, \cdot)$ in the following discussions.

\begin{corollary}[Dynkin's Formula]
    Under the conditions of~\cref{prop:augmented_generator}, the following relation holds:
    \begin{equation*}
        \E\left[f_t(x_t) | x_0 = x\right] - f_0(x) = \E\left[\int_0^t \left(\partial_s f_s(x_s) + \gL_s f_s(x_s)\right) \dif s\right],
    \end{equation*}
    for any stopping time $t$ with $\E[t|x_0 = x] < +\infty$.
    \label{cor:dynkin}
\end{corollary}

\begin{proof}
    First, we apply~\cref{thm:dynkin} to the augmented process $(\tilde x_t)_{t\in \sT}$, and obtain the following relation:
    \begin{equation*}
        \tilde \E\left[f(\tilde x_t)| \tilde x_0 = (0, x)\right] - f(0, x) = \tilde \E\left[\int_0^t \tilde \gL f(\tilde x_s) \dif s \bigg| \tilde x_0 = (0, x) \right],
    \end{equation*}
    where the left-hand side, by~\cref{def:augmented}, can be rewritten as 
    \begin{equation*}
        \tilde \E\left[f(\tilde x_t)| \tilde x_0 = (0, x)\right] - f(0, x) = \E\left[f_t(x_t)| x_0 = x\right] - f_0(x),
    \end{equation*}
    while the right-hand side, similarly, can be rewritten as
    \begin{equation*}
        \tilde \E\left[\int_0^t \tilde \gL f(\tilde x_s) \dif s \bigg| \tilde x_0 = (0, x) \right] = \E\left[\int_0^t \left(\partial_s f_s(x_s) + \gL_s f_s(x_s)\right) \dif s\right],
    \end{equation*}
    where we used the relation~\eqref{eq:augmented_generator} in the last equality. The proof is thus completed.
\end{proof}

Finally, the next theorem establishes the equivalence between Feller evolution systems and Feller semigroups via the augmented process.

\begin{theorem}[{\citet[Theorem~3.2]{bottcher2014feller}}]
    Let $(x_t)_{t\in \sT}$ be a Markov process on $E$ and $(\tilde x_t)_{t\in \sT}$ be the augmented process on $\sT \times E$ as defined above. Then the following statements are equivalent:
    \begin{enumerate}[label=(\arabic*)]
        \item The evolution system $(U_{t,s})_{s \leq t}$ is a Feller evolution system;
        \item The semigroup $(\tilde T_t)_{t\in \sT}$ is a Feller semigroup, and thus the augmented process $(\tilde x_t)_{t\in \sT}$ is a Feller process.
    \end{enumerate}
    \label{thm:equivalence}
\end{theorem}

\begin{remark}
    The time index set $\sT$ may also be taken as a closed interval $[0, T]$ for some $T > 0$. In this case, we may restrict all time indices in the preceding discussion to lie within $[0, T]$. The identity in \cref{eq:augmented_generator} remains valid for all $s \in [0, T)$ since it is a local property. To ensure well defined behavior at the endpoint $T$, one may extend the Markov process $(x_t)_{t \in [0, T]}$ to $(x_t)_{t \in [0, T^+]}$ for some small $\epsilon > 0$, where $T^+ = T + \epsilon$, or more conveniently, to $(x_t)_{t \in \R}$.
\end{remark}

\subsection{Time-Reversal of Markov Processes}

Throughout this subsection, we fix the time index set as $\sT = [0, T]$ and adopt the following assumption.
\begin{assumption}[Feller]
    The Markov process $(x_t)_{t \in [0, T]}$ is governed by a Feller evolution system $(U_{t,s})_{s \leq t, s,t \in [0, T]}$. Therefore, by \cref{thm:equivalence}, the associated augmented process $(\tilde x_t)_{t \in [0, T]}$ is a Feller process. 
    \label{ass:feller}
\end{assumption} 
Additionally, we assume that both the original and augmented processes are c\'adl\'ag.

We now define the time reversal of a Markov process, following the formulation in~\citet{leonard2022feynman,cattiaux2023time}.
\begin{definition}[Time-Reversal of Markov Process]
    Let $(x_t)_{t \in [0, T]}$ be a Markov process. The associated \emph{time-reversal process} $(\cev x_t)_{t \in [0, T]}$ is defined by
    \begin{equation*}
        \cev x_t := \lim_{h\to 0^+} x_{T-t-h} = x_{(T-t)^-},
    \end{equation*}
    and by setting $\cev x_T := x_0$.
    \label{def:time_reversal}
\end{definition}

Without loss of generality, we assume that $x_{T^-} = x_T$ with probability 1. Consequently, the map $\cev{*}: (x_t)_{t \in [0, T]} \mapsto (\cev x_t)_{t \in [0, T]}$ is almost surely bijective. The filtration $(\cev \gF_t)_{t \in [0, T]}$ generated by the time-reversal process is defined by
\begin{equation*}
    \cev \gF_t = \bigcap_{s > t} \sigma(x_{T-s:T}),
\end{equation*}
where $\sigma(x_{s:T})$ denotes the $\sigma$-algebra generated by the path $(x_t)_{t\in[s, T]}$. We define the right generator $\cev \gL_s$ of the time-reversal process $(\cev x_t)_{t\in[0, T]}$ and the generator $\tilde{\cev{\gL}}_s$ of the augmented time-reversal process $(\tilde{\cev x}_t)_{t\in[0, T]}$, similar to \cref{def:right_and_left_generator} and \cref{def:generator}.

For the ease of presentation, we introduce the following function class assumption.
\begin{assumption}[Function Class $\gU$]
    There exists a class of functions $\gU$ on $E$ satisfying the following:
    \begin{enumerate}[label=(\arabic*)]
        \item $\gU$ is an algebra, \emph{i.e.}, closed under addition, (pointwise) multiplication, and scalar multiplication;
        \item $\gU \subset C_b(E)$ and is dense in $C_b(E)$;
        \item $\gU \subset \bigcap_{t\in[0,T]} \dom(\gL_t) \cap \dom(\cev \gL_t)$.
    \end{enumerate}
    \label{ass:U}
\end{assumption}

Next, we define the formal adjoint of a generator. This concept is motivated by the observation that we only require adjointness to hold for functions in $\gU$, rather than for all elements of $L^2(E, \mu)$.
\begin{definition}[Adjoint Operator]
    We define the domain $\dom(\gL^*)$ of the \emph{adjoint operator} $\gL^*$ as the space of all functions $g \in L^2(E, \mu)$ such that there exists a unique function $\gL^* g \in L^2(E, \mu)$ satisfying that for any $f \in \gU$, the following relation holds:
    \begin{equation}
        \angle{ \gL f, g } = \angle{ f, \gL^* g },
        \label{eq:adjoint}
    \end{equation}
    where $\angle{ \cdot, \cdot }$ denotes the inner product on $L^2(E, \mu)$.
    \label{def:adjoint}
\end{definition}

The uniqueness of $\gL^* g$ follows from the density of $\gU$ in $C_b(E)$ and hence in $L^2(E, \mu)$. We now introduce another function class $\gV \subset \gU$ as test functions for the function class $\gU$:
\begin{assumption}[Function Class $\gV$]
    There exists a class of functions $\gV$ on $E$ satisfying:
    \begin{enumerate}[label=(\arabic*)]
        \item $\gV \subset \bigcap_{t\in[0,T]} \dom(\gL_t^*)$;
        \item $g\gU = \{fg| f \in \gU\} \in \gV$, for any $g \in \gV$.
    \end{enumerate}
    \label{ass:V}
\end{assumption}

To ensure well-definedness and integrability of expressions involving the generator, we also make the following regularity assumption on marginal and conditional densities.
\begin{assumption}[Regularity of Densities]
    We assume that 
    \begin{enumerate}[label=(\arabic*)]
        \item The marginal distribution of the Markov process $(x_t)_{t\in[0, T]}$ at time $t$ admits a density $p_t(x)$ with respect to the base measure $\mu$ and $p_t \in \gV \subset \gU$, for any $t \in [0, T]$
        \item The conditional distribution $p_{t|s}(x_t|x_s)$ of the Markov process $(x_t)_{t\in[0, T]}$ , defined as 
        \begin{equation*}
            p_{t|s}(\cdot|x_s) = \P(x_t \in \cdot | x_s = x_s),
        \end{equation*}
        admits a density $p_{t|s}(x_t|x_s)$ with respect to the base measure $\mu$ and $p_{t|s}(\cdot|x_s) \in \gV \subset \gU$ for any $x_s \in E$, for any $s, t \in [0, T]$ with $s < t$;
        \item For any $f \in \gU$, the following conditions hold:
        \begin{equation*}
            \E\left[\int_0^T |\gL_t f(x_t)| \dif t \right] < +\infty, \quad \text{and} \quad \E\left[\int_0^T |\cev \gL_t f(\cev x_t)| \dif t \right] < +\infty.
        \end{equation*}
    \end{enumerate}
    \label{ass:density}
\end{assumption}

We are now ready to state a key result that relates the generator of the original process to that of the time-reversal process.
\begin{theorem}[Time-Reversal Generator]
    Under~\cref{ass:U,ass:density,ass:V}, for each $t \in [0, T]$ and any $f \in \gU$, the generator $\cev \gL_t$ of the time reversal process satisfies the identity
    \begin{equation}
        p_t \cev \gL_{T-t} f = \gL_t^*(p_t f) - f \gL_t^* p_t.
        \label{eq:time_reversal}
    \end{equation}
    If in addition $f \in \gV$ and $p_t(x) > 0$ for all $x \in E$, the identity~\cref{eq:time_reversal} can be rewritten as
    \begin{equation*}
        \cev \gL_{T-t} f = p_t^{-1} \gL_t^*(p_t f) - p_t^{-1} f \gL_t^* p_t = \gL_t^*f + p_t^{-1} \Gamma_t^*(p_t, f)
    \end{equation*}
    where $\Gamma_t^*(p_t, f) = \gL_t^*(p_t f) - p_t \gL_t^* f - f \gL_t^* p_t$ is the \emph{carr\'e du champ operator} associated with the adjoint operator $\gL_t^*$. 
    \label{thm:time_reversal_app}
\end{theorem}

\begin{proof}
    By Theorem~3.17 in~\citet{cattiaux2023time}, for any $t \in [0, T]$ and $f, g \in \gU$, the following relation holds:
    \begin{equation*}
        \E_{x_t \sim p_t}\left[ (\gL_t + \cev \gL_{T-t}) f(x_t) g(x_t) + \Gamma_t(f, g)(x_t) \right] = 0,
    \end{equation*}
    where $\Gamma_t(f, g) = \gL_t(fg) - f \gL_t g - g \gL_t f$ is the carr\'e du champ operator of the original generator $\gL_t$. 
    
    Substituting the definition of $\Gamma_t(f, g)$, we obtain
    \begin{equation*}
        \begin{aligned}
            &\angle{ p_t \cev \gL_{T-t} f, g } 
            = - \angle{ p_t \gL_t f, g } - \angle{ p_t, \Gamma_t(f, g) } \\
            =& - \angle{ p_t \gL_t f, g }  - \angle{ p_t, \gL_t(fg) - f \gL_t g - g \gL_t f } \\
            =& \angle{ p_t f, \gL_t g } - \angle{ p_t, \gL_t (fg) }.
        \end{aligned}
    \end{equation*}
    
    Since $p_t, f \in \gV$ and $fg \in \gU$, the definition of the adjoint operator yields
    \begin{equation*}
        \begin{aligned}
            &\angle{ p_t \cev \gL_{T-t} f, g } = \angle{ p_t f, \gL_t g } - \angle{ p_t, \gL_t (fg) }\\
            =& \angle{ \gL_t^* (p_t f), g } - \angle{ f  \gL_t^* p_t, g }\\
            =& \angle{ \gL_t^* (p_t f) - f  \gL_t^* p_t, g },
        \end{aligned}
    \end{equation*}
    which establishes the identity \eqref{eq:time_reversal}.
\end{proof}

For completeness, we state the Kolmogorov forward equation under the current framework.
\begin{theorem}[Kolmogorov Forward Equation]
    Under~\cref{ass:U,ass:density,ass:V}, the transition density $p_{t | 0}(x | x_0)$ satisfies the evolution equation
    \begin{equation*}
        \partial_t p_{t|0}(x|x_0) = \gL_t^* p_{t|0}(x|x_0).
    \end{equation*}
    In particular, the marginal density satisfies
    \begin{equation*}
        \partial_t p_t = \gL_t^* p_t.
    \end{equation*}
    \label{thm:kolmogorov_forward}
\end{theorem}

\begin{proof}
    Applying Dynkin's formula (\cref{cor:dynkin}), for any $f \in \gU$, we compute
    \begin{equation*}
        \begin{aligned}
            &\angle{f, \partial_t p_{t|0}(\cdot|x_0)}=\partial_t \angle{f, p_{t|0}(\cdot|x_0)} = \partial_t \E[f(x_t)|x_0 = x] \\
            =& \E\left[(\partial_t + \gL_t) f(x_t)|x_0 = x\right] = \E\left[\gL_t f(x_t)|x_0 = x\right],
        \end{aligned}
    \end{equation*}
    which can also be written as 
    \begin{equation*}
        \E\left[\gL_t f(x_t)|x_0 = x\right] = \angle{\gL_t f, p_{t|0}(\cdot|x_0)} = \angle{f, \gL_t^* p_{t|0}(\cdot|x_0)},
    \end{equation*}
    and the proof is thus complete.
\end{proof}

\subsection{Change of Measure}

Motivated by the form of the time-reversal generator $\cev \gL_{T-t}$ in \cref{thm:time_reversal_app}, we now consider the following structural assumption for a family of generators $(\gK_t)_{t \in [0, T]}$ associated with another Markov process $(y_t)_{t \in [0, T]}$. This formulation provides a rigorous version of \cref{ass:K} in the main text.
\begin{assumption}[Parametrization of Backward Generator]
    There exists a family of generators $(\gK_t)_{t \in [0, T]}$ corresponding to a Markov process $(y_t)_{t \in [0, T]}$ such that, for each $t \in [0, T]$, the generator $\gK_t$ and the adjoint operator $\gL_t^*$ are related by a function $\varphi_t \in \gU$ in the following way:
    \begin{equation}
        \varphi_t \gK_{T-t} f = \gL_t^*(\varphi_t f) - f \gL_t^* \varphi_t,
        \label{eq:K_param_app}
    \end{equation}
    for all $f \in \gU$.
    \label{ass:K_app}
\end{assumption}

To ensure that the identity in \eqref{eq:K_param_app} is well defined, we require the following regularity properties of the density ratio $\eta_t := \varphi_t / p_t$, which informally ensures that this ratio is smooth and bounded.

\begin{assumption}[Regularity of $\eta_t$]
    For each $t \in [0, T]$, the function $\eta_t := \varphi_t p_t^{-1}$ satisfies:
    \begin{enumerate}[label=(\arabic*)]
        \item $\eta_\cdot(\cdot), \eta_\cdot^{-1}(\cdot) \in C_b([0, T] \times E)$, $\eta_\cdot(x), \eta_\cdot(x)^{-1} \in C^1([0, T])$ for any $x \in E$, and 
        $$
            \eta_t, \partial_t \eta_t, \eta_t^{-1}, \partial_t \eta_t^{-1}, \log \eta_t, \partial_t \log \eta_t \in \gU
        $$ for any $t \in [0, T]$;
        \item The function $\cev \gL_{T-t} \eta_t$, satisfying that
        $$
            \cev \gL_{T-t} \eta_t = p_t^{-1} \gL_t^*(p_t \eta_t) - p_t^{-1} \eta_t \gL_t^* p_t,
        $$ 
        is upper bounded, \emph{i.e.}, 
        $$
            \sup_{x\in E} \cev \gL_{T-t} \eta_t(x) = \sup_{x\in E}  \dfrac{p_t(x) \gL_t^* \varphi_t(x) - \varphi_t(x) \gL_t^* p_t(x)}{p_t^2(x)} < +\infty.
        $$
    \end{enumerate}
    \label{ass:regularity_alpha}
\end{assumption}

We note that condition (2) in \cref{ass:regularity_alpha} implicitly requires both $p_t$ and $\varphi_t$ to be strictly positive on $E$.
Under \cref{ass:regularity_alpha} (1), it is easy to see that $\varphi_t = p_t \eta_t \in \gV$ and also $\varphi_t f = p_t \eta_t f \in \gV$ for any $f \in \gU$, and thus~\cref{eq:K_param_app} is well-defined under \cref{def:adjoint}, \emph{i.e.}, $\gU \subset \dom(\gK_t)$ for any $t \in [0, T]$.
Analogous to~\cref{thm:time_reversal_app}, if we further assume that $f \in \gV$ and $\varphi_t(x) >0 $ for any $x \in E$,~\cref{eq:K_param_app} can be rewritten as 
\begin{equation*}
     \gK_{T-t} f = \varphi_t^{-1} \gL_t^*(\varphi_t f) - \varphi_t^{-1} f \gL_t^* \varphi_t = \gL_t^* f + \varphi_t^{-1} \Gamma_t^*(\varphi_t, f).
\end{equation*}

We now express the generator $\gK_t$ in terms of the time-reversal generator $\cev \gL_{T-t}$ and the ratio $\eta_t = \varphi_t / p_t$.
\begin{lemma}
    For any $t \in [0, T]$, under~\cref{ass:K_app,ass:regularity_alpha}, the following relation holds for any $f\in \gU$:
    \begin{equation*}
        \eta_t \gK_{T-t} f = \cev \gL_{T-t}(\eta_t f) - f \cev \gL_{T-t} \eta_t = \eta_t \cev \gL_{T-t} f + \cev \Gamma_{T-t}(\eta_t, f),
    \end{equation*}
    where $\cev \Gamma_t(f,g) =  \cev \gL_t(fg) - f\cev \gL_t g - g \cev \gL_t f$ is the carr\'e du champ operator associated with the time-reversal generator $\cev \gL_t$.
    \label{prop:K_cev_L}
\end{lemma}

\begin{proof}
    For any functions $f, g$ that allow the following computation (\emph{cf.} \cref{ass:U,ass:V}), using the definition of $\eta_t$, we have that
    \begin{equation*}
        \angle{ g, \cev \gL_{T-t}(\eta_t f) - f \cev \gL_{T-t} \eta_t } = \angle{ g p_t^{-1}, p_t \cev \gL_{T-t}(\varphi_t p_t^{-1} f) } - \angle{ g p_t^{-1}, f p_t \cev \gL_{T-t}(\varphi_t p_t^{-1}) }.
    \end{equation*}
    Notice that by the time-reversal formula (\cref{thm:time_reversal_app}), we have 
    \begin{equation*}
        p_t \cev \gL_{T-t}(\varphi_t p_t^{-1} f) = \gL_t^*(\varphi_t f) - \varphi_t p_t^{-1} f \gL_t^* p_t \quad \text{and} \quad p_t \cev \gL_{T-t}(\varphi_t p_t^{-1}) = \gL_t^*(\varphi_t) - \varphi_t p_t^{-1} \gL_t^* p_t,
    \end{equation*}
    and a careful algebraic manipulation shows that
    \begin{equation*}
        \begin{aligned}
            &\angle{ g, \cev \gL_{T-t}(\eta_t f) - f \cev \gL_{T-t} \eta_t } \\
            =& \angle{ g p_t^{-1}, \gL_t^*(\varphi_t f)} - \angle{ g p_t^{-1}, \varphi_t p_t^{-1} f \gL_t^* p_t } - \angle{ g p_t^{-1}, f \gL_t^* \varphi_t} + \angle{ g p_t^{-1}, f \varphi_t p_t^{-1} \gL_t^* p_t } \\
            =& \angle{ g p_t^{-1}, \gL_t^*(\varphi_t f)} - \angle{ g p_t^{-1}, f \gL_t^* \varphi_t } = \angle{ g p_t^{-1}, \gL_t^*(\varphi_t f) - f \gL_t^* \varphi_t },
        \end{aligned}
    \end{equation*}
    which by~\cref{ass:K_app} is exactly
    \begin{equation*}
        \angle{ g, \cev \gL_{T-t}(\eta_t f) - f \cev \gL_{T-t} \eta_t } = \angle{ g p_t^{-1}, \varphi_t \gK_{T-t} f } = \angle{ g, \eta_t \gK_{T-t} f },
    \end{equation*}
    and the proof is complete.
\end{proof}

\begin{corollary}
    Under~\cref{ass:K_app,ass:regularity_alpha}, suppose any $f \in C_b([0, T] \times E)$ satisfying that 
    \begin{enumerate}[label=(\arabic*)]
        \item $f(\cdot, x) \in C^1([0, T])$ for any $x \in E$;
        \item $f(s, \cdot), \partial_s f(s, \cdot) \in \gU$ for any $s \in \sT$,
    \end{enumerate}
    then the following relation holds:
    \begin{equation}
        \eta_{T-\cdot}(\cdot) \tilde \gK f 
        = \tilde{\cev \gL}(\eta_{T-\cdot}(\cdot) f) - f \tilde{\cev \gL} \eta_{T-\cdot}(\cdot) 
        = \eta_{T-\cdot}(\cdot) \tilde{\cev \gL} f + \tilde{\cev \Gamma}(\eta_{T-\cdot}(\cdot), f),
        \label{eq:K_cev_L}
    \end{equation}
    where $\tilde \gK$ is the augmented generator of the augmented process $(\tilde y_t)_{t\in[0, T]}$ defined according to \cref{def:augmented}, and $\tilde{\cev \Gamma}(f, g) = \tilde{\cev \gL}(fg) - f \tilde{\cev \gL}g - g \tilde{\cev \gL}f$ is the carr\'e du champ operator associated with the augmented time-reversal generator $\tilde{\cev \gL}$.
    \label{cor:augmented_K_cev_L}
\end{corollary}

\begin{proof}
    We first apply \cref{prop:augmented_generator} to the augmented time-reversal process $(\tilde{\cev x}_t)_{t\in[0, T]}$ to obtain that the augmented generator $\tilde{\cev \gL}_t$ satisfies that
    \begin{equation}
        \tilde{\cev \gL} f(s, x) = \partial_s f(s, x) + \cev \gL_s f(s, x),
        \label{eq:augmented_generator_2}
    \end{equation}
    and then apply similar arguments to the Markov process $(y_t)_{t\in[0, T]}$ assumed in \cref{ass:K_app}, we also have that 
    \begin{equation*}
        \tilde \gK f(s, x) = \partial_s f(s, x) + \gK_s f(s, x).
    \end{equation*}
    Using \cref{prop:K_cev_L}, we compute
    \begin{equation*}
        \begin{aligned}
            &\eta_{T-s}(x) \tilde \gK f(s, x) 
            = \eta_{T-s}(x) \partial_s f(s, x) +  \eta_{T-s}(x)  \gK_s f(s, x)\\
            =& \eta_{T-s}(x) \partial_s f(s, x) +  \cev \gL_s(\eta_{T-s} f(s, \cdot))(x) - f(s, x) \cev \gL_s \eta_{T-s}(x)\\
            =&  \partial_s (\eta_{T-s}(x) f(s, x)) +  \cev \gL_s(\eta_{T-s} f(s, \cdot))(x) - f(s, x)\left(\partial_s \eta_{T-s}(x) + \cev \gL_s \eta_{T-s}(x)\right)\\
            =& \tilde{\cev \gL}(\eta_{T-\cdot}(\cdot) f) - f \tilde{\cev \gL}(\eta_{T-\cdot}(\cdot)),
        \end{aligned}
    \end{equation*}
    which proves the result.
\end{proof}

\begin{corollary}
    Under the conditions of \cref{cor:augmented_K_cev_L}, we further have that 
    \begin{equation*}
        \eta_{T-\cdot}^{-1}(\cdot) \tilde{\cev \gL} f 
        = \tilde \gK(\eta_{T-\cdot}^{-1}(\cdot) f) - f \tilde \gK \eta_{T-\cdot}^{-1}(\cdot) 
        = \eta_{T-\cdot}^{-1}(\cdot) \tilde \gK f + \tilde{\Gamma}_{\gK}(\eta_{T-\cdot}^{-1}(\cdot), f),
    \end{equation*}
    where $\tilde{\Gamma}_{\gK}(f, g) = \tilde \gK(fg) - f \tilde \gK g - g \tilde \gK f$ is the carr\'e du champ operator associated with the augmented generator $\tilde \gK$.
    \label{cor:augmented_cev_L_K}
\end{corollary}

\begin{proof}
    By taking $f = \eta_{T-\cdot}^{-1}(\cdot)$ in~\cref{eq:K_cev_L}, we have
    \begin{equation*}
        \eta_{T-\cdot}(\cdot) \tilde \gK \eta_{T-\cdot}^{-1}(\cdot) 
        = \tilde{\cev \gL}1 - \eta_{T-\cdot}^{-1}(\cdot) \tilde{\cev \gL} \eta_{T-\cdot}(\cdot) = - \eta_{T-\cdot}^{-1}(\cdot) \tilde{\cev \gL} \eta_{T-\cdot}(\cdot),
    \end{equation*}
    and also by taking $f = \eta_{T-\cdot}^{-1}(\cdot) g$ in~\cref{eq:K_cev_L}, we have
    \begin{equation*}
        \begin{aligned}
            &\eta_{T-\cdot}(\cdot) \tilde \gK (\eta_{T-\cdot}^{-1}(\cdot) g) 
            = \tilde{\cev \gL}(\eta_{T-\cdot}(\cdot) \eta_{T-\cdot}^{-1}(\cdot) g) - g \eta_{T-\cdot}^{-1}(\cdot) \tilde{\cev \gL} \eta_{T-\cdot}(\cdot)\\
            =& \tilde{\cev \gL} g + g \eta_{T-\cdot}(\cdot) \tilde \gK \eta_{T-\cdot}^{-1}(\cdot),
        \end{aligned}
    \end{equation*}
    and thus 
    \begin{equation*}
        \eta_{T-\cdot}^{-1}(\cdot) \tilde{\cev \gL} g = \tilde \gK (\eta_{T-\cdot}^{-1}(\cdot) g) - g \tilde \gK \eta_{T-\cdot}^{-1}(\cdot),
    \end{equation*}
    and the proof is complete.
\end{proof}

The relation of $\gK_t$ (resp., $\tilde \gK$) and the time-reversal generator $\cev \gL_t$ (resp., $\tilde{\cev \gL}$) in~\cref{prop:K_cev_L} can be regarded as applying to $\cev \gL_t$ (resp., $\tilde{\cev \gL}$) a specific form of ``perturbation'' involving the carr\'e du champ operator acting with the ratio $\eta_{T-t}$:
$$
    \eta_{T-t}^{-1} \cev \Gamma_t(\eta_{T-t}, \cdot)\quad \big( \text{resp.}\quad \eta_{T-\cdot}(\cdot)^{-1} \tilde{\cev \Gamma}(\eta_{T-\cdot}(\cdot), \cdot) \big),
$$    
which allows an exponential change of measure argument, as summarized in the following theorem:
\begin{theorem}[Change of Measure]
    Let $\sQ$ be another probability measure absolutely continuous with respect to $\P$. Suppose that, for any $t \in [0, T]$, the conditional expected log-Radon-Nikodym derivative satisfies
    \begin{equation}
        \E\left[\log\dfrac{\dif \sQ}{\dif \P}\bigg|_{\cev \gF_t}\right] = \E\left[- \int_0^t \left(\eta_{T-s}^{-1}(\cev x_s) \cev \gL_s \eta_{T-s}(\cev x_s) - \cev \gL_s \log \eta_{T-s}(\cev x_s) \right) \dif s\right],
        \label{eq:radon_nikodym}
    \end{equation}
    then the time-reversal process $(\cev x_t)_{t\in[0, T]}$ under the probability measure $\sQ$ is governed by the generator $\gK_t$.

    Furthermore, the original probability measure $\P$ is also absolutely continuous to $\sQ$, and the conditional expected log-Radon-Nikodym derivative above satisfies the following relation:
    \begin{equation*}
        \E\left[\log\dfrac{\dif \P}{\dif \sQ}\bigg|_{\cev \gF_t}\right] = \E\left[\int_0^t \left(\eta_{T-s}^{-1}(\cev x_s) \cev \gL_s \eta_{T-s}(\cev x_s) - \cev \gL_s \log \eta_{T-s}(\cev x_s) \right) \dif s\right].
    \end{equation*}
    \label{thm:change_of_measure_app}
\end{theorem}

\begin{proof}
    Based on \cref{cor:augmented_K_cev_L}, we apply Theorem~4.2 in~\citet{palmowski2002technique} to the augmented generator $\tilde{\cev \gL}$ of the augmented time-reversal process $(\tilde{\cev x}_t)_{t\in[0, T]}$ by verifying condition (M2) in Proposition~3.2 therein, and obtain that under another probability measure $\tilde \sQ$ absolutely continuous to $\tilde \P$, 
    the augmented time-reversal process $(\tilde{\cev x}_t)_{t\in[0, T]}$ is governed by the augmented generator $\tilde \gK$ and the Radon-Nikodym derivative is given by
    \begin{equation*}
        \dfrac{\dif \tilde \sQ}{\dif \tilde \P}\bigg|_{\tilde{\cev \gF}_t} = \dfrac{\eta_{T-t}(\cev x_t)}{\eta_T(\cev x_0)} \exp\left(- \int_0^t \eta_{T-s}^{-1}(\cev x_s) \tilde{\cev \gL} \eta_{T-s}(\cev x_s) \dif s \right),
    \end{equation*}
    and thus 
    \begin{equation*}
        \begin{aligned}
            \E\left[\log\dfrac{\dif \tilde \sQ}{\dif \tilde \P}\bigg|_{\tilde{\cev \gF}_t}\right] 
            =& \E\left[- \int_0^t \eta_{T-s}^{-1}(\cev x_s) \tilde{\cev \gL} \eta_{T-s}(\cev x_s) \dif s + \log\eta_{T-t}(\cev x_t) - \log\eta_T(\cev x_0)\right]\\
            =& \E\left[- \int_0^t \left(\eta_{T-s}^{-1}(\cev x_s) \tilde{\cev \gL} \eta_{T-s}(\cev x_s) - \tilde{\cev \gL} \log \eta_{T-s}(\cev x_s) \right) \dif s\right],
        \end{aligned}
    \end{equation*} 
    where the boundary terms $\log\eta_{T-t}(\cev x_t)-\log\eta_T(\cev x_0)$ are rewritten via Dynkin's formula (\cref{thm:dynkin}) applied to the time-inhomogeneous function $(s,x)\mapsto \log\eta_{T-s}(x)$ on the interval $[0,t]$.

    Notice that in the construction of the augmented processes (\cref{def:augmented}), no extra randomness is introduced apart from that from the probability space $(\Omega, \gF, \P)$, and thus we may define the probability measure $\sQ$ on the original probability space $(\Omega, \gF, \P)$ under which the time-reversal process $(\cev x_t)_{t\in[0, T]}$ is governed by the generator $\gK_t$ and the Radon-Nikodym derivative is given by
    \begin{equation*}
        \begin{aligned}
            &\E\left[\log\dfrac{\dif \sQ}{\dif \P}\bigg|_{\cev \gF_t}\right] := \E\left[\log\dfrac{\dif \tilde \sQ}{\dif \tilde \P}\bigg|_{\tilde{\cev \gF}_t}\right] \\
            =& \E\left[- \int_0^t \left(\eta_{T-s}^{-1}(\cev x_s) \tilde{\cev \gL} \eta_{T-s}(\cev x_s) - \tilde{\cev \gL} \log \eta_{T-s}(\cev x_s) \right) \dif s\right] \\
            =& \E\left[- \int_0^t \bigg(\eta_{T-s}^{-1}(\cev x_s) \left(\partial_s \eta_{T-s}(\cev x_s) + \cev \gL_s \eta_{T-s}(\cev x_s)\right) - \left(\partial_s \log \eta_{T-s}(\cev x_s) + \cev \gL_s \log \eta_{T-s}(\cev x_s) \right) \bigg) \dif s\right] \\
            =& \E\left[- \int_0^t \left(\eta_{T-s}^{-1}(\cev x_s) \cev \gL_s \eta_{T-s}(\cev x_s) - \cev \gL_s \log \eta_{T-s}(\cev x_s) \right) \dif s\right] ,
        \end{aligned}
    \end{equation*}
    where the second equality is due to~\cref{eq:augmented_generator_2} and the last equality is by 
    \begin{equation*}
        \eta_{T-s } \partial_s \log \eta_{T-s} = \partial_s \eta_{T-s}.
    \end{equation*}

    To prove the absolute continuity of $\P$ to $\sQ$, one applies the same argument to the augmented generator $\tilde \gK$ and the original generator $\gK_t$ of the Markov process $(y_t)_{t\in[0, T]}$ under the probability measure $\sQ$ and $\tilde \sQ$, respectively, given the symmetry of the two augmented generators $\tilde \gK$ and $\tilde{\cev \gL}$ (\cref{cor:augmented_cev_L_K}), and the proof is thus completed.
\end{proof}

\begin{proof}[Proof of~\cref{cor:error_bound}]
    The error bound is obtained from~\cref{thm:change_of_measure} by 
    \begin{equation}
        \KL(p_0 \| q_T) \leq \KL(p_{0:T} \| q_{0:T}) = \KL(p_T \| q_0) + \KL(p_{0:T} \| q_{0:T}|p_T),
        \label{eq:data_processing}
    \end{equation}
    where the first inequality is by applying the data-processing inequality and the second equality is by the chain rule of KL divergence. The conditional KL divergence $\KL(p_{0:T} \| q_{0:T}|p_T)$, defined as 
    \begin{equation*}
        \KL(p_{0:T} \| q_{0:T}|p_T) = \E_{\cev x_0 \sim p_T}[\KL(p_{0:T}(\cev x_{0:T}) \| q_{0:T}(\cev x_{0:T}))| \cev x_0],
    \end{equation*} 
    can be further bounded by the change-of-measure theorem~\cref{thm:change_of_measure} by aligning the two processes $(\cev x_t)_{t \in [0, T]}$ and $(y_t)_{t \in [0, T]}$ at time $0$ with probability distribution $\cev p_0 = p_T$.
\end{proof}

\begin{corollary}
    Under the conditions in~\cref{thm:change_of_measure_app}, the Radon-Nikodym derivative~\eqref{eq:radon_nikodym} further satisfies the following relation:
    \begin{equation}
        \KL(\P \| \sQ) = \E\left[\log\dfrac{\dif \P}{\dif \sQ}\right] =  \E\left[ \int_0^T \eta_s(x_s) \gL_s\eta_s^{-1}(x_s) + \gL_s \log \eta_s(x_s)  \dif s\right] := \mathfrak{L}[\eta_t].
        \label{eq:radon_nikodym_2}
    \end{equation}
    \label{cor:change_of_measure}
\end{corollary}

\begin{proof}
    By \cref{thm:change_of_measure_app}, we have that 
    \begin{align*}
        &\KL(\P \| \sQ) = \E\left[\log\dfrac{\dif \P}{\dif \sQ}\right]\\ 
        =& \E\left[ \int_0^T  \eta_{T-s}^{-1}(\cev x_s) \cev \gL_s \eta_{T-s}(\cev x_s) - \cev \gL_s \log \eta_{T-s}(\cev x_s)  \dif s \right] \\
        =& \int_0^T  \E_{\cev x_s \sim p_{T-s}}\left[ \eta_{T-s}^{-1}(\cev x_s) \cev \gL_s \eta_{T-s}(\cev x_s) - \cev \gL_s \log \eta_{T-s}(\cev x_s) \right]  \dif s \\
        =& \int_0^T \left(\angle{ p_{T-s}, \eta_{T-s}^{-1} \cev \gL_s \eta_{T-s} } - \angle{ p_{T-s}, \cev \gL_s \log \eta_{T-s} }\right) \dif s\\
        =& \int_0^T \left(\angle{ \eta_{T-s}^{-1},  p_{T-s} \cev \gL_s \eta_{T-s} } - \angle{1, p_{T-s} \cev \gL_s \log \eta_{T-s} }\right) \dif s.
    \end{align*}

    Notice that by \cref{thm:time_reversal_app}, we have the following time-reversal formula:
    \begin{equation*}
        p_{T-s}\cev \gL_s \eta_{T-s} = \gL_{T-s}^*(p_{T-s} \eta_{T-s}) - \eta_{T-s} \gL_{T-s}^* p_{T-s},
    \end{equation*}
    and 
    \begin{equation*}
        p_{T-s} \cev \gL_s \log \eta_{T-s} = \gL_{T-s}^*(p_{T-s} \log \eta_{T-s}) - \log \eta_{T-s} \gL_{T-s}^* p_{T-s}.
    \end{equation*}

    Thus, we continue our computation as follows:
    \begin{align*}
        & \KL(\P \| \sQ)\\
        =& \int_0^T \angle{ \eta_{T-s}^{-1}, \gL_{T-s}^*(p_{T-s} \eta_{T-s}) - \eta_{T-s} \gL_{T-s}^* p_{T-s} } \dif s\\
        & - \int_0^T \angle{ 1, \gL_{T-s}^*(p_{T-s} \log \eta_{T-s}) - \log \eta_{T-s} \gL_{T-s}^* p_{T-s} } \dif s\\
        =& \int_0^T \angle{ \eta_{T-s}^{-1}, \gL_{T-s}^*(p_{T-s} \eta_{T-s})} - \angle{ \eta_{T-s}^{-1}, \eta_{T-s} \gL_{T-s}^* p_{T-s} } \dif s\\
        & - \int_0^T \angle{ 1, \gL_{T-s}^*(p_{T-s} \log \eta_{T-s})} - \angle{ 1, \log \eta_{T-s} \gL_{T-s}^* p_{T-s} } \dif s\\
        =& \int_0^T \angle{\gL_{T-s}\eta_{T-s}^{-1}, p_{T-s}\eta_{T-s}} - \angle{ 1, \gL_{T-s}^* p_{T-s} } \dif s\\
        & - \int_0^T \angle{\gL_{T-s}1, p_{T-s}\log \eta_{T-s}} - \angle{\log \eta_{T-s}, \gL_{T-s}^* p_{T-s} } \dif s\\
        =& \int_0^T \left(\angle{ \eta_{T-s} \gL_{T-s}\eta_{T-s}^{-1}, p_{T-s} } + \angle{ \gL_{T-s} \log \eta_{T-s}, p_{T-s} } \right) \dif s\\
        =& \int_0^T \left(\angle{ \eta_s \gL_s\eta_s^{-1}, p_s } + \angle{ \gL_s \log \eta_s, p_s }\right) \dif s\\
        =& \int_0^T \E_{x_s \sim p_s}\left[  \eta_s(x_s) \gL_s\eta_s^{-1}(x_s) + \gL_s \log \eta_s(x_s) \right] \dif s
    \end{align*}
    where the third equality is by noticing that $\gL_t 1 = 0$ by the definition of the generator.
    The proof is thus complete.
\end{proof}

\begin{remark}[Direct route via generalized Doob's $h$-transform]
    The conclusion above may also be obtained more directly from \cref{thm:doob}. Indeed, by \cref{lem:alpha}, after replacing $t$ by $T-t$, we have
    \begin{equation*}
        \gK_t f
        =
        \eta_{T-t}^{-1}\cev \gL_t(\eta_{T-t} f)
        -
        \eta_{T-t}^{-1} f \cev \gL_t \eta_{T-t}.
    \end{equation*}
    Therefore $\gK_t$ is precisely the generalized Doob's $h$-transform of the true backward generator $\cev \gL_t$ with
    \begin{equation*}
        h_t = \eta_{T-t},
        \qquad
        \lambda_t = \eta_{T-t}^{-1}\cev \gL_t \eta_{T-t}.
    \end{equation*}
    Applying \cref{thm:doob} to the time-reversal process $(\cev x_t)_{t\in[0,T]}$ yields
    \begin{equation*}
        \dfrac{\dif \sQ}{\dif \P}(\cev x_{[0,T]})
        =
        \dfrac{\eta_0(\cev x_T)}{\eta_T(\cev x_0)}
        \exp\!\left(
            -\int_0^T
            \left(
                \eta_{T-t}^{-1}\partial_t \eta_{T-t}
                +
                \eta_{T-t}^{-1}\cev \gL_t \eta_{T-t}
            \right)(\cev x_t)\,\dif t
        \right).
    \end{equation*}
    Taking logarithms, inverting the Radon--Nikodym derivative, and then applying Dynkin's formula to the time-inhomogeneous function $(t,x)\mapsto \log \eta_{T-t}(x)$ give
    \begin{equation*}
        \KL(\P\|\sQ)
        =
        \E_{\P}\left[
            \int_0^T
            \left(
                \eta_{T-t}^{-1}\cev \gL_t \eta_{T-t}
                -
                \cev \gL_t \log \eta_{T-t}
            \right)(\cev x_t)\,\dif t
        \right].
    \end{equation*}
    Finally, \cref{thm:time_reversal} rewrites this in terms of the forward generator as
    \begin{equation*}
        \KL(\P\|\sQ)
        =
        \E_{\P}\left[
            \int_0^T
            \left(
                \eta_t \gL_t \eta_t^{-1}
                +
                \gL_t \log \eta_t
            \right)(x_t)\,\dif t
        \right],
    \end{equation*}
    which is exactly the identity in \cref{cor:change_of_measure}. 
    \label{rem:doob_direct_route}
\end{remark}

The next corollary rewrites the Radon-Nikodym derivative~\eqref{eq:radon_nikodym} in terms of the conditional distribution $p_{t|0}$.

\begin{corollary}[Score-Matching]
    Suppose the discussions above can be extended to the case where $p_0 = \delta_{x_0}$ for any $x_0 \in E$, then there exists a constant $C$ such that
    \begin{equation*}
        \KL(\P \| \sQ) = \E\left[\log\dfrac{\dif \P}{\dif \sQ}\right] = \mathfrak{L}_{\mathrm SM}[\eta_{t|0}] + C,
    \end{equation*}
    where we define the score-matching loss $\mathfrak{L}_{\mathrm SM}[\eta_{t|0}]$ as
    \begin{equation*}
        \mathfrak{L}_{\mathrm SM}[\eta_{t|0}] = \E_{x_0 \sim p_0} \left[ \int_0^T \E_{x_s \sim p_{s|0}(\cdot| x_0)}\left[ \eta_{s|0}(x_s|x_0) \gL_s\eta_{s|0}^{-1}(x_s|x_0) + \gL_s \log \eta_{s|0}(x_s|x_0)\right]  \dif s\right] ,
    \end{equation*}
    the constant $C$ only depends on the Markov process $x_t$, or more precisely, the marginal distribution $p_t$, the initial distribution $p_0$, and the conditional distributions $p_{t|0}(\cdot|x_0)$, and does not depend on $\varphi_t$.
    \label{cor:score_matching_app}
\end{corollary}

\begin{proof}
    Using similar arguments as in the proof of \cref{cor:change_of_measure}, we have 
    \begin{align*}
        &\KL(\P \| \sQ) = \E\left[\log\dfrac{\dif \P}{\dif \sQ}\right]=\int_0^T \left(\angle{ p_s, \eta_s \gL_s\eta_s^{-1} } + \angle{ p_s, \gL_s \log \eta_s }\right) \dif s\\
        =& \int_0^T \left(\angle{ \varphi_s, \gL_s\eta_s^{-1} } + \angle{ p_s, \gL_s \left(\log \varphi_s - \log p_s\right)}\right) \dif s\\
        \simeq &\int_0^T \left(\angle{ \gL_s^*\varphi_s, \varphi_s^{-1} p_s } + \angle{ p_s, \gL_s \log \varphi_s }\right) \dif s
    \end{align*}
    where the notation $\simeq$ indicates the equality up to a constant term that does not depend on $\varphi_{s|0}$.
    Further computation gives
    \begin{align*}
        &\KL(\P \| \sQ)  \\
        \simeq&\E_{x_0 \sim p_0} \left[ \int_0^T \left(\angle{  \gL_s^* \varphi_s, \varphi_s^{-1} p_{s|0}(\cdot|x_0) } + \angle{ p_{s|0}(\cdot|x_0), \gL_s \log \varphi_s }\right) \dif s\right]\\
        =& \E_{x_0 \sim p_0} \left[ \int_0^T \left(\angle{  \gL_s^* \varphi_s, \eta^{-1}_{s|0}(\cdot|x_0) } + \angle{ p_{s|0}(\cdot|x_0), \gL_s \left(\log \eta_{s|0} + \log p_{s|0} \right) }\right) \dif s\right] \\
        \simeq&\E_{x_0 \sim p_0} \left[ \int_0^T \left(\angle{ p_{s|0}(\cdot|x_0),\eta_{s|0}(\cdot|x_0) \gL_s\eta_{s|0}^{-1}(\cdot|x_0) } + \angle{ p_{s|0}(\cdot|x_0), \gL_s \log \eta_{s|0}(\cdot|x_0) }\right) \dif s \right]\\
        =&\E_{x_0 \sim p_0} \left[ \int_0^T \E_{x_s \sim p_{s|0}(\cdot| x_0)}\left[ \eta_{s|0}(x_s|x_0) \gL_s\eta_{s|0}^{-1}(x_s|x_0) + \gL_s \log \eta_{s|0}(x_s|x_0) \right] \dif s\right],
    \end{align*}
    where the omitted terms depend only on the forward process through $p_t$, $p_0$, and $p_{t|0}(\cdot|x_0)$. Collecting those terms yields the constant $C$, which does not depend on $\varphi_t$ (equivalently, on the model parameters).
\end{proof}

\subsection{Overall Error Bounds}
\label{app:overall_bound}

\begin{proof}[Proof of~\cref{thm:numerical_error}]
    We first consider one specific interval $t \in [\ell \kappa, (\ell+1) \kappa)$.    Similar to the arguments in~\cref{eq:data_processing}, we have, by the chain rule of KL divergence, that
    \begin{equation*}
        \KL(\cev p_{(\ell+1)\kappa} \| \hat q_{(\ell+1)\kappa}) - \KL(\cev p_{\ell\kappa} \| \hat q_{\ell\kappa}) 
        \leq \KL(\cev p_{(\ell+1)\kappa|\ell\kappa}\|\hat q_{(\ell+1)\kappa|\ell\kappa}|\cev p_{\ell\kappa}),
    \end{equation*}
    where the conditional KL divergence is defined as 
    \begin{equation*}
        \KL(\cev p_{(\ell+1)\kappa|\ell\kappa}\|\hat q_{(\ell+1)\kappa|\ell\kappa}|\cev p_{\ell\kappa}) 
            = \E_{\cev x_{\ell\kappa} \sim \cev p_{\ell\kappa}}\left[\KL(\cev p_{(\ell+1)\kappa|\ell\kappa}(\cdot|\cev x_{\ell\kappa})\|\hat q_{(\ell+1)\kappa|\ell\kappa}(\cdot|\cev x_{\ell\kappa}))\right].
    \end{equation*}
    In the following, we will omit the subscripts for the expectation with respect to the true backward process $(\cev x_t)_{t\in[0, T]}$ for brevity.

    For each $\cev x_{\ell\kappa} \sim \cev p_{\ell\kappa}$, we consider the following decomposition of the integrand:
    \begin{equation*}
        \begin{aligned}
            &\KL(\cev p_{(\ell+1)\kappa|\ell\kappa}(\cdot|\cev x_{\ell\kappa})\|\hat q_{(\ell+1)\kappa|\ell\kappa}(\cdot|\cev x_{\ell\kappa})) 
            = \E\left[\log \dfrac{\cev p_{(\ell+1)\kappa|\ell\kappa}(\cev x_{(\ell+1)\kappa}|\cev x_{\ell\kappa})}{\hat q_{(\ell+1)\kappa|\ell\kappa}(\cev x_{(\ell+1)\kappa}|\cev x_{\ell\kappa})} \bigg| \cev x_{\ell\kappa} \right]\\
            =&\E\left[\log \dfrac{\cev p_{(\ell+1)\kappa|\ell\kappa}(\cev x_{(\ell+1)\kappa}|\cev x_{\ell\kappa})}{q_{(\ell+1)\kappa|\ell\kappa}(\cev x_{(\ell+1)\kappa}|\cev x_{\ell\kappa})} \bigg| \cev x_{\ell\kappa} \right]
            + \E\left[\log \dfrac{q_{(\ell+1)\kappa|\ell\kappa}(\cev x_{(\ell+1)\kappa}|\cev x_{\ell\kappa})}{\hat q_{(\ell+1)\kappa|\ell\kappa}(\cev x_{(\ell+1)\kappa}|\cev x_{\ell\kappa})} \bigg| \cev x_{\ell\kappa} \right]
        \end{aligned}
    \end{equation*}
    where the first term is first bounded by the data-processing inequality as
    \begin{equation*}
        \KL(\cev p_{(\ell+1)\kappa|\ell\kappa}(\cdot|\cev x_{\ell\kappa})\|q_{(\ell+1)\kappa|\ell\kappa}(\cdot|\cev x_{\ell\kappa})) 
        \leq \KL(\cev p_{\ell\kappa:(\ell+1)\kappa}(\cdot|\cev x_{\ell\kappa})\|q_{\ell\kappa:(\ell+1)\kappa}(\cdot|\cev x_{\ell\kappa})),
    \end{equation*}
    which is further bounded by~\cref{thm:change_of_measure} as 
    \begin{equation*}
        \KL(\cev p_{\ell\kappa:(\ell+1)\kappa}(\cdot|\cev x_{\ell\kappa})\|q_{\ell\kappa:(\ell+1)\kappa}(\cdot|\cev x_{\ell\kappa})) \leq \E\left[ \int_{T-(\ell+1)\kappa}^{T-\ell\kappa} \left( \eta_t \gL_t \eta_t^{-1} + \gL_t \log \eta_t \right)(x_t) \dif t \bigg|  x_{T-\ell\kappa} \right].
    \end{equation*}

    The second term corresponding to the one-step error is bounded by~\cref{ass:one_step_simulation} after taking expectation over $\cev x_{\ell\kappa}$ as 
    \begin{equation*}
        \E\left[\E\left[\log \dfrac{q_{(\ell+1)\kappa|\ell\kappa}(\cev x_{(\ell+1)\kappa}|\cev x_{\ell\kappa})}{\hat q_{(\ell+1)\kappa|\ell\kappa}(\cev x_{(\ell+1)\kappa}|\cev x_{\ell\kappa})} \bigg| \cev x_{\ell\kappa}\right]\right] \lesssim \kappa^{1+r},
    \end{equation*}
    and by iterating for $\ell\in[0:L-1]$, we have 
    \begin{equation*}
        \begin{aligned}
            &\KL(\cev p_T \| \hat q_T) = \KL(\cev p_{L\kappa} \| \hat q_{L\kappa})\\
            \lesssim& \sum_{\ell=0}^{L-1} \left(\E\left[ \int_{T-(\ell+1)\kappa}^{T-\ell\kappa} \left( \eta_t \gL_t \eta_t^{-1} + \gL_t \log \eta_t \right)(x_t) \dif t\right] + \kappa^{1+r}\right) + \KL(\cev p_0 \| \hat q_0)\\
            \lesssim&  \KL(p_T \| q_0) + \mathfrak{L}(\eta_t) + T \kappa^r,
        \end{aligned}
    \end{equation*}
    and the proof is complete.
\end{proof}

\section{Example Details}
\label{app:proofs}

In this section, we present the proofs of the examples in the main text. Specifically, we will be verifying the assumptions that are required in \cref{app:math}.

\subsection{Diffusion Process}
\label{app:diffusion}

In this example, we choose the space $E = \R^d$, evidently a locally compact and separable space equipped with the Lebesgue measure. This example is the generalized case of continuous diffusion models.

\begin{definition}[H\"older Space]
    In $\R^d$, define the H\"older semi-norm as 
    \begin{equation*}
        |f|_{0,\eta} = \sup_{\vx, \vy \in \R^d,\ \vx \neq \vy} \dfrac{|f(\vx) - f(\vy)|}{\|\vx - \vy\|^\eta},
    \end{equation*}
    the H\"older space of order $k$ with exponent $\eta \in (0, 1]$ is defined as
    \begin{equation*}
        C^{k,\eta}(\R^d) = \left\{ f \in C^k(\R^d) \ \bigg|\ 
        \|f\|_{C^{k,\eta}}
        = \max_{|\vbeta| \le k} \sup_{\vx\in\R^d} |\partial_{\vbeta} f(\vx)|
        + \max_{|\vbeta| = k} |\partial_{\vbeta} f|_{0,\eta} < \infty \right\},
    \end{equation*}
    where $\vbeta$ is a multi-index,
    $|\vbeta|=\sum_{i=1}^d \beta_i$,
    and $\partial_{\vbeta}=\partial^{|\vbeta|}/\partial x_1^{\beta_1}\cdots \partial x_d^{\beta_d}$.
    \label{def:holder}
\end{definition}

\paragraph{Generator of Diffusion Processes.}

For a diffusion process $(\vx_t)_{t \in [0, T]}$ on $\R^d$,
the generator $\gL_t$ is given by 
\begin{equation}
    \gL_t f = \vb_t(\vx) \cdot \nabla f + \dfrac{1}{2} \mD_t(\vx) : \nabla^2 f,
    \label{eq:generator_diffusion}
\end{equation}
where $\vb_t(\vx) = (b_t^i(\vx))_{i \in [d]} \in \R^d$ is the drift vector, and $\mD_t(\vx) = (D_t^{ij}(\vx))_{i,j\in[d]} \in \R^{d \times d}$ is the diffusion matrix.

The diffusion process $(\vx_t)_{t\geq 0}$ governed by the generator $\gL_t$ can also be expressed in the following SDE form:
\begin{equation*}
    \dif \vx_t = \vb_t(\vx_t) \dif t + \mSigma_t(\vx_t) \dif \vw_t,
\end{equation*}
where the matrix $\mSigma_t(\vx)$ satisfies $\mSigma_t(\vx) \mSigma_t^\top(\vx) = \mD_t(\vx)$ for any $\vx \in \R^d$ and $t \geq 0$, and $(\vw_t)_{t \geq 0}$ is a standard $d$-dimensional Wiener process.

To proceed with further discussions, we make the following assumptions on the coefficients of the generator $\gL_t$:
\begin{assumption}[Regularity of the Diffusion Process]
    We assume the following smoothness and regularity conditions on the coefficients of the generator $\gL_t$~\eqref{eq:generator_diffusion} for any $t \in [0, T]$:
    \begin{enumerate}[label=(\arabic*)]
        \item The drift vector $\vb_t(\vx)$ satisfies that $\vb_t^i \in C^{1,0}(\R^d)$ for any $i \in [d]$ and $\vx \in \R^d$;
        \item The diffusion matrix $\mD_t(\vx)$ is positive semidefinite and satisfies that $D_t^{ij} \in C^{2,0}(\R^d)$ for any $i, j \in [d]$ and $\vx \in \R^d$.
    \end{enumerate}
    \label{ass:regularity_diffusion}
\end{assumption}

\paragraph{Domain of the Generator.}

Set the function class $\gU = C^{2,0}(\R^d)$, it is straightforward to see that $\gU \subset \dom(\gL_t)$ for any $t \in [0, T]$ under~\cref{ass:regularity_diffusion}. Then the function class $\gV$ can be chosen accordingly as sufficiently rapidly decreasing functions, \emph{e.g.}, 
$$
    \gV = \left\{f \in C^2(\R^d) \bigg| \sup_{\vx \in \R^d, |\vbeta| \leq 2} \left|\|\vx\|^k \partial_\vbeta f(\vx)\right| < + \infty, \forall k \in \sN \right\},
$$
such that~\cref{eq:adjoint} satisfies for any $f \in \gU$ and $g \in \gV$:
\begin{equation*}
    \begin{aligned}
        &\angle{\gL_t f,  g} = \int_{\R^d} g(\vx) \gL_t f(\vx) \dif \vx\\
        =& \int_{\R^d} g(\vx) \left(\vb_t(\vx) \cdot \nabla f(\vx) + \dfrac{1}{2} \mD_t(\vx) : \nabla^2 f(\vx)\right) \dif \vx\\
        =& -\int_{\R^d} \nabla \cdot (\vb_t(\vx) g(\vx)) f(\vx) \dif \vx + \int_{\R^d} \dfrac{1}{2} \nabla^2 : \left(\mD_t(\vx) g(\vx)\right) f(\vx)  \dif \vx\\
        =& \angle{f,  \gL_t^* g},
    \end{aligned}
\end{equation*}
where the integration by parts is justified by the choice of $\gV$ in the second-to-last equality, and thus
\begin{equation*}
    \gL_t^* g = -\nabla \cdot (\vb_t g) + \dfrac{1}{2} \nabla^2 : \left(\mD_t g\right),
\end{equation*}
which implies that $\gV \subset \dom(\gL_t^*)$ for any $t \in [0, T]$. It is also easy to check that for any $g \in \gV$, $g \gU \subset \gV$, and thus both~\cref{ass:U} and~\cref{ass:V} are satisfied.

\paragraph{Backward Generator.}

\cref{thm:time_reversal_app} gives the exact form of the backward generator $\cev \gL_t$ as follows:
\begin{align*}
        &\cev \gL_{T-t} f = p_t^{-1} \gL_t^*(p_t f) - p_t^{-1} f \gL_t^* p_t \\
        =& - p_t^{-1} \nabla \cdot(\vb_t p_t f) + \dfrac{1}{2}p_t^{-1} \nabla^2 : (\mD_t p_t f) + p_t^{-1} f \nabla \cdot(\vb_t p_t) - \dfrac{1}{2} p_t^{-1} f \nabla^2 : (\mD_t p_t) \\
        =& - p_t^{-1} \nabla \cdot(\vb_t p_t) f - p_t^{-1} \vb_t p_t \cdot \nabla f + \dfrac{1}{2}p_t^{-1} \nabla \cdot (f \nabla \cdot (\mD_t p_t)) + \dfrac{1}{2}p_t^{-1} \left(\nabla \cdot (\mD_t p_t)\right) \cdot \nabla f \\
        &+ p_t^{-1} f \nabla \cdot(\vb_t p_t) - \dfrac{1}{2} p_t^{-1} f \nabla^2 : (\mD_t p_t)\\
        =& - \vb_t  \cdot \nabla f + \dfrac{1}{2} \mD_t: \nabla^2 f + p_t^{-1} \left(\nabla \cdot (\mD_t p_t)\right) \cdot \nabla f \\
        =& \left(-\vb_t + \mD_t \nabla \log p_t + \nabla \cdot \mD_t \right) \cdot \nabla f + \dfrac{1}{2} \mD_t: \nabla^2 f,
\end{align*}  
where we adopted the notation 
$$
    \nabla \cdot \mD_t(\vx) = \left(\textstyle\sum_{j=1}^d \partial_j D_t^{ij}(\vx)\right)_{i \in [d]},
$$ 
corresponding to the following backward SDE:
\begin{equation}
    \dif \cev \vx_t = \left(-\vb_t(\cev \vx_t) + \mD_t(\cev \vx_t) \nabla \log p_t(\cev \vx_t) + \nabla \cdot \mD_t(\cev \vx_t) \right) \dif t + \mSigma_t(\cev \vx_t) \dif \vw_t,
\end{equation}
where one often defines $\vs_t = \nabla \log p_t$ as the score function.

Similarly, \cref{ass:K_app} reduces to the following form:
\begin{equation*}
    \gK_{T-t} f = \left(-\vb_t + \mD_t \nabla \log \varphi_t + \nabla \cdot \mD_t \right) \cdot \nabla f + \dfrac{1}{2} \mD_t: \nabla^2 f,
\end{equation*}
and thus the Markov process $(y_t)_{t \in [0, T]}$ assumed in~\cref{ass:K_app} corresponds to another diffusion process with the following SDE:
\begin{equation*}
    \dif \vy_t = \left(-\vb_{T-t}(\vy_t) + \mD_{T-t}(\vy_t) \hat \vs_{T-t}(\vy_t) + \nabla \cdot \mD_{T-t}(\vy_t) \right) \dif t + \mSigma_{T-t}(\vy_t) \dif \vw_t,
\end{equation*}
where $\hat \vs_t = \nabla \log \varphi_t$ is an estimate of the true score function $\vs_t$.

\paragraph{KL Divergence and Loss Function.}

We also compute the KL divergence in \cref{cor:change_of_measure} explicitly as follows:
\begin{align*}
    &\KL(\P \| \sQ) =  \E\left[ \int_0^T \eta_t(\vx_t) \gL_t\eta_t^{-1}(\vx_t) + \gL_t \log \eta_t(\vx_t)  \dif t\right]\\
    =& \E\bigg[ \int_0^T \bigg( \eta_t(\vx_t) \vb_t(\vx_t) \cdot \nabla \eta_t^{-1}(\vx_t) + \dfrac{1}{2} \eta_t(\vx_t) \mD_t(\vx_t) : \nabla^2 \eta_t^{-1}(\vx_t) \\
    &\quad \quad \quad \quad + \vb_t(\vx_t) \cdot \nabla \log \eta_t(\vx_t) + \dfrac{1}{2} \mD_t(\vx_t) : \nabla^2 \log \eta_t(\vx_t) \bigg) \dif t\bigg]\\
    =& \E\left[\int_0^T \dfrac{1}{2} \mD_t(\vx_t) : \nabla \log \eta_t(\vx_t) \nabla^\top \log \eta_t(\vx_t) \dif t \right], \numberthis \label{eq:diffusion_KL}
\end{align*}
where the last equality is by the following identity:
\begin{equation*}
    \begin{aligned}
        \eta_t \nabla^2 \eta_t^{-1}
        &= \eta_t \nabla\left(-\eta_t^{-1}\nabla \log \eta_t\right) \\
        &= -\eta_t\Big[\nabla \eta_t^{-1}\nabla^\top \log \eta_t
              + \eta_t^{-1}\nabla^2 \log \eta_t\Big] \\
        &= -\nabla^2 \log \eta_t
           + \nabla \log \eta_t\nabla^\top \log \eta_t .
    \end{aligned}        
\end{equation*}
Notice that 
$$
    \nabla \log \eta_t(\vx_t) = \nabla \log \varphi_t(\vx_t) - \nabla \log p_t(\vx_t) = \hat \vs_t(\vx_t) - \nabla \log p_t(\vx_t),
$$
if we define $\vs_t = \nabla \log p_t$ as the score function and $\hat \vs_t = \nabla \log \varphi_t$ as its estimation, then the KL divergence can be rewritten as
\begin{equation*}
    \KL(\P \| \sQ) = \E\left[\int_0^T \dfrac{1}{2} \mD_t(\vx_t) : \left(\hat \vs_t(\vx_t) - \nabla \log p_t(\vx_t)\right) \left(\hat \vs_t(\vx_t) - \nabla \log p_t(\vx_t)\right)^\top \dif t\right],
\end{equation*}
and the corresponding score-matching loss (\cref{cor:score_matching_app}) is thus 
\begin{equation*}
    \begin{aligned}
        \KL(\P \| \sQ) = &\E_{\vx_0 \sim p_0}\bigg[\int_0^T \E_{\vx_t \sim p_{t|0}(\cdot| \vx_0)}\bigg[\\
        & \dfrac{1}{2} \mD_t(\vx_t) : \left(\hat \vs_t(\vx_t) - \nabla \log p_{t|0}(\vx_t|\vx_0)\right) \left(\hat \vs_t(\vx_t) - \nabla \log p_{t|0}(\vx_t|\vx_0)\right)^\top \bigg] \dif t\bigg] +C.
    \end{aligned}
\end{equation*}
For diagonal diffusion matrices $\mD_t(\vx)$, the loss function can be further simplified as
\begin{equation*}
    \KL(\P \| \sQ) = \E_{\vx_0 \sim p_0}\bigg[\int_0^T \E_{\vx_t \sim p_{t|0}(\cdot| \vx_0)}\bigg[\dfrac{1}{2} \diag D_t(\vx_t) \left(\hat \vs_t(\vx_t) - \nabla \log p_{t|0}(\vx_t|\vx_0)\right)^2 \bigg] \dif t \bigg]+C,
\end{equation*}
where the square is element-wise.

\subsection{Jump Process}
\label{app:jump_process}

In this example, we choose the space $E = \sX$ to be a finite set equipped with the discrete topology and the counting measure $\mu$. 
This example is the generalized case of discrete diffusion models.

\paragraph{Generator of Jump Processes.} For a jump process $(\vx_t)_{t \in [0, T]}$ on $\sX$, the generator $\gL_t$ is given by
\begin{equation*}
    \gL_t f(x) = \int_\sX (f(y) - f(x)) \lambda_t(y, x) \mu(\dif y) = \sum_{y \in \sX} (f(y) - f(x)) \lambda_t(y, x),
\end{equation*}
where $\lambda_t(\cdot, x)$ is the rate of jumps from $x$ to $y$ at time $t$.

The jump process $(\vx_t)_{t\geq 0}$ governed by the generator $\gL_t$ can also be expressed in the form of a continuous-time Markov chain:
\begin{equation*}
    \dfrac{\dif}{\dif t} \vp_t = \mLambda_t \vp_t, \quad \text{with} \quad \Lambda_t(y, x) = 
    \begin{cases}
        \lambda_t(y, x), &\text{if } y \neq x,\\
        - \sum_{y' \neq x} \lambda_t(y', x), &\text{if } y = x,
    \end{cases}
\end{equation*}
where $\vp_t = (p_t(x))_{x \in \sX}$ is the vector of probability masses at time $t$.

\paragraph{Domain of the Generator.} The choices of the function class $\gU$ and $\gV$ are trivial in this case, as the generator $\gL_t$ is a linear operator on the space of real-valued functions on $\sX$, and thus $\gU = \gV = \R^{|\sX|}$ by noticing that for any $f, g \in \R^{|\sX|}$, we have 
\begin{equation*}
    \begin{aligned}
        &\angle{\gL_t f, g} = \sum_{x \in \sX} g(x) \gL_t f(x) = \sum_{x, y \in \sX} g(x) (f(y) - f(x)) \lambda_t(y, x) \\
        =& \sum_{x, y \in \sX} g(y) f(x) \lambda_t(x, y) - g(x) f(x) \lambda_t(x, y)\\
        =& \sum_{x \in \sX} f(x) \sum_{y \in \sX} \left(g(y) \lambda_t(x, y) - g(x) \lambda_t(y, x)\right),
    \end{aligned}
\end{equation*}
\emph{i.e.}, \cref{eq:adjoint} holds with
\begin{equation*}
    \gL_t^* g(x) = \sum_{y \in \sX} (g(y) \lambda_t(x, y) - g(x) \lambda_t(y, x)).
\end{equation*}

\paragraph{Backward Generator.} The generator $\cev \gL_t$ of the time-reversal process $(\cev x_t)_{t\in[0, T]}$ is given by \cref{ass:K_app} as
\begin{equation*}
    \begin{aligned}
        &\cev \gL_{T-t} f = p_t^{-1} \gL_t^*(p_t f) - p_t^{-1} f \gL_t^* p_t\\
        =& \dfrac{1}{p_t(x)} \sum_{y \in \sX} \left(p_t(y) f(y) \lambda_t(x, y) - p_t(x) f(x) \lambda_t(y, x)\right)  - \dfrac{f(x)}{p_t(x)}  \sum_{y \in \sX} \left(p_t(y) \lambda_t(x, y) - p_t(x) \lambda_t(y, x)\right)\\
        =& \sum_{y \in \sX} (f(y) - f(x)) \dfrac{p_t(y)}{p_t(x)} \lambda_t(x, y) = \sum_{y \in \sX} (f(y) - f(x)) s_t(x, y) \lambda_t(x, y),
    \end{aligned}
\end{equation*} 
where we define the score function as 
\begin{equation*}
    \vs_t(x) = \left(s_t(x, y)\right)_{y \in \sX} = \left(\frac{p_t(y)}{p_t(x)}\right)_{y \in \sX}.
\end{equation*}
Thus, the backward process $(\cev x_t)_{t\in[0, T]}$ also corresponds to another jump process, which can be expressed as the following continuous-time Markov chain:
\begin{equation*}
    \dfrac{\dif}{\dif t} \vq_t = \overline \mLambda_t \vq_t, 
    \quad \text{with} \quad \overline\Lambda_t(y, x) = 
    \begin{cases}
        s_{T-t}(x, y) \lambda_{T-t}(x, y), &\text{if } y \neq x,\\
        - \sum_{y' \neq x} s_{T-t}(x, y') \lambda_{T-t}(x, y'), &\text{if } y = x.
    \end{cases}
\end{equation*}

\cref{ass:K_app} can be similarly reduced to 
\begin{equation*}
    \gK_{T-t} f = \sum_{y \in \sX} (f(y) - f(x)) \dfrac{\varphi_t(y)}{\varphi_t(x)} \lambda_t(x, y) = \sum_{y \in \sX} (f(y) - f(x)) \hat s_t(x, y) \lambda_t(x, y),
\end{equation*}
where $\hat \vs_t$, defined as
\begin{equation*}
    \hat \vs_t(x) = \left(\hat s_t(x, y)\right)_{y \in \sX} = \left(\frac{\varphi_t(y)}{\varphi_t(x)}\right)_{y \in \sX},
\end{equation*}
is an estimate of the true score function $\vs_t$. And the estimated backward process $(y_t)_{t \in [0, T]}$ corresponds to the following continuous-time Markov chain:
\begin{equation*}
    \dfrac{\dif}{\dif t} \hat \vq_t = \hat{\overline \mLambda}_t \hat \vq_t, 
    \quad \text{with} \quad \hat{\overline \Lambda}_t(y, x) = 
    \begin{cases}
        \hat s_{T-t}(x, y) \lambda_{T-t}(x, y), &\text{if } y \neq x,\\
        - \sum_{y' \neq x} \hat s_{T-t}(x, y') \lambda_{T-t}(x, y'), &\text{if } y = x.
    \end{cases}
\end{equation*}

\paragraph{KL Divergence and Loss Function.} The KL divergence in \cref{cor:change_of_measure} can be computed explicitly as
\begin{align*}
        &\KL(\P \| \sQ) = \E\left[\int_0^T \eta_t(x_t) \gL_t\eta_t^{-1}(x_t) + \gL_t \log \eta_t(x_t)  \dif t\right] \\
        =& \E\bigg[\int_0^T \bigg( \sum_{y \in \sX} \eta_t(x_t)\left(\eta_t^{-1}(y) - \eta_t^{-1}(x_t)\right)  \lambda_t(y, x_t) + \sum_{y \in \sX} \left(\log \eta_t(y) - \log \eta_t(x_t) \right) \lambda_t(y, x_t) \bigg) \dif t\bigg]\\
        =& \E\left[\int_0^T \sum_{y \in \sX} \left(\frac{\eta_t(x_t)}{\eta_t(y)} - 1 - \log \frac{\eta_t(x_t)}{\eta_t(y)}\right) \lambda_t(y, x_t) \dif t\right] \numberthis \label{eq:jump_KL}\\
        =& \E\left[\int_0^T \sum_{y \in \sX} \left( \frac{\varphi_t(x_t) p_t(y)}{\varphi_t(y) p_t(x_t)} - 1 - \log \frac{\varphi_t(x_t) p_t(y)}{\varphi_t(y) p_t(x_t)} \right) \lambda_t(y, x_t) \dif t\right]\\
        =& \E\left[\int_0^T \sum_{y \in \sX} \left(\frac{s_t(x_t, y)}{\hat s_t(x_t, y)} -1 - \log \frac{s_t(x_t, y)}{\hat s_t(x_t, y)} \right)  \lambda_t(y, x_t) \dif t\right].
\end{align*}
After swapping dummy variables in the double-sum representation, this can be further transformed as follows:
\begin{align*}
    &\KL(\P \| \sQ) = \E\left[\int_0^T \sum_{y \in \sX} \left(\frac{\hat s_t(x_t, y)}{s_t(x_t, y)} -1 - \log \frac{\hat s_t(x_t, y)}{s_t(x_t, y)} \right)  s_t(x_t, y)\lambda_t(x_t, y) \dif t\right]\\
    =& \int_0^T \sum_{x \in \sX} p_t(x) \sum_{y \in \sX} \left(\frac{\hat s_t(x, y)}{s_t(x, y)} -1 - \log \frac{\hat s_t(x, y)}{s_t(x, y)} \right)  s_t(x, y)\lambda_t(x, y) \dif t\\
    =& \int_0^T \sum_{x \in \sX} p_t(x) \sum_{y \in \sX} \left(\frac{\hat s_t(x, y)}{s_t(x, y)} -1 - \log \frac{\hat s_t(x, y)}{s_t(x, y)} \right) \dfrac{p_t(y)}{p_t(x)} \lambda_t(x, y) \dif t\\
    =& \E\left[\int_0^T \sum_{y \in \sX} \left(\frac{\hat s_t(x_t, y)}{s_t(x_t, y)} -1 - \log \frac{\hat s_t(x_t, y)}{s_t(x_t, y)} \right) s_t(x_t, y) \lambda_t(x_t, y) \dif t\right].
\end{align*}

By~\cref{cor:score_matching_app}, the corresponding score-matching loss is
\begin{align*}
    &\KL(\P \| \sQ) \\
    =& \E_{x_0 \sim p_0}\left[\int_0^T \E_{x_t \sim p_{t|0}(\cdot| x_0)}\left[\sum_{y \in \sX} \left( \frac{\varphi_t(x_t) p_{t|0}(y|x_0)}{\varphi_t(y) p_{t|0}(x_t|x_0)} - 1 - \log \frac{\varphi_t(x_t) p_{t|0}(y|x_0)}{\varphi_t(y) p_{t|0}(x_t|x_0)} \right) \lambda_t(x_t, y) \right] \dif t\right] + C\\
    =& \E_{x_0 \sim p_0}\left[\int_0^T \sum_{x\in\sX} p_{t|0}(x|x_0)\sum_{y \in \sX} \left( \frac{\varphi_t(x) p_{t|0}(y|x_0)}{\varphi_t(y) p_{t|0}(x|x_0)} - \log \frac{\varphi_t(x)}{\varphi_t(y)} \right) \lambda_t(x, y) \dif t\right] + C\\
    =& \E_{x_0 \sim p_0}\left[\int_0^T \sum_{x\in\sX} p_{t|0}(x|x_0) \sum_{y \in \sX} \left( \frac{\varphi_t(y) p_{t|0}(x|x_0)}{\varphi_t(x) p_{t|0}(y|x_0)} - \log \frac{\varphi_t(y)}{\varphi_t(x)} \right) \dfrac{p_{t|0}(y|x_0)}{p_{t|0}(x|x_0)} \lambda_t(x, y) \dif t\right] + C\\
    =& \E_{x_0 \sim p_0}\left[\int_0^T \E_{x_t \sim p_{t|0}(\cdot| x_0)}\left[ \sum_{y \in \sX} \left( \frac{\varphi_t(y) p_{t|0}(x_t|x_0)}{\varphi_t(x_t) p_{t|0}(y|x_0)} - \log \frac{\varphi_t(y)}{\varphi_t(x_t)} \right) \dfrac{p_{t|0}(y|x_0)}{p_{t|0}(x_t|x_0)} \lambda_t(x_t, y) \right] \dif t\right] + C\\
    =& \E_{x_0 \sim p_0}\left[\int_0^T \E_{x_t \sim p_{t|0}(\cdot| x_0)}\left[\sum_{y \in \sX} \left(\hat s_t(x_t, y) - \dfrac{p_{t|0}(y|x_0)}{p_{t|0}(x_t|x_0)} \log \hat s_t(x_t, y) \right) \lambda_t(x_t, y) \right] \dif t\right] + C.
\end{align*}

\subsection{General L\'evy-Type Process}
\label{app:general_levy_type_process}

We now turn to the general case of L\'evy-type processes, where we take the space $E = \R^d$. For ease of presentation, we extend the Markov process $(\vx_t)_{t \in [0, T]}$ to $(\vx_t)_{t \in \R}$, \emph{i.e.}, take $\sT = \R$ in \cref{def:augmented}, and make the following assumption:
\begin{assumption}[Courr\`ege]
    Under \cref{ass:feller}, the augmented process $(\tilde \vx_t)_{t \in \R}$ is conservative, \emph{i.e.}, $\tilde T_t 1 = 1$, where $\tilde T_t$ is the Feller semigroup defined in~\cref{def:feller_semigroup}, and $C_c^\infty(\R^{d+1}) \subset \dom(\tilde \gL)$, where $C_c^\infty(\R^{d+1})$ is the space of all compactly supported smooth functions in $\R^{d+1}$. 
    \label{ass:smooth}
\end{assumption}

\paragraph{Generator of General L\'evy-Type Process.}

We first present the following theorem stating that under~\cref{ass:smooth}, the most general form of the forward generator is the generator of general L\'evy-Type processes.

\begin{theorem}
    Under \cref{ass:feller,ass:smooth}, the right generator $\gL_s$ of the Markov process $(\vx_t)_{t\in\R}$ at time $s$ is of the following form:
    \begin{equation}
        \begin{aligned}
            \gL_s f_s(\vx) =& \vb_s(\vx) \cdot \nabla f_s(\vx) + \dfrac{1}{2} \mD_s(\vx) : \nabla^2 f_s(\vx) \\
            &+ \int_{\R^d \backslash \{\vx\}} \left(f_s(\vy) - f_s(\vx) - (\vy - \vx) \cdot \nabla f_s(\vx) \chi_s(\vx, \vy)\right) \lambda_s(\dif \vy, \vx),
        \end{aligned}
        \label{eq:courrege_form}
    \end{equation}
    for any $f_\cdot(\cdot) \in C_c^\infty(\R \times \R^d)$, where $\chi$ is a \emph{local unit}\footnote{We say $\chi$ is a local unit if it is a $C^\infty$ mapping from $\R^d \times \R^d$ to $[0, 1]$ satisfying that (1) $\chi(\vx, \vy) = 1$ in a neighborhood of the diagonal set $\{(\vx, \vx)| \vx\in\R^d\}$; (2) for any compact set $K \subset \R^d$, the mappings $\{\chi(\vx, \cdot)\}_{\vx \in K}$ are supported in a fixed compact set in $\R^d$.} 
    and the coefficients satisfy the following conditions:
    \begin{enumerate}[label=(\arabic*)]
        \item The drift $\vb_s(\vx) \in \R^d$ for any $\vx \in \R^d$;
        \item The diffusion matrix $\mD_s(\vx) \in \R^{d \times d}$ is positive semidefinite for any $\vx \in \R^d$;
        \item $\lambda_s(\cdot, \vx)$ is a L\'evy measure, \emph{i.e.}, a Borel measure on $\R^d \backslash \{\vx\}$ and 
        \begin{equation*}
            \int_{\R^d \backslash \{\vx\}} \left(1 \wedge \|\vy-\vx\|^2\right) \lambda_s(\dif \vy, \vx) < +\infty,
        \end{equation*}
        for any $\vx \in \R^d$.
    \end{enumerate}
    \label{thm:courrege_form}
\end{theorem}

\begin{proof}
    By Theorem~3.5.5 in~\citet{applebaum2009levy}, which is originally derived by~\citet{courrege1965forme}, the augmented generator $\tilde \gL$ of the augmented process $(\tilde \vx_t)_{t \in \R}$, the domain $\dom(\tilde \gL)$ contains $C_c^\infty(\R^{d+1})$ as assumed in~\cref{ass:smooth}, is of the following form for any 
    \begin{equation*}
        \begin{aligned}
            \tilde \gL f(\tilde\vx) =& - \tilde \vgamma(\tilde\vx) f(\tilde\vx) + \tilde\vb(\tilde\vx) \cdot \tilde \nabla f(\tilde\vx) + \dfrac{1}{2} \mD(\tilde\vx): \tilde\nabla^2 f(\tilde\vx) \\
            &+ \int_{\R^{d+1} \backslash \{\tilde\vx\}} \left(f(\tilde\vy) - f(\tilde\vx) - (\tilde\vy - \tilde\vx) \cdot \tilde\nabla f(\tilde\vx) \chi(\tilde\vx, \tilde\vy)\right) \lambda(\dif \tilde\vy, \tilde\vx),
        \end{aligned}
    \end{equation*}
    where $\tilde \nabla = (\partial_s, \nabla)$.
    The coefficients satisfy the following conditions:
    \begin{enumerate}[label=(\arabic*)]
        \item The killing rate $\gamma(\tilde\vx) \geq 0$ for any $\tilde\vx \in \R^{d+1}$;
        \item The drift $\vb(\tilde\vx) \in \R^{d+1}$ for any $\tilde\vx \in \R^{d+1}$;
        \item The diffusion matrix $\mD(\tilde\vx) \in \R^{(d+1) \times (d+1)}$ is positive semidefinite for any $\tilde\vx \in \R^{d+1}$;
        \item $\lambda(\cdot, \tilde\vx)$ is a L\'evy measure on $\R^{d+1} \backslash \{\tilde\vx\}$.
    \end{enumerate}
    Now that we assume the augmented generator $\tilde \gL$ is conservative, we see by the definition of the augmented generator (\cref{def:augmented_generator}) that 
    \begin{equation*}
        0 = \tilde \gL 1 = - \tilde\vgamma \leq 0,
    \end{equation*}
    and thus $\tilde\vgamma = 0$ (\emph{cf.} \citep[Lemma~2.32]{bottcher2013levy}).

    Notice that for any $f \in C_c^\infty(\R^{d+1})$, it also satisfies the conditions in~\cref{prop:augmented_generator}, and thus the augmented generator $\tilde \gL$ only consists of the time derivative $\partial_s$ and the right generator $\gL_s$ only involving operations on the function $f_s(\vx)$ at time $s$ regarded as a function of only $\vx \in \R^d$, where $f_s(\vx)$ is an alternative notation for $f(\tilde \vx)$ with $\tilde \vx = (s, \vx)$.
    Therefore, the augmented generator $\tilde \gL$ can be reduced to the following form:
    \begin{equation*}
        \begin{aligned}
            \tilde \gL f_s(\vx) =& \partial_s f_s(\vx) + \gL_s f_s(\vx),
        \end{aligned}
    \end{equation*}
    where the right generator $\gL_s$ satisfies
    \begin{equation*}
        \begin{aligned}
            \gL_s f_s(\vx) =& \vb_s(\vx) \cdot \nabla f_s(\vx) + \dfrac{1}{2} \mD_s(\vx) : \nabla^2 f_s(\vx) \\
            &+ \int_{\R^d \backslash \{\vx\}} \left(f_s(\vy) - f_s(\vx) - (\vy - \vx) \cdot \nabla f_s(\vx) \chi_s(\vx, \vy)\right) \lambda_s(\dif \vy, \vx),
        \end{aligned}
    \end{equation*}
    and it is easy to check that the required conditions of the coefficients are satisfied.
    The proof is thus completed.
\end{proof}

We remark that the Courr\`ege form~\eqref{eq:courrege_form} is closely related to the L\'evy-Khintchine representation of L\'evy processes, and thus, without loss of generality, we choose the local unit $\chi$ to be the indicator function $\vone_{B(\vzero, 1)}$ of the unit ball $B(\vzero, 1)$ centered at the origin with radius $1$, and give the following definition of general L\'evy-type processes.

\begin{definition}[General L\'evy-Type Process]
    A general L\'evy-type process $(\vx_t)_{t \in \R}$ is a Markov process on $\R^d$ with the generator $\gL_t$ of the following form:
    \begin{equation}
        \begin{aligned}
            \gL_t f(\vx) =& \vb_t(\vx) \cdot \nabla f(\vx) + \dfrac{1}{2} \mD_t(\vx) : \nabla^2 f(\vx) \\
            &+ \int_{\R^d \backslash \{\vx\}} \left(f(\vy) - f(\vx) - (\vy - \vx) \cdot \nabla f(\vx) \vone_{B(\vx, 1)}(\vy - \vx)\right) \lambda_t(\dif \vy, \vx),
        \end{aligned}
        \label{eq:generator_levy_type_app}
    \end{equation}
    for any $f \in C_c^\infty(\R^d)$, where $\vb_t(\vx) \in \R^d$ is the drift, $\mD_t(\vx) \in \R^{d \times d}$ is the diffusion matrix, and $\lambda_t(\cdot, \vx)$ is a L\'evy measure. In the following, we will call $(\vb_t, \mD_t, \lambda_t)$ the L\'evy triplet of the general L\'evy-type process.
\end{definition}

To enable further discussion on the construction of the backward process, we summarize and strengthen the required conditions for the coefficients in the following assumption as in~\cref{ass:regularity_diffusion}:
\begin{assumption}[Regularity of General L\'evy-Type Process]
    We assume the following smoothness and regularity conditions on the L\'evy triple of the generator $\gL_t$ in~\cref{eq:generator_levy_type_app} for any $t \in \R$:
    \begin{enumerate}[label=(\arabic*)]
        \item The drift vector $\vb_t(\vx)$ satisfies $b_t^i \in C^{1, 0}(\R^d)$ for any $i \in [d]$ and $\vx \in \R^d$;
        \item The diffusion matrix $\mD_t(\vx)$ is positive semidefinite and satisfies $D_t^{ij} \in C^{2, 0}(\R^d)$ for any $i, j \in [d]$ and $\vx \in \R^d$;
        \item The L\'evy measure $\lambda_t(\cdot, \vx)$ admits a density function with respect to the Lebesgue measure, which we denote with a slight abuse of notation with $\lambda_t(\vy, \vx)$, \emph{i.e.},
        \begin{equation*}
            \lambda_t(\dif \vy, \vx) = \lambda_t(\vy, \vx) \dif \vy,
        \end{equation*}
        and the density satisfies $\lambda_t(\vx, \vx) = 0$, $\lambda_t(\cdot, \vx) \in C^{1, 0}(\R^d)$ for any $\vx \in \R^d$, $\lambda_t(\vy, \cdot) \in C^{1, 0}(\R^d)$ for any $\vy \in \R^d$, and the following integral
        \begin{equation}
            \int_{\R^d} \left(\dfrac{p_t(\vy)}{p_t(\vx)} \lambda_t(\vx, \vy) + \lambda_t(\vy, \vx) \right) (\vy - \vx) \dif \vy
            \label{eq:detailed_balance}
        \end{equation}
        exists either \emph{bona fide} or in the sense of Cauchy principal value, for any $\vx \in \R^d$ and $t \in [0, T]$.
    \end{enumerate}
    \label{ass:regularity_levy}
\end{assumption}

The following proposition gives a stochastic integral representation of the general L\'evy-type process.
\begin{proposition}
    Under~\cref{ass:regularity_levy}, the generator associated with the Markov process defined by following stochastic integral:
    \begin{equation}
        \begin{aligned}
            \vx_t &= \vx_0 + \int_0^t \vb_s(\vx_s) \dif s + \int_0^t \mSigma_s(\vx_s) \dif \vw_s \\
            &+ \int_0^t \int_{\R^d \backslash B(\vx_{s^-}, 1)} (\vy - \vx_{s^-}) N[\lambda](\dif s, \dif \vy) + \int_0^t \int_{B(\vx_{s^-}, 1)} (\vy - \vx_{s^-}) \tilde N[\lambda](\dif s, \dif \vy),
        \end{aligned}
        \label{eq:integral_general_levy}
    \end{equation}
    coincides with the form~\eqref{eq:generator_levy_type_app},
    where the matrix $\mSigma_s(x)$ satisfies $\mSigma_s(x) \mSigma_s(x)^\top = \mD_s(x)$ for any $x \in \R^d$, $(\vw_s)_{s \geq 0}$ is a $d$-dimensional Wiener process, $N[\lambda](\dif s, \dif \vy)$ denotes the Poisson random measure with evolving intensity $\lambda_s(\vy, \vx_{s^-}) \dif \vy$, and $\tilde N[\lambda](\dif s, \dif \vy)$ is the compensated version of $N[\lambda](\dif s, \dif \vy)$, \emph{i.e.},
    \begin{equation*}
        \tilde N[\lambda](\dif s, \dif \vy) = N[\lambda](\dif s, \dif \vy) - \lambda_s( \vy, \vx_{s^-}) \dif s \dif \vy.
    \end{equation*}
    \label{prop:stochastic_integral_general_levy}
\end{proposition}

\begin{proof}
    By taking $X(t) = \vx_t$, $z = (\vy, \xi)$, $\alpha(t, X(t)) = \vb_t(\vx_t)$, $\sigma(t, X(t)) = \mSigma_t(\vx_t)$, and 
    $$
        \gamma(t, z, \omega) = (\vy - \vx_{t^-}(\omega)) \vone_{[0, \lambda_t(\vy, \vx_{t^-}(\omega))]}(\xi)
    $$ 
    in~\citep[Theorem~1.16]{oksendal2019stochastic},
    where $\xi \in \R$ is an auxiliary variable augmented to the space $\R^d$ to account for varying jump intensities at different locations, following the strategy in~\citet{protter1983point}, we see that the stochastic integral~\eqref{eq:integral_general_levy} corresponds to a unique strong Markov c\`adl\`ag solution and by the general form of~\citep[Theorem~1.22]{oksendal2019stochastic} with the compensated Poisson random measure part, we have that
	\begin{equation*}
		\begin{aligned}
			\gL_t f(\vx) =& \vb_t(\vx) \cdot \nabla f(\vx) + \dfrac{1}{2} \mD_t(\vx) : \nabla^2 f(\vx) \\
            &+  \int_{\R^d \backslash B(\vx, 1)} \int_\R \left( f(\vx + (\vy - \vx) \vone_{[0, \lambda_t(\vy, \vx)]}(\xi)) - f(\vx) \right) \dif \vy \dif \xi \\
            &+ \int_{B(\vx, 1)} \int_\R \big( f(\vx + (\vy - \vx) \vone_{[0, \lambda_t(\vy, \vx)]}(\xi)) - f(\vx)\\
            &\quad \quad \quad \quad \quad \quad \quad \quad \quad \quad - (\vy - \vx) \cdot \nabla f(\vx) \vone_{[0, \lambda_t(\vy, \vx)]}(\xi) \big) \dif \vy \dif \xi\\
			=& \vb_t(\vx) \cdot \nabla f(\vx) + \dfrac{1}{2} \mD_t(\vx) : \nabla^2 f(\vx) + \int_{\R^d \backslash B(\vx, 1)} \left( f(\vy) - f(\vx) \right) \lambda_t(\vy, \vx) \dif \vy \\
            &+ \int_{B(\vx, 1)} \left( f(\vy) - f(\vx)  - (\vy - \vx) \cdot \nabla f(\vx) \right) \lambda_t(\vy, \vx) \dif \vy,
		\end{aligned}
	\end{equation*}
    which is exactly the form~\eqref{eq:generator_levy_type_app} and the proof is complete.
\end{proof}

\paragraph{Domain of the Generator.} We set $\gU = C^\infty(\R^d) \cap C^{2,0}(\R^d)$, and $\gV = \gS(\R^d)$, the Schwartz space on $\R^d$. Clearly, under~\cref{ass:regularity_levy}, we have that $\gU \subset \dom(\gL_t)$ for any $t \in [0, T]$ and that for any $f \in \gU$ and $g \in \gV$,~\cref{eq:adjoint} holds as 
\begin{align*}
        &\angle{\gL_t f,  g} = \int_{\R^d} g(\vx) \gL_t f(\vx) \dif \vx\\
        =& \int_{\R^d} g(\vx) \left(\vb_t(\vx) \cdot \nabla f(\vx) + \dfrac{1}{2} \mD_t(\vx) : \nabla^2 f(\vx)\right) \dif \vx \\
        &+\int_{\R^d} \int_{\R^d} \left(f(\vy) - f(\vx) - (\vy - \vx) \cdot \nabla f(\vx) \vone_{B(\vzero, 1)}(\vy - \vx)\right) \lambda_t(\vy, \vx) g(\vx) \dif \vx \dif \vy\\
        =& -\int_{\R^d} \nabla \cdot (\vb_t(\vx) g(\vx)) f(\vx) \dif \vx + \int_{\R^d} \dfrac{1}{2} \nabla^2 : \left(\mD_t(\vx) g(\vx)\right) f(\vx)  \dif \vx\\
        &+ \int_{\R^d} \int_{\R^d} \left(f(\vx) \lambda_t(\vx, \vy) g(\vy) - f(\vx) \lambda_t(\vy, \vx) g(\vx) \right) \dif \vx \dif \vy\\
        &+ \int_{\R^d}  \int_{\R^d} f(\vx) \nabla \cdot \left((\vy - \vx) \vone_{B(\vx, 1)} (\vy - \vx) \lambda_t(\vy, \vx) g(\vx)\right) \dif \vx \dif \vy = \angle{f, \gL_t^* g},
\end{align*}
\emph{i.e.}, $\gV \subset \dom(\gL_t^*)$ for any $t \in [0, T]$, with the adjoint operator $\gL_t^*$ given by
\begin{equation}
    \begin{aligned}
        \gL_t^* g = &-\nabla \cdot (\vb_t g) + \dfrac{1}{2} \nabla^2 : (\mD_t g)+ \int_{\R^d} \left(\lambda_t(\vx, \vy) g(\vy) - \lambda_t(\vy, \vx) g(\vx)\right) \dif \vy\\
        &+ \int_{\R^d} \nabla \cdot \left((\vy - \vx) \vone_{B(\vx, 1)}(\vy - \vx) \lambda_t(\vy, \vx) g(\vx)\right) \dif \vy.
    \end{aligned}
    \label{eq:adjoint_levy}
\end{equation}
It is also checked that for any $g \in \gV$, $g \gU \subset \gV$.

\paragraph{Backward Generator.} We plug the form of the adjoint generator $\gL_t^*$~\eqref{eq:adjoint_levy} into~\cref{thm:time_reversal_app} and obtain the backward generator $\cev \gL_t$ as follows:
\begin{align*}
    &\cev \gL_{T-t} f = p_t^{-1} \gL_t^*(p_t f) - p_t^{-1} f \gL_t^* p_t \\
    =&\left(-\vb_t + \mD_t \nabla \log p_t + \nabla \cdot \mD_t \right) \cdot \nabla f + \dfrac{1}{2} \mD_t: \nabla^2 f + \int_{\R^d} \left(f(\vy) - f(\vx)\right) \dfrac{p_t(\vy)}{p_t(\vx)} \lambda_t(\vx, \vy) \dif \vy \\
    &+ \int_{\R^d} (\vy - \vx) \cdot \nabla f(\vx) \vone_{B(\vx, 1)}(\vy - \vx) \lambda_t(\vy, \vx) \dif \vy\\
    =&\left(-\vb_t + \mD_t \nabla \log p_t + \nabla \cdot \mD_t \right) \cdot \nabla f + \dfrac{1}{2} \mD_t: \nabla^2 f + \int_{\R^d} \left(f(\vy) - f(\vx)\right) \dfrac{p_t(\vy)}{p_t(\vx)} \lambda_t(\vx, \vy) \dif \vy \\
    &+ \int_{\R^d} (\vy - \vx) \cdot \nabla f(\vx) \vone_{B(\vx, 1)}(\vy - \vx) \left(\dfrac{p_t(\vy)}{p_t(\vx)} \lambda_t(\vx, \vy) + \lambda_t(\vy, \vx)\right)  \dif \vy\\
    &- \int_{\R^d} (\vy - \vx) \cdot \nabla f(\vx) \vone_{B(\vx, 1)}(\vy - \vx) \dfrac{p_t(\vy)}{p_t(\vx)} \lambda_t(\vx, \vy)  \dif \vy
\end{align*}
which also corresponds to a general L\'evy-type process with the generator of the form~\eqref{eq:generator_levy_type_app} with the L\'evy triplet $(\cev \vb_t, \cev \mD_t, \cev \lambda_t)$:
\begin{gather*}
    \cev \vb_{T-t}(\vx) = -\vb_t(\vx) + \mD_t(\vx) \nabla \log p_t(\vx) + \nabla \cdot \mD_t(\vx)\\
    \quad \quad \quad \quad \quad \quad \quad \quad \quad+ \int_{B(\vx, 1)} \left(\dfrac{p_t(\vy)}{p_t(\vx)} \lambda_t(\vx, \vy) + \lambda_t(\vy, \vx)\right) (\vy - \vx) \dif \vy,\\
    \cev \mD_{T-t}(\vx) = \mD_t(\vx),\quad
    \cev \lambda_{T-t}(\vy, \vx) = \dfrac{p_t(\vy)}{p_t(\vx)} \lambda_t(\vx, \vy).
\end{gather*}

\paragraph{KL Divergence and Loss Function.} The KL divergence in~\cref{cor:change_of_measure} can be rewritten as 
\begin{align*}
    &\KL(\P \| \sQ) = \E\left[\int_0^T \eta_t(\vx_t) \gL_t\eta_t^{-1}(\vx_t) + \gL_t \log \eta_t(\vx_t)  \dif t\right] \\
    =& \E\bigg[\int_0^T \dfrac{1}{2} \mD_t(\vx_t) : \nabla \log \eta_t(\vx_t) \nabla^\top \log \eta_t(\vx_t) \dif t \\
    &\quad + \int_0^T \int_{\R^d} \left(\dfrac{\eta_t(\vx_t)}{\eta_t(\vy)} - 1 - \log \dfrac{\eta_t(\vx_t)}{\eta_t(\vy)}\right) \lambda_t(\vy, \vx_t) \dif \vy \dif t\bigg],
\end{align*}
where the first part corresponds to the diffusion part $\vb_t(\vx) \cdot \nabla f(\vx)$ and the second part the large jump part $\int_{\R^d \backslash \{\vx\}} (f(\vy) - f(\vx)) \lambda_t(\vy, \vx) \dif \vy$. The small jump part $\int_{\R^d \backslash \{\vx\}} (\vy - \vx) \cdot \nabla f(\vx) \vone_{B(\vx, 1)}(\vy - \vx) \lambda_t(\vy, \vx) \dif \vy$ does not contribute to the KL divergence by noticing that $\eta_t \nabla \eta_t^{-1} + \nabla \log \eta_t = 0$. Then simple algebraic manipulations lead to~\cref{eq:general_levy_loss,eq:general_levy_score_matching}.

\vspace{.5em}

\begin{proof}[Proof of \cref{thm:heavy_tail}]
    The affine linear SDE admits the explicit variation-of-constants representation
    $$
    \vx_T = \mM_T \vx_0 + \vm_T + \mZ_T,
    \quad
    \vm_T := \int_0^T \mM_{T,s}\vu_s\dif s,\quad
    \mZ_T := \int_0^T \mM_{T,s}\mSigma_s\dif \vw_s,
    $$
    where $\mM_t$ solves $\dot{\mM}_t=\mA_t\mM_t$ with $\mM_0=\mI$, and $\mM_{T,s}:=\mM_T\mM_s^{-1}$. For each $t$, $\mM_t$ is invertible and we denote $\sigma_{\min}(\mM_T)>0$ as its smallest singular value. Since $\mSigma_t$ is bounded, $\E\|\mZ_T\|^2<\infty$ and $\vm_T$ is deterministic and finite.
    
    We first show heavy-tail persistence at finite time:
    $$
    \E\|\vx_T\|^2 
    \ge \frac12\E\|\mM_T\vx_0\|^2 - \E\|\vm_T+\mZ_T\|^2
    \ge \frac12\sigma_{\min}^2(\mM_T)\E\|\vx_0\|^2 - C
    = \infty,
    $$
    using $\E\|\vx_0\|^2=\infty$.
    
    Let $q_0=\gN(\vmu,\mQ)$ and denote by $\lambda_{\max}(\mQ)$ the largest eigenvalue. By the Donsker-Varadhan variational formula,
    $$
    \KL(p_T\|q_0)=\sup_{f}\Bigl\{\E_{p_T}[f(\vx)]-\log\E_{q_0}[e^{f(\vx)}]\Bigr\}.
    $$
    Choosing $f(\vx)=\alpha\|\vx\|^2$ with any $\alpha\in\bigl(0,\tfrac{1}{2\lambda_{\max}(\mQ)}\bigr)$ yields
    $$
    \KL(p_T\|q_0)\ge\alpha\E\|\vx_T\|^2 -\log\E_{q_0}\bigl[e^{\alpha\|\vx\|^2}\bigr] = \infty,
    $$
    which completes the proof.
    \end{proof}
\end{document}